\pdfoutput=1
\documentclass[10pt]{article}
\usepackage{thuidm}

\usepackage[utf8]{inputenc}
\usepackage{amsfonts}
\usepackage{mathtools}
\usepackage{array}
\usepackage{tikz}

\DeclareRobustCommand{\stepnum}[1]{\tikz[baseline=(sn.base)]\node[circle, fill=black!75, text=white, font=\footnotesize\bfseries\sffamily, inner sep=0pt, minimum size=1.35em] (sn) {#1};}

\theoremstyle{thuidm}                     

\newtheorem{hypothesis}[theorem]{Hypothesis}
\theoremstyle{thuidmplain}                

\crefname{equation}{Eq.}{Eqs.}
\Crefname{equation}{Eq.}{Eqs.}
\crefname{hypothesis}{Hypothesis}{Hypotheses}
\Crefname{hypothesis}{Hypothesis}{Hypotheses}
\crefname{conjecture}{Conjecture}{Conjectures}
\Crefname{conjecture}{Conjecture}{Conjectures}
\crefname{principle}{Principle}{Principles}
\Crefname{principle}{Principle}{Principles}
\crefname{claim}{Claim}{Claims}
\Crefname{claim}{Claim}{Claims}

\DeclareMathOperator*{\argmax}{arg\,max}
\DeclareMathOperator*{\argmin}{arg\,min}
\DeclareMathOperator{\CVaR}{CVaR}

\newcommand{\KL}[2]{D_{\mathrm{KL}}\!\left[#1\,\Big\|\,#2\right]}
\newcommand{\E}{\mathbb{E}}
\newcommand{\R}{\mathbb{R}}
\newcommand{\indicator}[1]{\mathbf{1}\!\left[#1\right]}
\newcommand{\norm}[1]{\left\|#1\right\|}
\newcommand{\abs}[1]{\left|#1\right|}
\newcommand{\calH}{\mathcal{H}}
\newcommand{\calI}{\mathcal{I}}
\newcommand{\calC}{\mathcal{C}}
\newcommand{\calP}{\mathcal{P}}
\newcommand{\calL}{\mathcal{L}}

\newcommand{\calF}{\mathcal{F}}
\newcommand{\calA}{\mathcal{A}}
\newcommand{\calS}{\mathcal{S}}
\newcommand{\calD}{\mathcal{D}}
\newcommand{\calR}{\mathcal{R}}

\newcommand{\calT}{\mathcal{T}}

\newcommand{\calU}{\mathcal{U}}

\newcommand{\calV}{\mathcal{V}}
\newcommand{\calO}{\mathcal{O}}

\newenvironment{restated}[1]{\par\medskip\noindent\textbf{\Cref{#1} (restated).}\ \itshape\ignorespaces}{\par\medskip}

\title{Does Scaling Toward AGI Depend on What System-1 Lets System-2 Execute?}

\author{%
  Qi Wang\aff{1}, Yun Qu\aff{1}, Lizhou Cai\aff{1}, Heming Zou\aff{1}, Haoang Chi\aff{1}, Yiqin Lv\aff{1}, Xiang Li\aff{1},\\
  Haiyan Feng\aff{1}, Mengyu Ge\aff{2}, Yize Geng\aff{2}, Yuhang Jiang\aff{1}, Yixiu Mao\aff{1}, Xiangyang Ji\aff{1}$^\dagger$%
}

\affiliation{%
  \aff{1}Department of Automation, Tsinghua University, Beijing, China\\
  \aff{2}Department of Computer Science, Tsinghua University, Beijing, China%
}

\datestamp{October 27, 2026}
\correspondence{$^\dagger$Prof. Xiangyang Ji at \href{mailto:xyji@tsinghua.edu.cn}{xyji@tsinghua.edu.cn}}
\projectwebsite{https://cheemswang.github.io/blogs/sys1-sys2-codesign.html}

\labwebsite[www.thuidm.com]{https://www.thuidm.com/}

\begin{document}

\begin{abstract}
Progress toward artificial general intelligence (AGI) has largely relied on scaling a resource-intensive executive, System-2, through larger models, more data and compute, longer reasoning, and more agents. We ask whether the next stage may come instead from co-designing it with System-1, a lightweight, semantically aware predictor that evaluates candidate interactions before execution, so that System-2 acts only on those worth their cost. Our answer is a qualified yes.
We study \textit{conditional interaction}, a predict--gate--execute--update loop: System-1 estimates each candidate's utility, System-2 executes those admitted by an adaptive threshold, and the outcomes update System-1. The loop spans training, test-time reasoning, adaptive decision-making, and agentic tool use, with utility and cost defined by the application.
Under ideal screening and equal execution costs, execution compute falls by a factor approaching $1/\rho_u$, where $\rho_u$ is the density of utility-relevant interactions (cost-weighted when costs differ); prediction errors, exploration, and distribution shift reduce the saving. We bound selection regret by predictor error and the utility margin.
These results motivate our central hypothesis: $\rho_u$ is a measurable coordinate of scaling efficiency and makes informative compute a better description of scaling than total compute. 
\end{abstract}

\maketitle

\section{The Interaction Problem at Scale}
\label{sec:intro}
\begin{figure}
\centering
\centerline{\includegraphics[width=0.95\textwidth]{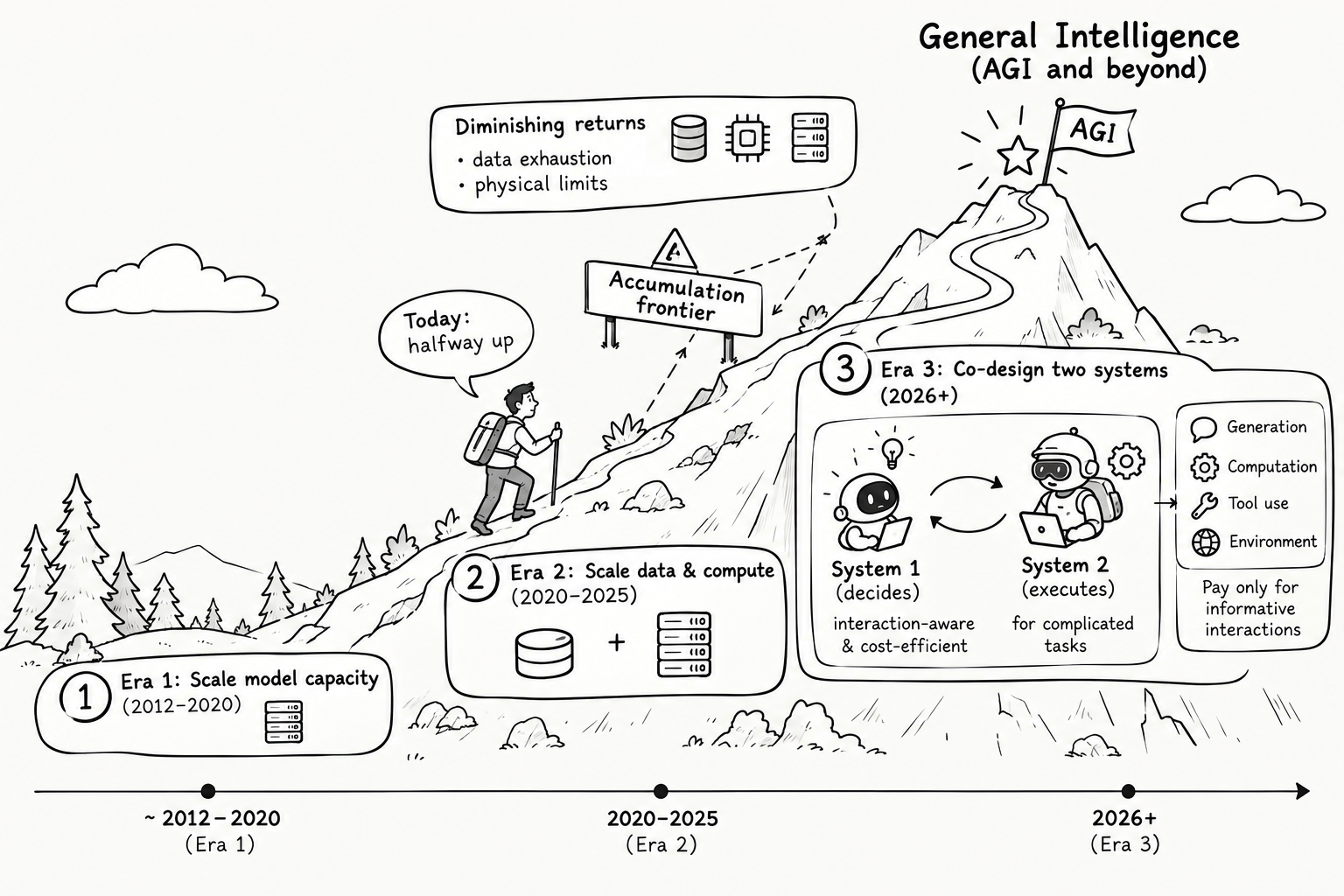}}
\caption{\textbf{An Illustrative Roadmap to AGI and its Beyond.}
As AI development progresses, broad capability resembles a summit that systems approach but may hit diminishing returns due to constraints in data, compute, and parameters. 
In Era~1 (2012--2020) progress came mainly from scaling model capacity, and in Era~2 (2020--2025) from scaling data and computation.
For Era~3 (2026+), this work emphasizes a next-scaling lens: co-designing a cost-effective System-1 that prioritizes interactions by expected utility and a resource-intensive System-2 for complicated tasks. 
This emphasis becomes increasingly consequential for foundation models and embodied agents as they depend on generation, computation, tool use, and environmental interaction, and thus on the cost of each.}
\label{fig:mountain}
\end{figure}

Recent advances toward artificial general intelligence (AGI) \citep{morris2024levels} have largely focused on scaling model capacity, data volume, and computational resources \citep{sutton2019bitter, kaplan2020scaling, hoffmann2022training}, along with innovations in architectures, pretraining or post-training techniques, retrieval, and inference-time search. 
A significant portion of this investment has gone to a resource-intensive intelligent system, referred to as \emph{System-2}, which handles reasoning, planning, decision-making, function calling, or even computer use. 
As frontier large models \citep{jaech2024openai, guo2025deepseek}, extend compute, and larger collectives of agents emerge, the potential of this executive system grows proportionally with resource availability. 
However, as systems become more capable, the costs of generating, evaluating, and acting on each additional interaction may become a constraining factor. 
This presents an alternative scaling challenge: not merely how many interactions a system can manage, but which interactions justify their associated costs.

An \emph{interaction} encompasses any information exchange between an intelligent system and its data, computations, or environment. 
It gathers information at a cost: for a large language model (LLM), a \emph{generation episode} in which a prompt produces tokens, such as a Chain-of-Thought (CoT) \citep{wei2023cot}; for an embodied or reinforcement-learning agent, a \emph{trajectory} of states, actions, and rewards; for a tool-using agent, a \emph{tool invocation} that calls an external resource and integrates its result. 
Here, we treat interaction as the unit of analysis, recognizing that these operations are not necessarily equivalent in a physical or statistical sense; they share that each incurs resource commitments before observing the outcome and realizes \emph{utility} only afterward. 
Even though reasoning models \citep{guo2025deepseek} and embodied agents \citep{kim2024openvla} benefit from increased interactions, the relative returns from each interaction tend to diminish. 
For instance, tokens that fail to alter the relevant belief state, training rollouts that yield little policy-gradient signal under the selected objective \citep{yu2026dapo}, and environment steps that reiterate known information all consume valuable resources while providing minimal task-relevant learning signals. 

Determining which interactions to execute is a challenge that System-2 cannot resolve autonomously because assessing an interaction's value must cost considerably less than executing it. 
This necessitates the co-design of System-1 and System-2, and this work introduces the \emph{conditional interaction} as a feasible co-design principle: a lightweight System-1 predicts the utility of potential interactions, the resource-intensive System-2 executes only those interactions that exceed a resource-aware threshold, and the resulting outcomes then serve to refine System-1. 
Because each system shapes the other, the two must be designed together, and this Perspective asks whether such co-design is a promising path toward AGI.
Our central hypothesis posits that the density of utility-relevant interactions, denoted as $\rho_u$, rather than the sheer count of interactions alone, serves as a pivotal metric of scaling efficiency. 
The co-design of System-1 and System-2, as detailed below, indicates a promising approach for scenarios where interactions are costly, exhibit uneven value, and can be anticipated. This framework is based on the premise that low-cost experiments can be conducted and foreseen effectively.

\paragraph{The Closing of the Accumulation Frontier.}
The pursuit of accumulating additional data is increasingly difficult. \emph{Human-generated data are finite}, since pretraining consolidates much of humanity's recorded text into a single model and the stock of public text may be depleted within this decade \citep{villalobos2024run}, while \emph{high-quality data are scarcer still}, as the subset that teaches most is only a fraction of that stock. 
\emph{Synthetic data shift the focus from quantity to coverage}: they can partially replenish both \citep{xu2025magpie}, but the space of synthesizable interactions is combinatorially vast, e.g., robotic data alone would require covering scenes, objects, dynamics, lighting, and sensing, so the pertinent question becomes which synthetic interactions are worth their cost. 
Moreover, \emph{some interactions are inherently expensive}: in scientific discovery, robotics, and specialized clinical and industrial settings, data are produced one experiment, actuation, or measurement at a time, at the pace of the underlying process, and no amount of computational power can accelerate a chemical reaction. 
Shifting the focus toward continual, post-deployment adaptation, we can find that self-evolution sharpens rather than alleviates the bottleneck: an agent that interacts indiscriminately exhausts its finite budget on experiences that yield minimal task-relevant learning, so its progress hinges on which interactions it selects, precisely the decision this Perspective examines.

\paragraph{The vanishing marginal return of undirected interactions and complexity of search space.} 
Neural scaling laws \citep{kaplan2020scaling, hoffmann2022training} illustrate why counting interactions is ineffective. 
Let $\calC$ be cumulative compute, $\calL(\calC)$ the loss, and $\calP=1-\calL$ a performance measure, e.g., any smooth $\calP=f(\calL)$ with bounded $f'$ yields similar conclusions. 
Under a power law $\calL(\calC)\propto\calC^{-\alpha}$, $\alpha>0$, the marginal gain from one more \emph{unconditional} interaction is
\begin{equation}
  \eta_{\text{unc}}(t)
  \;=\; \frac{d\calP}{d\calC}\bigg|_t
  \;=\; -\frac{d\calL}{d\calC}\bigg|_t
  \;\propto\; \alpha\,\calC^{-\alpha-1},
  \label{eq:marginal-efficiency}
\end{equation}
which vanishes as compute grows. 
This empirical relation describes diminishing returns on a \emph{fixed} landscape traversed by \emph{undifferentiated} interactions; altering which interactions are executed can shift the curve itself, and this Perspective explores that degree of freedom. 
Information gain is a canonical utility but not a universal one (\cref{def:utility,rem:utility-scope}). 
The following two regimes make the gap between interaction count and interaction utility concrete, and show where the waste arises.

\textit{The digital regime: intensive computations, wasted rollouts.}  
In reinforcement learning with verifiable rewards (RLVR), a reasoning model is post-trained on problems with automatically checkable answers, scoring each sampled thought chain (a \emph{rollout}) with a low-cost verifier. Group Relative Policy Optimization (GRPO) \citep{shao2024deepseekmath, guo2025deepseek} samples $G$ responses $\{o_g\}_{g=1}^G$ per prompt $\tau$, scores each with reward $r_g$, and forms group-relative advantages
\begin{equation}
  A_g^{\text{GRPO}}(\tau) = \frac{r_g - \mu_r}{\sigma_r + \epsilon},
  \quad \mu_r = \frac{1}{G}\sum_{g=1}^G r_g,
  \quad \sigma_r^2 = \frac{1}{G}\sum_{g=1}^G (r_g - \mu_r)^2,
  \label{eq:grpo-advantage}
\end{equation}
which weight the per-token log-likelihood gradients in a clipped surrogate objective regularized toward a reference policy. When all responses receive the same reward ($r_g \equiv \mu_r$), every $A_g=0$, rendering the prompt \emph{degenerate}: it provides no relative policy-gradient signal while consuming $G$ rollouts. Only prompts of intermediate difficulty yield mixed outcomes, and a substantial fraction of the pool is degenerate \citep{yu2026dapo,qu2026moppps}. Conditional interaction aids in predicting which prompts yield non-zero advantage, thereby optimizing the generation budget.

\textit{The physical regime: expensive rollouts in high-dimensional task spaces.}  
In contrast to the digital regime's cheap interactions, the physical regime is constrained by expensive rollouts: real-world interactions cannot be freely parallelized and consume tangible time tied to physical dynamics. A generalist agent trained by domain randomization performs across a simulation parameter space $\Xi\subseteq\R^d$, where the rollouts needed for generalization grow with the effective dimensionality of $\Xi$ and the complexity of its deployment scenarios. 
The qualitative insight is that, in high-dimensional spaces, \emph{which} tasks are evaluated holds equal importance to how many are assessed. 

\paragraph{The super parrot hypothesis: AI for AI, and its efficiency issue.}
Fig.~\ref{fig:super-parrot} sketches a trend we call the \emph{super parrot hypothesis}. The early foundation model first behaves as a ``stochastic parrot'' \citep{bender2021dangers}, imitating human text without grounded understanding (\emph{Stage~1}, roughly 2020--2022); CoT prompting \citep{wei2023cot} and RLVR \citep{shao2024deepseekmath, guo2025deepseek} then let it develop reasoning that humans can still follow (\emph{Stage~2}, 2022--2025); and we hypothesize a \emph{Stage~3} in which it encodes information in self-coded representations and improves itself through self-evolution, beyond human interpretation. Viewed as a trajectory, this is a trend toward AI for AI: at each stage a larger share of what the system learns from is produced by AI itself, synthetic instructions \citep{xu2025magpie}, self-generated reasoning traces scored by automatic verifiers, and ultimately self-chosen experiments \citep{silver2025era}. 
The same trend exhibits the efficiency issue in sharper form. 
Every self-generated interaction is paid for with the system's own compute, yet many of them, such as degenerated RLVR prompts, reasoning tokens that leave the answer unchanged, and self-proposed tasks that are already mastered, teach little to the model. 
Developing a super parrot therefore needs an internal judgment of value before execution, which is precisely the role this Perspective assigns to System-1.

All of the above discussions raise a fundamental Research Question: 
\textit{How can a System-2 be aware of an interaction's value before execution?} 

\begin{figure}[t]
\centering
\includegraphics[width=0.95\textwidth]{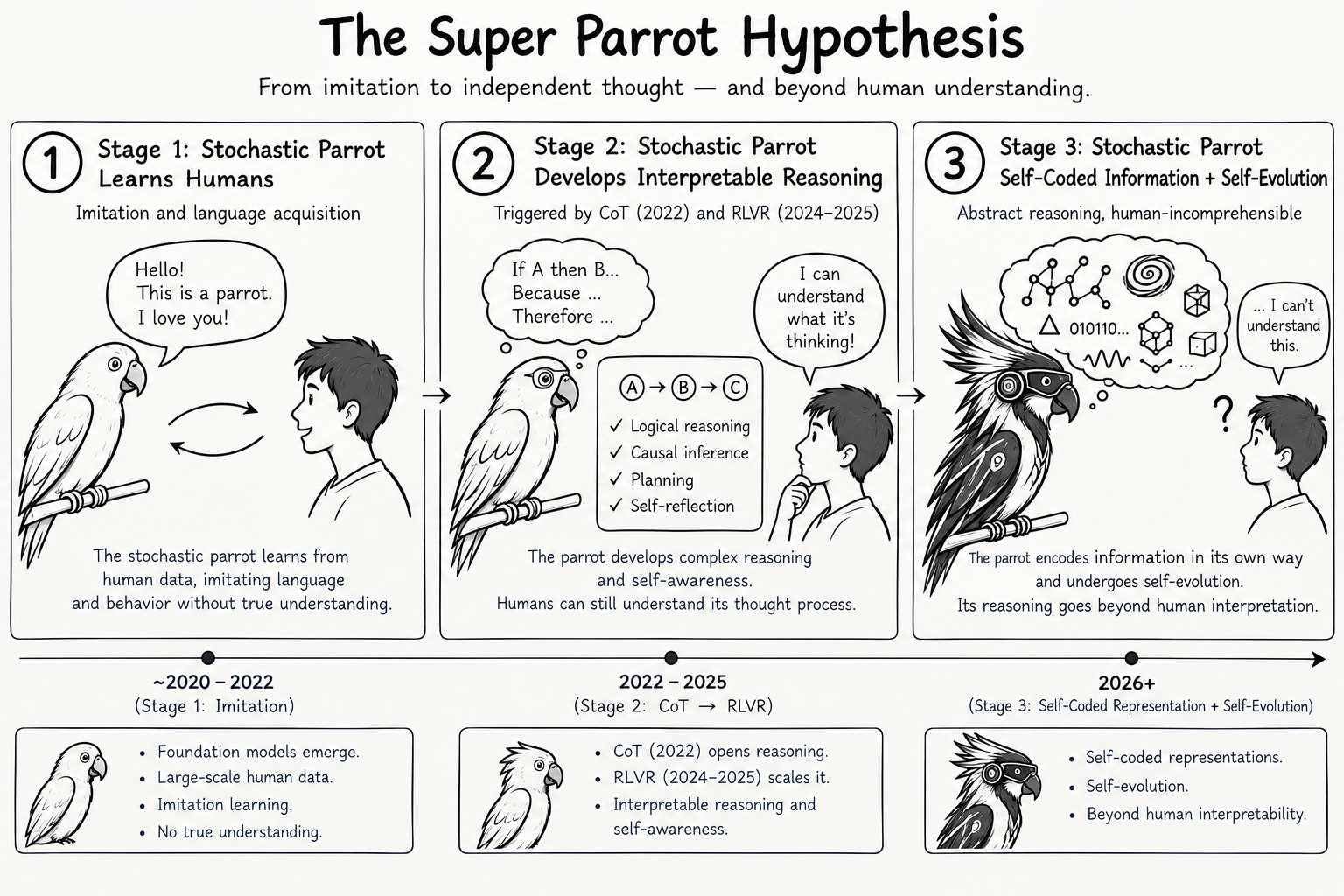}
\caption{\textbf{The Super Parrot Hypothesis.}
An illustrative, hypothesized trajectory rather than an established chronology: a conceptual progression of AI from imitation-based language learning (2020--2022), through interpretable reasoning enabled by chain-of-thought (CoT) and reinforcement learning with verifiable rewards (RLVR) (2022--2025), to a hypothesized future (2026+) of self-coded representations and self-evolution that may exceed human interpretability. As AI increasingly produces the data and feedback it learns from (AI for AI), deciding which self-generated interactions are worth executing becomes a central efficiency problem.}
\label{fig:super-parrot}
\end{figure}

\paragraph{Our position: co-design the two systems, and the conditional interaction decides what the other pays for.}
In response, we prioritize architecture over algorithm: give the executive a partner whose job is to judge, and design the two together. 
This work argues that a major part of the next advance toward AGI and beyond may come from \emph{co-designing} System-1 and System-2 rather than from scaling System-2 alone: specifying, training, and scaling the two as one system, where System-1 decides what System-2 executes, which executive runs it, how much it may spend, and what it keeps in memory; System-2's outcomes are System-1's supervision; and each is sized with the other in view (\cref{sec:architecture}). 
Integrating a cheap predictor with a costly executor introduces three decisions at scale: which prompts, tasks, or requests merit rollouts, trials, or extensive thought chains (\emph{compute}); which samples, buffers, caches, and context entries should be preserved (\emph{storage}); and how costly outcomes should adjust the inexpensive judgment that allocated them (\emph{feedback}).

This work proposes conditional interaction as a coupling mechanism that allocates interactions based on predicted utility, explicitly optimizing to avoid those predicted to fall below a certain threshold. 
This perspective applies to diverse applications, including online prompt selection in reinforcement learning from verifiable rewards (RLVR) \citep{qu2026moppps}, task selection in embodied agents \citep{qu2025fast}, model routing, and adaptive inference \citep{li2026limited} (\cref{sec:master}).
The same pattern applies in a deployed system, where the ledger records inference compute and real-world actions rather than gradient updates \citep{snell2024scaling}. 
The transfer focuses on the decision pattern and the scope is intentionally confined to situations with a reliable utility predictor, as this principle does not substitute exploration when the predictor is uncertain (\cref{sec:exploration}).

\paragraph{$\rho_u$: an efficiency coordinate to scale limited resource.}
Classical scaling laws express performance $\calP$ through parameters $N$, data $D$, and compute $\calC$. The density of utility-relevant interactions $\rho_u$, the fraction of candidate interactions whose utility $u$ clears a declared floor $u_{\min}$ in Eq.~(\ref{eq:rho-operational}), is not an additional resource but an \emph{efficiency coordinate} $\calP=f(N,D,\calC;\rho_{u})$ describing how well the three convert into capability, generated by the learning dynamics rather than supplied externally. For the canonical utility, Bayesian information gain, $\rho_u$ is the \emph{informative interaction density} $\rho_{\mathrm{info}}$, a special case of the general $\rho_u$ in which every formal result is stated; every statement is nonetheless relative to a declared utility, since the densities of prompts with non-zero GRPO advantage, of tasks above an adaptation-risk threshold, and of queries whose answer improves with more compute are different quantities and are not interchangeable.

Uninformative interactions waste both computation and the data slots they occupy, so the effective resources under unconditional training are discounted as $\rho_{u}D$ and $\rho_{u}\calC$. The compute discount is a direct accounting identity; $\rho_uD$ is a proposed measure of utility-relevant data \emph{consumed}, not a smaller available corpus, which a gate leaves unchanged. Two interpretations of this fraction must be kept apart. The \emph{pool} fraction $\rho_{u}(\theta_t)$ is a property of the candidate distribution at the current state, produced by the learning dynamics, and no gate can change it; the \emph{executed} fraction is the share of interactions actually paid for that are informative. Unconditional training executes the pool, so the two coincide, while a perfect gate drives the executed fraction to one. Conditional interaction raises the executed fraction, so that at fixed $N$, $D$ and $\calC$ more of the budget tends to buy utility-relevant interactions and, hypothetically, as a dynamic account of the exponent. 
Note that $\rho_{u}$ depends on the policy state, it is not a free axis.

\begin{figure}[t]
\centering
\includegraphics{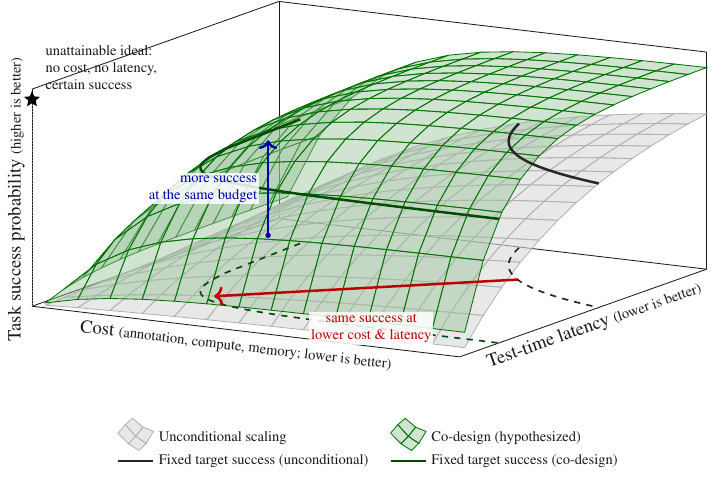}
\caption{\textbf{The design objective: a better Pareto frontier over cost, latency and task success.} Each surface is the best task success probability attainable within a cost and latency budget. The co-designed pair would spend both only where they change the answer, so its frontier (green) is hypothesized to dominate: the same success more cheaply (red), or more of it at the same budget (blue). This expresses the co-design objective.}
\label{fig:pareto3d}
\end{figure}

\paragraph{A potentially tighter Cost--QoS--Intelligence frontier.}
Directing interactions by predicted utility can tighten the Pareto frontier over compute, data, and model complexity: 
For example, under the assumptions of \cref{prop:rescaling}, a target is probably reached with less compute, or a fixed budget yields more performance. Suppressing the uninformative fraction $1-\rho_{u}$ acts on every axis at once, since fewer rollouts and gradient reduce compute and data, and routing each task only to the System-2 models it needs cuts the capacity activated per task. 
The same holds at deployment: training-time gating lowers \emph{cost}, test-time gating improves \emph{quality of service} with lowers latency, and concentrating both on informative interactions raises \emph{task success} per unit of either, so the co-designed pair aims at a frontier that improves on scaling System-2 alone on all three (See Fig.~\ref{fig:pareto3d}), the comprehensive co-design objective.

\paragraph{What we claim, and with what confidence.}
Throughout this work, the claim hierarchy is as follows.
\emph{(i) Premise:} future digital and robotic systems may be limited by interaction efficiency more than by parameters or data.
\emph{(ii) Position:} progress toward general intelligence should co-design a cheap System-1 with a costly System-2, coupled by a strategy that \emph{values interactions before execution} and suppresses those predicted to be uninformative.
\emph{(iii) Hypothesis:} the density of utility-relevant interactions $\rho_u$ ($\rho_{\mathrm{info}}$ for information gain) is an important determinant of how efficiently resources become capability.
\emph{(iv) Prediction:} $\rho_u$ \emph{decays} as the system improves, and this decay may account for part of the diminishing returns and of the scaling exponent, so that scaling laws are better described in informative compute than in compute alone.

The rest of this Perspective proceeds as follows: formalizing the predict--gate--execute--update loop to implement the conditional interaction (\cref{sec:formal}); building and co-scaling the two subsystems (\cref{sec:architecture}); deriving the loop's screening and regret properties, and why a lightweight System-1 can provide them (\cref{sec:theory}); instantiating the framework across digital training, physical adaptation, test-time computation and agentic use (\cref{sec:domains,sec:test-time}); and formulating informative compute as a scaling hypothesis (\cref{sec:implications}). 
\section{Conditional Interaction as A Co-Design Strategy: Predict Before You Pay}
\label{sec:formal}

Co-design requires an interface: a precise statement of what System-1 judges, what System-2 executes as a resource-intensive operation, and what flows back. 
\emph{Conditional interaction} is that interface. 
Both systems are defined by execution role rather than model family: System-1 is a lightweight, semantic-aware module that takes a candidate interaction as input, e.g., a prompt, task configuration, query, or tool call, with its context and returns a score rather than a solution, and it is consulted on every candidate at a cost $c_{\text{pred}}\ll c_{\text{exec}}$.
System-2 is the resource-demanding model or agent that performs the interaction, and is invoked only on the candidates System-1 certifies. 
The loop runs in both directions, making it co-design rather than a filter bolted onto a finished executive: System-1 shapes the data and effort System-2 receives, and System-2's outcomes become System-1's supervision. 
AlphaZero illustrates this, since its play comes not from deeper search alone but from a policy-value network and a tree search that trained each other \citep{silver2018general}: a more capable System-2 raises the intelligence of accomplishing tasks, while a well-calibrated System-1 determines how much of it is reached per unit of cost and latency (See Fig.~\ref{fig:pareto3d}).

\paragraph{Static cascades versus adaptive co-design.}
The pairing can be built in two ways (See Fig.~\ref{fig:cascade}). 
In a \emph{static cascade}, a standalone fast model classifies each request and hands the hard minority to a large model; Jev, a commercial model described by its developers as a ``System One'' \citep{typesafe2026jev}, is a clear instance, returning calibrated probabilities over answers defined in advance \citep{osman2026jev}. Such products illustrate commercial interest in fast, structured judgment models compatible with the System-1 role, although a single product does not establish an industry-wide convergence.
A static cascade is, by construction, trained once, routes only at the inference layer, and does not receive System-2's outcomes, which suffices when its predictions are stationary but fails here, because how informative an interaction is \emph{for the current System-2} changes with every update. 
In \emph{co-design}, System-1 instead predicts utility such as difficulty or learning value relative to the \emph{current} System-2; it acts in training as well as at inference, scheduling compute and storage such as rollout budgets, replay buffers, batches, key--value (KV) caches, context windows, rather than a single call path.
System-2's realized rewards, gradients, and outcomes recalibrate System-1 online (See Fig.~\ref{fig:flow}). 
A cascade retrained online on its outcomes moves toward co-design, so the distinction is not the presence of two models but feedback and joint optimization.

\begin{figure}[t]
\centering
\includegraphics[width=\linewidth]{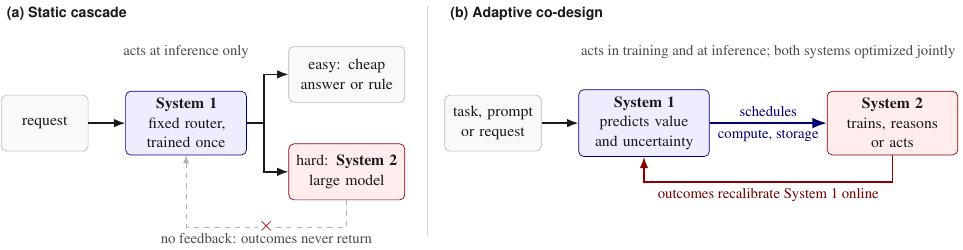}
\caption{\textbf{Static cascades versus adaptive co-design.} (a) A standalone fast model routes each request once and is never updated by the outcomes. (b) In co-design, System~1 schedules System~2's compute and storage, System~2's outcomes recalibrate System~1, and the two are optimized jointly, so System~1's judgements track an executive that keeps changing. Cascades retrained online on their outcomes lie between the two.}
\label{fig:cascade}
\end{figure}

\subsection{Core Definition}
\label{sec:formal-core}

\begin{definition}[Conditional Interaction Policy]
\label{def:ci-policy}
Let $\calI$ be the space of feasible interactions,
$\calH_t = \{(i_s, o_s)\}_{s=0}^{t-1}$ the history of interaction--outcome
pairs, and $\hat{u}:\calI\times\calH^*\to\R$ a predicted utility. 
A \emph{conditional interaction policy}
is the selection rule
\begin{subequations}\label{eq:ci-policy}
    \begin{align}
        \pi_{\mathrm{CI}}(i_t \mid \calH_{0:t})
  &\;=\;
  \frac{1}{Z_t}\,
  \exp\!\bigl(\beta\,\hat{u}(i_t \mid \calH_{0:t})\bigr)
  \,\cdot\,
  \indicator{\hat{u}(i_t \mid \calH_{0:t}) \ge \kappa_t},
  \\[2pt]
  Z_t &\;=\; \sum_{i\in\calI}
        \exp\!\bigl(\beta\,\hat u(i\mid\calH_{0:t})\bigr)\,
        \indicator{\hat u(i\mid\calH_{0:t})\ge\kappa_t},
    \end{align}
\end{subequations}
where $\beta>0$ controls selection sharpness, $\kappa_t$ is a utility threshold enforcing a minimum predicted utility, and $Z_t$ is a sum for finite $\calI$, or an integral against a base measure on a continuous $\calI$ to normalize $\pi_{\mathrm{CI}}$ over the admissible set $\{i:\hat u\ge\kappa_t\}$. 
When the admissible set is empty ($Z_t=0$), the policy selects an explicit abstention action $\varnothing$ with probability one and executes nothing at step $t$; $\pi_{\mathrm{CI}}$ is thus a distribution on $\calI\cup\{\varnothing\}$. The score $\hat u$ is a predicted \emph{gross} utility, which is the right ranking signal when execution costs are equal; when they differ, the policy is applied to the cost-adjusted score $\tilde u$ of Eq.~(\ref{eq:net-score}). Score-based results (the greedy limit, the threshold rule, the regret bounds and \cref{lem:margin}) then apply to $\tilde u$ provided the true score is redefined consistently as $u-\lambda\calC$, with relevance, error rates and regret all measured on that net score. Results on the IIE and on $\rho_u$ keep their gross-utility definitions unless explicitly redefined; a gate that is accurate for the net score need not have the same error rates with respect to $u$.
\end{definition}

\textit{Interpretation.}  
The indicator in Eq.~(\ref{eq:ci-policy}a) restricts the search space by suppressing candidates whose predicted utility falls too far below the current budget. 
At the same time, the exponential factor sets how sharply the policy concentrates on the best feasible candidates, and hence how much it explores among them. 
This work reserves \emph{gate} for this mechanism in prose and uses \emph{selection policy} to denote $\pi_{\mathrm{CI}}$. 
We call the resulting closed loop --- predict, gate, execute, and fold the outcome back into the predictor --- the \emph{predict-then-gate} architecture (See Fig.~\ref{fig:flow}); every instance shares it in this Perspective, whatever utility fills and whatever ledger it is scored against.

\paragraph{From soft gating to greedy selection: limiting regimes of the policy.}
The two parameters of Eq.~(\ref{eq:ci-policy}) play distinct roles: the threshold $\kappa_t$ decides \emph{which} interactions are admissible, and the inverse temperature $\beta$ decides \emph{how sharply} the policy prefers the best admissible one. The threshold is a statement about utility, not about budget; a budget enters through a price, since in the master objective, i.e., Eq.~(\ref{eq:master}), a multiplier $\lambda$ on cost turns the threshold into the cost-dependent rule $\hat u(i)\ge\lambda\calC(i)$, and the adaptive update Eq.~(\ref{eq:threshold-update}) is one way of setting $\kappa_t$ so that a target admission rate is met. Write $\hat u_t(i)=\hat u(i\mid\calH_{0:t})$, and let
\begin{equation}
  \calI_t^{\mathrm{adm}}=\{i\in\calI:\hat u_t(i)\ge\kappa_t\},
  \qquad
  \hat u_t^{\max}=\max_{i\in\calI}\hat u_t(i),
  \qquad
  \calI_t^\star=\argmax_{i\in\calI}\hat u_t(i)
  \label{eq:ci-admissible}
\end{equation}
denote the admissible set, the largest predicted utility, and the set of maximizers (assumed finite and attained).

\begin{proposition}[Degeneration to the greedy argmax]
\label{prop:ci-argmax}
Let $\calI$ be finite and $\kappa_t\le\hat u_t^{\max}$, so that every maximizer is admissible. As $\beta\to\infty$,
\begin{equation}
  \pi_{\mathrm{CI}}(i_t\mid\calH_{0:t})
  \;\longrightarrow\;
  \frac{\indicator{i_t\in\calI_t^\star}}{|\calI_t^\star|},
  \label{eq:ci-argmax-limit}
\end{equation}
the greedy rule $i_t=\argmax_{i}\hat u(i\mid\calH_{0:t})$ with ties broken uniformly. If the maximizer $i_t^\star$ is unique, with gap $\Delta_t>0$ to the best other admissible candidate, then
\begin{equation}
  1-\pi_{\mathrm{CI}}(i_t^\star\mid\calH_{0:t})
  \;\le\;
  \bigl(|\calI_t^{\mathrm{adm}}|-1\bigr)\,e^{-\beta\Delta_t},
  \label{eq:ci-argmax-rate}
\end{equation}
so, when $|\calI_t^{\mathrm{adm}}|\ge2$, $\beta\ge\Delta_t^{-1}\ln\bigl((|\calI_t^{\mathrm{adm}}|-1)/\delta\bigr)$ selects $i_t^\star$ with probability at least $1-\delta$; when $|\calI_t^{\mathrm{adm}}|=1$ it is selected with probability one for every $\beta$.
\end{proposition}

\noindent Any $\kappa_t\le\hat u_t^{\max}$ gives the same greedy limit, whereas $\kappa_t>\hat u_t^{\max}$ empties the admissible set and triggers abstention; a small threshold alone, at finite $\beta$, yields not the argmax but the unconstrained Boltzmann policy $\pi\propto e^{\beta\hat u_t}$, and as $\beta\to0$ the policy becomes a pure filter that admits candidates without ranking them. 

\begin{figure}[t]
\centering
\includegraphics[width=0.92\linewidth]{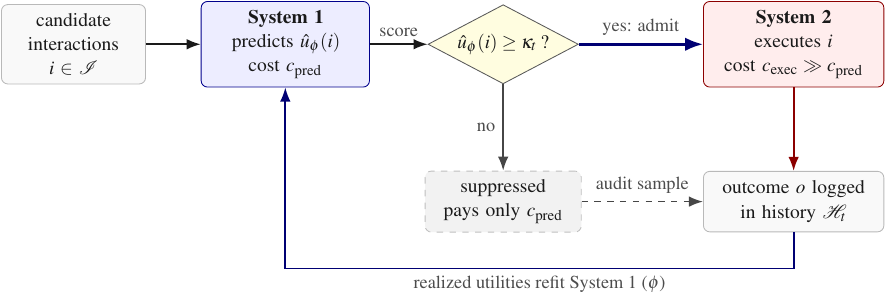}
\caption{\textbf{The predict-then-gate mechanism.} System~1 scores every candidate at cost $c_{\text{pred}}$, and the threshold $\kappa_t$ admits to System~2 only those predicted to be informative. Suppressed candidates cost only $c_{\text{pred}}$; executed outcomes, together with a small audit sample of suppressed candidates, are logged and refit System~1 (\cref{prop:amortized-screening}).}
\label{fig:flow}
\end{figure}

\begin{definition}[Interaction Utility]
\label{def:utility}
As a running example we take the expected Bayesian surprise \citep{itti2005bayesian}: the \emph{utility} of interaction $i$ at executive state $\theta_t$ is the expected posterior-to-prior divergence it induces,
\begin{equation}
  u(i;\theta_t) \;\triangleq\; \E\!\left[\,\KL{p(\theta\mid\calH_{0:t},i,o)}{p(\theta\mid\calH_{0:t})} \;\big|\; \calH_{0:t}, i\,\right],
  \label{eq:utility-primitive}
\end{equation}
with the expectation over the predictive distribution of the outcome $o$. It answers: \emph{if I execute this interaction and observe its outcome, how much will my belief change?} Interactions with predictable outcomes score near zero; utility measures the magnitude of the belief update, not difficulty or novelty. It is the conditional mutual information $I(\theta;o\mid i,\calH_{0:t})$, which, when the relevant entropies are finite, equals the expected reduction in the entropy of $\theta$. The predictor $\hat u$ in Eq.~(\ref{eq:ci-policy}) estimates whatever utility fills this slot; information gain is one possible choice of $u$, not the universal definition of usefulness. Different contexts require different utilities:
\begin{equation}
  u(i) =
  \begin{cases}
    I(\theta;o\mid i) & \text{(exploration / experiment design),}\\[2pt]
    \E[\Delta J\mid i] & \text{(expected improvement in curriculum learning),}\\[2pt]
    \mathrm{VoI}(i)=\E[\Delta V\mid i] & \text{(decision making: value of information),}\\[2pt]
    \calR_{\text{adapt}}(i;\theta_t) & \text{(robust domain randomization / meta learning risk)}
  \end{cases}
  \label{eq:utility-typed}
\end{equation}
Their maximizers generally vary case by case. Eq.~(\ref{eq:utility-primitive}) coincides with the first line (\cref{rem:invariance}) and, combined with a cost, underlies the third; the fourth encodes adaptation risk, and the second is the objective that predictive task samplers estimate directly. Two conventions keep these lines consistent with the master objective of \cref{sec:master}. First, every $u$ is a \emph{gross} utility: cost never enters $u$ and is charged once, through the ledger $\calC(i)$ and its price $\lambda$, so the value of information becomes net of cost only through that charge. Second, the fourth line is an \emph{acquisition score}, not the training objective: a robust learner \emph{minimizes} the tail risk $\CVaR_\nu$ of Eq.~(\ref{eq:cvar}), and to do so it \emph{prioritizes} the tasks with the largest predicted adaptation risk $\calR_{\text{adapt}}$ of Eq.~(\ref{eq:adaptation-risk}), the score maximized in Eq.~(\ref{eq:phys-acquisition}). We do not merge these utilities. What they share is the predict-then-gate \emph{architecture}, and the efficiency results below (screening gain, regret bounds, rescaling) hold for any of them, given the predictor's ranking quality and the cost asymmetry.
\end{definition}

\begin{remark}[Representation invariance]
\label{rem:invariance}
The KL divergence is invariant under invertible reparameterizations of $\theta$, and $u(i;\theta_t)=I(\theta;o\mid i,\calH_{0:t})$ is also invariant under invertible transformations of $o$. Hence the utility, and the IIE of Eq.~(\ref{eq:IIE}), do not depend on the chosen coordinates.
\end{remark}

\begin{remark}[Scope of the utility]
\label{rem:utility-scope}
Eq.~(\ref{eq:utility-primitive}) measures belief change, not task success, descent direction, or robustness, and coincides with each only under conditions:
\emph{(a) inference success}, only under uniform outcome preferences (\cref{prop:active-inference});
\emph{(b) the post-training update direction}, only through Fisher (natural-gradient) preconditioning, so an uncertain but noisy interaction need not help \citep{qu2026moppps,zou2026uds};
\emph{(c) robust adaptation}, only when evaluated on the tail, as in the CVaR adaptation risk of Eqs.~(\ref{eq:adaptation-risk}) and~(\ref{eq:cvar}) in \cref{app:acquisition} \citep{qu2025fast,wang2025modelpredictivetasksampling}.
Conversely, an interaction that barely moves the posterior can still be valuable: it may confirm a belief, test robustness, improve calibration, avert a rare failure, or inform another agent. Information gain, learning progress, decision value and task improvement are therefore distinct quantities, and conditional interaction is always relative to whichever one is declared.
\end{remark}

\subsection{The Interaction Information Efficiency Metric}

\begin{definition}[Interaction Information Efficiency]
For a stream of $T$ screened candidates $i_0,\dots,i_{T-1}$, the \emph{Interaction Information Efficiency} (IIE) is the utility obtained per unit of system-level cost,
\begin{equation}
  \mathrm{IIE}_T
  \;=\; \frac{%
      \sum_{t=0}^{T-1}\,
        \mathbb{I}(i_t)\,\indicator{i_t \text{ executed}}
    }{%
      \sum_{t=0}^{T-1}\,
        \bigl(c_{\text{pred}}(i_t)+\calC(i_t)\,\indicator{i_t \text{ executed}}\bigr)
    },
  \label{eq:IIE}
\end{equation}
where $\mathbb{I}(i_t)=u(i_t;\theta_t)$ is the declared utility (the expected Bayesian surprise of Eq.~(\ref{eq:utility-primitive}) in the canonical case), $\calC(i_t)$ the execution cost, and $c_{\text{pred}}(i_t)$ the cost of screening the candidate, paid whether or not it is executed (zero for unconditional interaction); it is the full per-candidate screening cost, encoding plus scoring, which \cref{prop:amortized-screening} splits as $c_{\text{screen}}=c_{\text{embed}}+c_{\text{score}}$. Predictor refitting and audit executions are amortized system costs, reported separately (\cref{sec:amortization}). Because a ratio can be raised by obtaining very little at very low cost, IIE is always reported together with the \emph{yield}, the utility obtained per screened candidate,
\begin{equation}
  Y_T \;=\; \frac{1}{T}\sum_{t=0}^{T-1}\mathbb{I}(i_t)\,\indicator{i_t \text{ executed}}.
  \label{eq:yield}
\end{equation}
Both are empirical estimates, over a finite stream, of population quantities for candidates $i\sim\calD$; for the IIE,
\begin{equation}
  \mathrm{IIE}_{\mathrm{pop}}
  \;=\;\frac{\E\bigl[\mathbb{I}(i)\,\indicator{i\text{ executed}}\bigr]}{\E\bigl[c_{\text{pred}}(i)+\calC(i)\,\indicator{i\text{ executed}}\bigr]},
  \label{eq:IIE-pop}
\end{equation}
and $\mathrm{IIE}_T\to\mathrm{IIE}_{\mathrm{pop}}$ almost surely for an i.i.d.\ stream, by the strong law of large numbers applied to numerator and denominator. A ratio of random sums is not the ratio of their expectations at finite $T$, so the results below are stated for $\mathrm{IIE}_{\mathrm{pop}}$; they hold exactly for a fixed pool when fractions and class means are taken over that pool.
\end{definition}

\paragraph{Notation: $\rho_u$ in general, $\rho_{\mathrm{info}}$ as a special case.}
Every proposition, theorem and corollary below is stated for a general declared utility $u$ and its \emph{density of utility-relevant interactions} $\rho_u$, the fraction of candidates whose true utility clears a floor $u_{\min}$ fixed in advance (Eq.~(\ref{eq:rho-operational})); candidates below the floor are called \emph{dead}. The informative interaction density $\rho_{\mathrm{info}}$ is the special case in which $u$ is the expected information gain of Eq.~(\ref{eq:utility-primitive}): there ``utility-relevant'' reads ``informative'', $\mathbb{I}$ is the expected Bayesian surprise, and utilities are measured in nats. No result below relies on that choice unless it says so.

Partition the candidate pool into the utility-relevant fraction $\rho_u$, with mean utility $\bar{\mathbb{I}}_{\text{rel}}$ and mean cost $\bar\calC_{\text{rel}}$, and the dead fraction $1-\rho_u$, with mean utility $\bar{\mathbb{I}}_{\text{dead}}$ and mean cost $\bar\calC_{\text{dead}}$. Pool averages are $\bar{\mathbb{I}}=\rho_u\bar{\mathbb{I}}_{\text{rel}}+(1-\rho_u)\bar{\mathbb{I}}_{\text{dead}}$ and $\bar\calC=\rho_u\bar\calC_{\text{rel}}+(1-\rho_u)\bar\calC_{\text{dead}}$, and unconditional interaction attains $\mathrm{IIE}^{\text{unc}}=\bar{\mathbb{I}}/\bar\calC$.

\begin{proposition}[IIE Gain]
\label{prop:iie-gain}
Assume dead candidates carry no utility ($\bar{\mathbb{I}}_{\text{dead}}=0$, as when $u_{\min}\to0^+$ for a non-negative utility), and use the system-level ledger of Eq.~(\ref{eq:IIE}), in its population form Eq.~(\ref{eq:IIE-pop}), which charges a predictor cost $c_{\text{pred}}$ to every screened candidate.
\begin{enumerate}[leftmargin=1.6em]
\item \emph{(Perfect, free predictor.)} If exactly the utility-relevant candidates are executed and $c_{\text{pred}}=0$,
\begin{equation}
  \frac{\mathrm{IIE}^{\mathrm{CI}}}{\mathrm{IIE}^{\text{unc}}}
  \;=\;
  \frac{\bar{\mathbb{I}}_{\text{rel}}/\bar{\calC}_{\text{rel}}}{\bar{\mathbb{I}}/\bar{\calC}}
  \;=\;
  \frac{1}{\rho_u}\cdot\frac{\bar{\calC}}{\bar{\calC}_{\text{rel}}}
  \;=\;
  1+\frac{1-\rho_u}{\rho_u}\cdot\frac{\bar\calC_{\text{dead}}}{\bar\calC_{\text{rel}}},
  \label{eq:IIE-ratio}
\end{equation}
which equals $1/\rho_u$ when $\bar\calC_{\text{dead}}=\bar\calC_{\text{rel}}$ and exceeds it if and only if $\bar\calC_{\text{dead}}>\bar\calC_{\text{rel}}$.
\item \emph{(Imperfect, costly predictor.)} With class-conditional false-positive rate $\mathrm{fp}$ (the fraction of dead candidates admitted) and false-negative rate $\mathrm{fn}$ (the fraction of utility-relevant candidates suppressed),
\begin{equation}
  \frac{\mathrm{IIE}^{\mathrm{CI}}}{\mathrm{IIE}^{\text{unc}}}
  \;=\;
  \frac{\bar{\mathbb{I}}_{\text{rel}}/\bar{\calC}_{\text{rel}}}{\bar{\mathbb{I}}/\bar{\calC}}
  \cdot
  \left(1+\frac{\mathrm{fp}\,(1-\rho_u)\,\bar\calC_{\text{dead}}+c_{\text{pred}}}{(1-\mathrm{fn})\,\rho_u\,\bar\calC_{\text{rel}}}\right)^{-1},
  \label{eq:IIE-ratio-imperfect}
\end{equation}
which is non-increasing in each of $\mathrm{fp}$, $\mathrm{fn}$ and $c_{\text{pred}}$, strictly decreasing in any of them whose coefficient is positive, and recovers Eq.~(\ref{eq:IIE-ratio}) as $\mathrm{fp},c_{\text{pred}}\to0$.
\end{enumerate}
\end{proposition}

False negatives do not dilute the IIE \emph{ratio} directly, since suppressed interactions are simply not paid for; they enter by spreading screening and false-positive costs over fewer utility-relevant executions and, more importantly, by reducing the yield Eq.~(\ref{eq:yield}) by the factor $1-\mathrm{fn}$. Efficiency and yield must therefore be reported together: an aggressive gate can raise IIE while starving the learner of signal. With $\rho_u\ll1$, small error rates, and $c_{\text{pred}}\ll\bar\calC_{\text{rel}}$, the realized multiplier remains substantial, consistent with the several-fold rollout savings reported for predictive prompt sampling \citep{qu2026moppps}, but these are ideal values, not guarantees. Under equal execution costs and perfect screening, $1/\rho_u$ is the benchmark for the IIE gain; with heterogeneous costs the ideal gain is Eq.~(\ref{eq:IIE-ratio}) itself, which exceeds $1/\rho_u$ when dead interactions cost more than relevant ones. The zero-utility assumption on dead candidates matters once $u_{\min}>0$: if they retain a mean utility $\bar{\mathbb{I}}_{\text{dead}}>0$, the first equality of Eq.~(\ref{eq:IIE-ratio}) still holds, but the ideal ratio becomes $\frac{1}{\rho_u}\frac{\bar\calC}{\bar\calC_{\text{rel}}}\bigl(1-(1-\rho_u)\bar{\mathbb{I}}_{\text{dead}}/\bar{\mathbb{I}}\bigr)$, so the closed forms of Eq.~(\ref{eq:IIE-ratio}) become upper bounds (\cref{app:iie-gain}), and $1/\rho_u$ bounds the gain only when, in addition, $\bar\calC_{\text{dead}}\le\bar\calC_{\text{rel}}$.

\begin{remark}[Cost-correlated errors]
\label{rem:cost-conditioned}
Eq.~(\ref{eq:IIE-ratio-imperfect}) assumes error rates independent of cost. If errors concentrate on expensive interactions, the false-positive cost becomes $\int\mathrm{fp}(\calC)\,\calC\,dF_{\text{dead}}(\calC)>\mathrm{fp}\,\bar\calC_{\text{dead}}$, so validation should report cost-conditioned rates.
\end{remark}

\begin{remark}[Identifiability of $\rho_u$]
\label{rem:identifiability}
Fitting $\rho_u$ to a loss curve is circular, since $\calL=a(\rho(\calC)\calC)^{-\alpha'}$ with $\rho(\calC)=\calC^{s}$ fits any exponent. We therefore measure it independently: for a declared utility $u$ and a threshold $u_{\min}>0$ fixed \emph{a priori},
\begin{equation}
  \rho_u(\theta_t)
  \;=\;
  \Pr_{i\sim\calD}\!\bigl[\,u(i;\theta_t)\ge u_{\min}\,\bigr],
  \label{eq:rho-operational}
\end{equation}
with $u$ the \emph{true} utility estimated from executed interactions, not the predictor's $\hat u_\phi$; $\rho_{\mathrm{info}}$ denotes $\rho_u$ for the information-gain utility, and densities under different utilities are not comparable. In RLVR the natural declared utility is the probability that a group of $G$ rollouts has mixed rewards, $u_{\mathrm{adv}}(\tau;\theta_t)=1-p_\tau^G-(1-p_\tau)^G$ with $p_\tau$ the current pass rate; with $u_{\min}\to0^+$, $\rho_u$ becomes $\rho_{\mathrm{adv}}$, the fraction of prompts that can yield a non-zero GRPO advantage. This is a proxy for useful policy-gradient signal, not for information gain: an all-fail group still informs a posterior over the prompt's difficulty, and a mixed group can give a noisy update, so $\rho_{\mathrm{adv}}$ and $\rho_{\mathrm{info}}$ are distinct densities. In the physical regime $\rho_u$ is the fraction of tasks whose adaptation risk exceeds $u_{\min}$. Because $\rho_u$ depends on $(u,u_{\min})$, the meaningful object is its relative decay $\rho_u(\theta_t)/\rho_u(\theta_0)$. We propose its agreement across thresholds and across proxies such as the expected gradient norm or the learnability score $\partial_t\E[r\mid\tau,\theta_t]$ as a robustness diagnostic, not a theoretical invariance: tail probabilities at different thresholds need not decay proportionally, and gradient norm and learning progress are not equivalent. Measurement also needs a protocol. Because the gate censors what is observed (\cref{sec:selective-supervision}), $\rho_u$ is estimated on an audit pool drawn from $\calD$ independently of the gate, with repeated rollouts per candidate, confidence intervals, and re-estimation as $\theta_t$ changes. 
\end{remark}

\subsection{A Master Objective and Its Special Cases}
\label{sec:master}

Conditional interaction provides a common predict-then-gate \emph{architecture} for online prompt selection in LLM post-training, active domain randomization for Sim2Real transfer, and adaptive experimental design in AI for science. 
The unification is one of architecture, not of objective. 
What elevates the family from heuristics to a scaling lens is (i) \emph{suppression as a first-class objective}, (ii) \emph{predictive amortization}, i.e., estimating utility before execution, and (iii) \emph{generality} across interaction types, cost structures, policy representations, and reward topologies. 
On its own, this is a family resemblance; to formalize the claim, we state one optimization problem and obtain each technique by instantiating its variables, including the utility that fills the gating slot.

\begin{definition}[Master Conditional-Interaction Objective]
\label{def:master}
Let $\calI$ be a finite menu of candidates with predicted utility $\hat u(i)$ and cost $\calC(i)>0$, and let $x(i)\in[0,1]$ be the probability of executing candidate $i$. The master objective maximizes expected executed utility under a cost budget $\Gamma$:
\begin{equation}
  x^\star
  \;=\;
  \argmax_{x\in[0,1]^{\calI}}\;
  \sum_{i\in\calI} x(i)\,\hat u(i)
  \qquad\text{s.t.}\qquad
  \sum_{i\in\calI} x(i)\,\calC(i)\le \Gamma.
  \label{eq:master}
\end{equation}
This is a fractional knapsack. For a multiplier $\lambda\ge0$ its Lagrangian $\sum_i x(i)\bigl(\hat u(i)-\lambda\calC(i)\bigr)$ is maximized by the threshold rule $x^\star(i)=\indicator{\hat u(i)\ge\lambda\calC(i)}$, i.e.\ by admitting candidates in order of \emph{value per unit cost} $\hat u(i)/\calC(i)$, with $\lambda$ fixed by the budget (at most one candidate is fractional at the optimum). The threshold rule is the gate of Eq.~(\ref{eq:ci-policy}) with the cost-dependent threshold $\lambda\calC(i)$, which reduces to a common $\kappa_t=\lambda\bar\calC$ when costs are equal. Adding an entropy regularizer $\beta^{-1}\sum_i H_b(x(i))$ softens the gate into $x^\star(i)=\sigma\bigl(\beta(\hat u(i)-\lambda\calC(i))\bigr)$, with $\sigma$ the logistic function, of which Eq.~(\ref{eq:ci-policy}) applied to the cost-adjusted score is the normalized sampling analogue. A cardinality constraint $\sum_i x(i)=B$ with integral $x$ gives the top-$B$ batch problem, which differs from independent thresholding only when utilities interact (redundancy among co-selected candidates), captured by the diversity term of Eq.~(\ref{eq:subset-selection}).
\end{definition}

\paragraph{Gross utility, net value, and three selection rules.} With the cost-adjusted (net) score
\begin{equation}
  \tilde u(i)\;=\;\hat u(i)-\lambda\,\calC(i),
  \label{eq:net-score}
\end{equation}
three rules must be kept apart when costs differ. (a) The \emph{budgeted batch} rule of Eq.~(\ref{eq:master}) admits every candidate with $\tilde u(i)\ge0$, i.e.\ with value per unit cost $\hat u(i)/\calC(i)\ge\lambda$. (b) The \emph{single-draw net-value} rule selects $\argmax_i\tilde u(i)$, as in Eq.~(\ref{eq:net-utility}). (c) The policy of Eq.~(\ref{eq:ci-policy}) applied to the raw score $\hat u$ selects $\argmax_i\hat u(i)$ in its greedy limit (\cref{prop:ci-argmax}). When costs are equal the three rank candidates identically, by $\hat u$, although (a) selects a set and (b) and (c) a single candidate. When they differ, (c) can contradict (a) and (b): the candidate with the largest raw utility need not have the largest net value or value per unit cost. Applying Eq.~(\ref{eq:ci-policy}) to $\tilde u$ with $\kappa_t=0$ removes the conflict: its admissible set is exactly the solution of (a), up to the one fractional item, and its greedy limit is (b). We therefore use the raw score only when costs are equal, and $\tilde u$ otherwise.

\noindent The techniques discussed here arise by choosing $\calI$, $\hat u$, and $\calC$. In \textbf{agentic interaction selection} (\cref{sec:domains,sec:tts,sec:agentic}), $\calI$ is the task or prompt pool, the reasoning increments available for a query, or the tool/think/ask menu; $\hat u$ is adaptation risk, prompt informativeness, or expected reduction in answer uncertainty; and $\calC$ is rollout cost, inference compute, or a composite of latency, money, tokens, and risk, so that Eq.~(\ref{eq:master}) becomes risk-prioritized sampling, adaptive halting, or the agentic gate. In the \textbf{orchestration of System-2 executives} (\cref{sec:architecture}), System-1 interprets a task semantically before any of it is executed and can therefore also decide \emph{who} should execute it and \emph{how much} they may spend: $\calI$ is the pool of available System-2 models or agents together with their compute budgets, $\hat u$ their predicted contribution, and $\calC$ their inference cost, so Eq.~(\ref{eq:master}) is a joint routing and budget-allocation rule; token-level mixture-of-experts layers and dynamic depth are the same decision at finer granularity. Eq.~(\ref{eq:master}) captures the \emph{selection} structure these methods share, not the interchangeability of their learning dynamics or guarantees.

\subsection{Scope: When Gating Is and Is Not Justified}
\label{sec:exploration}

The principle applies only where interactions are individually costly, heterogeneous in informativeness, and predictable before execution by a model cheaper than execution itself. Gating has no guaranteed advantage when candidates are homogeneous in utility, when predicting their utility costs about as much as executing them, or when execution is cheap. Dense pretraining, minibatch stochastic gradient descent over independent, identically distributed data, and static supervised learning often come close to these conditions, but need not: a fixed corpus can hold examples whose value differs and is predictable from cheap features, which data-selection methods exploit. A gain observed there calls for an accounting of where it comes from, not a refutation of the principle. Within that regime, gating is justified when a predictor exists whose cost is well below execution cost and whose error $\varepsilon_\phi$ is small relative to the utility gaps it must resolve (\cref{thm:regret}).

It should be relaxed, executing uncertain or even predicted-uninformative interactions, when:
(1) the predictor is poorly calibrated or off-distribution, since its blind spots are where it wrongly suppresses;
(2) interactions carry option or exploratory value not captured by one-step utility;
(3) prediction costs about as much as execution, voiding \cref{prop:amortized-screening};
(4) false negatives are catastrophic (safety-critical control, irreversible actions); or
(5) the utility is misspecified, so that the predictor systematically suppresses the very interactions that would reveal the misspecification.

In such regimes the right object is an \emph{exploratory} value of information, e.g.\ an upper-confidence or Thompson criterion, replacing $\hat u$ in Eq.~(\ref{eq:master}) with $\hat u+c\,\sigma_\phi$, which recovers pure exploitation as $c\to0$ and uncertainty-seeking exploration, dominated by $\sigma_\phi$, as $c\to\infty$. That limit favours the most uncertain candidates rather than sampling the pool uniformly; undirected exploration needs a separate mixture with a uniform policy, as in \cref{sec:selective-supervision}; the uncertainty-aware acquisition functionals of Eqs.~(\ref{eq:phys-acquisition}) and~(\ref{eq:digital-acquisition}) (\cref{app:acquisition}) are the mechanism by which the principle avoids over-pruning. Conditional interaction is a statement about how to spend a \emph{trustworthy} informativeness signal, not a claim that one always exists. Regimes (1) and (5) are the most insidious, since the predictor's errors concentrate precisely where it has no data and the interactions that would reveal them are never executed; the audit sampling described after \cref{prop:amortized-screening} is the minimal safeguard, and \cref{sec:selective-supervision} treats this censoring formally.

\section{Co-Designing the Two Systems: Architecture and Co-Scaling}
\label{sec:architecture}

\Cref{sec:formal} specified what System-1 decides and how its decisions are scored. Once the predictor exists, three questions follow: how should it be trained, how accurate must it be, and how should its capacity scale with the executor? We address them in turn.

\subsection{The Two Systems as an Operator Decomposition}

Let $\theta$ and $\phi$ be the executive and predictive parameters. Define
\begin{align}
  \mathcal{T}_{\text{exec}}(\theta,\phi;\,\calH)
    &= \Bigl(\theta - \eta_\theta\,
       \E_{i\sim\pi_{\mathrm{CI}}(\cdot\mid\calH;\phi)}\!\left[
         \nabla_\theta\calL(\theta,i)\right] + \xi_\theta,\;\;\phi\Bigr),
    \label{eq:T-exec}\\
  \mathcal{T}_{\text{pred}}(\theta,\phi;\,\calH)
    &= \Bigl(\theta,\;\;\phi - \eta_\phi\,
       \nabla_\phi\,
       \E_{i\in\calH}\!\bigl[
         d(\hat{u}_\phi(i;\calH),\,u(i;\theta))\bigr] + \xi_\phi\Bigr),
    \label{eq:T-pred}
\end{align}
where $d$ is a strictly proper scoring rule (e.g.\ squared error or log loss) and $\xi_\theta,\xi_\phi$ are stochastic-gradient noise. $\mathcal{T}_{\text{exec}}$ improves the policy along gradients sampled through the gate; $\mathcal{T}_{\text{pred}}$ fits the predictor to realized utilities. Each operator maps the pair $(\theta,\phi)$ to a pair and updates one component, so their composition is well defined: System-1 is refitted first, and System-2 then steps along gradients sampled through the refitted gate.

\begin{definition}[Conditional-Interaction Operator Decomposition]
\label{def:operator-decomp}
A \emph{conditional-interaction architecture} is any learning system whose update can be written as
\begin{equation}
  (\theta_{t+1},\phi_{t+1})
  \;=\;
  \bigl(\mathcal{T}_{\text{exec}}\circ\mathcal{T}_{\text{pred}}\bigr)(\theta_t,\phi_t;\calH_{0:t}),
  \label{eq:operator-decomp}
\end{equation}
with $\pi_{\mathrm{CI}}$ satisfying \cref{def:ci-policy} for $\hat{u}=\hat{u}_\phi$.
\end{definition}

This decomposition is the load-bearing object of co-design: neither operator is defined without the other. It is not a specific algorithm but a class of algorithms, indexed by latent space, predictor family, and scoring rule, of which the predictive-sampling systems cited in \cref{sec:digital} are special cases. The two operators run on different clocks: System-1 ($\phi$) adapts quickly to realized utilities, while System-2 ($\theta$) improves slowly along the gated gradients. Because better predictions change which interactions are executed, and hence the supervision System-1 receives, one may worry that the coupled loop diverges; two-timescale stochastic approximation \citep{borkar2008} gives conditions under which such coupled updates converge, and MPTS trains the two systems jointly in this way in practice \citep{wang2025modelpredictivetasksampling}, but a convergence guarantee for the general architecture remains open.

\paragraph{Co-evolution: how the two systems learn together.}
Co-evolution concerns how the two systems learn together, $(\theta_t,\phi_t)\to(\theta_{t+1},\phi_{t+1})$, as when an agent's execution harness and the policy running inside it are evolved jointly \citep{chen2026harnessforge}, and is dynamical; co-scaling (\cref{sec:co-scaling}) concerns how System-1's optimal capacity responds to System-2's, and is a comparative-statics problem. As a heuristic reference point for the former, one may expect the predictor's error to decrease with the amount of executed history and the quality of supervision,
\begin{equation}
  \varepsilon_\phi^{(t)}
  \;\lesssim\;
  \frac{c(\theta_t)}{\sqrt{\abs{\calH_{0:t}}}},
  \label{eq:accuracy-growth}
\end{equation}
where $c(\theta_t)$ would decrease as a stronger policy yields cleaner outcomes, and a better predictor in turn makes selection sharper. We write the familiar $O(|\calH_{0:t}|^{-1/2})$ form, borrowed from independent, identically distributed sampling, only as a benchmark: under adaptive, censored collection the rate need not hold, because the predictor's training data are gathered by a changing policy and censored by the gate itself (\cref{sec:selective-supervision}), and a stronger System-2 can make utilities harder rather than easier to predict. Where the loop does work, it resembles self-play along the informativeness axis rather than adversarial reward, and the separation of timescales is what keeps it stable.

\paragraph{Ranking by the posterior mean: using a predictor of given quality.}
Let $U$ be the unknown utility of a candidate and $\hat U$ the predictor's output. Intuitively, the predictor should compress history into just enough information about $U$: too expressive a predictor wastes parameters and erodes the amortization benefit, while too coarse a predictor cannot discriminate informative from uninformative interactions. An information-bottleneck (IB) objective trades complexity against informativeness,
\begin{equation}
  \min_{q(\hat U\mid\calH)}\;
  \underbrace{I(\hat U;\,\calH)}_{\text{predictor complexity}}
  \;-\; \beta_{\mathrm{IB}}\,
  \underbrace{I(\hat U;\,U)}_{\text{prediction informativeness}}.
  \label{eq:IB}
\end{equation}
Whatever representation $\hat U$ is learned, the question of how to \emph{use} it has a clean answer.

\begin{proposition}[Rank by the Posterior Mean]
\label{prop:IB}
Assume equal costs, a budget of $B$ executions, and that each candidate's utility depends on the predictor outputs only through its own, $\E[U(i)\mid\hat U]=\E[U(i)\mid\hat U(i)]$. Then, among all selection rules that see only the predictor outputs $\hat U$, executing the $B$ candidates with the largest posterior mean $\E[U\mid\hat U]$ maximizes expected executed utility and hence the IIE. Setting $\kappa_t$ to the $B$-th largest posterior mean makes the admissible set of Eq.~(\ref{eq:ci-policy}) exactly that top-$B$ set, up to ties at the $B$-th value.
\end{proposition}

\noindent The proposition is modest: it says nothing about which $\hat U$ to learn, which is what Eq.~(\ref{eq:IB}) addresses. Its practical content is that the calibrated mean of the predictor, not an uncalibrated score, is the right ranking signal.

\subsection{Co-Scaling: System-1's Size as System-2 Grows}
\label{sec:co-scaling}

Current practice sizes System-2 and treats any selector as an afterthought. Eq.~(\ref{eq:IIE-ratio-imperfect}) suggests this is incomplete: a larger System-1 lowers its error rates but raises $c_{\text{pred}}$, and where the balance lies depends on how expensive System-2 is. Let System-1 have a size $s$ from a totally ordered set $\calS_1$ (parameters, depth, or context), with error rates $\mathrm{fp}(s)$, $\mathrm{fn}(s)$ and scoring cost $c_{\text{pred}}(s)$, and let scaling System-2 multiply every execution cost by a factor $k\ge1$, so that $\bar\calC_{\text{rel}}\mapsto k\bar\calC_{\text{rel}}$ and $\bar\calC_{\text{dead}}\mapsto k\bar\calC_{\text{dead}}$.

\begin{proposition}[{Co-Scaling Under Fixed Error Curves}]
\label{prop:co-scaling}
Assume the setting of \cref{prop:iie-gain}(ii); that $\rho_u$, $\mathrm{fp}(s)$ and $\mathrm{fn}(s)<1$ do not depend on $k$; that $c_{\text{pred}}(s)/(1-\mathrm{fn}(s))$ is non-decreasing in $s$; and that the IIE ratio Eq.~(\ref{eq:IIE-ratio-imperfect}) attains its maximum over $\calS_1$ for every $k$, on a set $S^\star(k)$ of optimal System-1 sizes. Then (i) $S^\star$ is non-decreasing in $k$ in the strong set order: if $k'>k$, $s\in S^\star(k)$ and $s'\in S^\star(k')$, then $\max(s,s')\in S^\star(k')$ and $\min(s,s')\in S^\star(k)$, so the largest and the smallest optimal sizes, where they exist, are non-decreasing in $k$; and (ii) the maximal IIE ratio is non-decreasing in $k$.
\end{proposition}

\noindent The mechanism the proof isolates is that scaling System-2 dilutes the predictor's own cost, $c_{\text{pred}}/k$ relative to execution, but not the cost of its mistakes, which grows with System-2 itself, so accuracy becomes relatively cheaper to buy. Under these assumptions, a System-1 that is optimal for one System-2 remains optimal or becomes too small, never too large, as System-2's cost scale grows; the proposition does not say how large the resulting loss is. The result is comparative statics under fixed error curves: it does not show that growing System-1 helps an evolving executive, nor at what rate System-1 should grow, and both remain experimental hypotheses. Its assumptions do most of the work, and it does not imply the same monotonicity when scaling System-2 changes the prediction problem itself, as happens whenever a stronger executive changes the distribution of utilities, the separation between useful and dead interactions, or the cost structure. The statement is also qualitative: it fixes the direction of co-scaling, not its rate, since $\mathrm{fp}(s)$, $\mathrm{fn}(s)$ and $c_{\text{pred}}(s)$ are empirical curves. Measuring them, and hence a compute-optimal ratio between the two systems in the spirit of the parameter-to-data ratio of \citet{hoffmann2022training}, is the co-design counterpart of a scaling law (\cref{sec:implications}).

\subsection{Example of Architectural Implementation}

The decomposition is realized by four domain-agnostic components, summarized by four verbs (See Fig.~\ref{fig:architecture}). \emph{Represent:} candidates $\{i_k\}_{k=1}^n$ are embedded by $e:\calI\to\R^{d_e}$, e.g., frozen language-model encoders in the digital domain; graph or transformer encoders over the structure of the task's Markov decision process in the physical one. \emph{Predict (System-1):} a lightweight $f_\phi:\R^{d_e}\times\calH^*\to\R\times\R_+$ outputs $(\hat\mu_k,\hat\sigma_k)$, trained online on executed outcomes, with $d_\phi\ll d_\theta$ keeping $c_{\text{pred}}\ll c_{\text{exec}}$; System-1 may itself be a small language model when semantic understanding is required, but its defining property is computational role rather than model family. \emph{Gate:} the selection policy $\pi_{\mathrm{CI}}$ produces an execution mask $\mathbf{z}_t\in\{0,1\}^n$, and when the number executed is set by the threshold rather than fixed, $\kappa_t$ tracks a target admission rate $\rho^*$ by stochastic approximation,
\begin{equation}
  \kappa_{t+1}
  \;=\; \kappa_t
  + \alpha_\kappa\!\left(\frac{n_t^{\text{adm}}}{n}-\rho^*\right),
  \qquad
  n_t^{\text{adm}}=\bigl|\{i:\hat u_\phi(i)\ge\kappa_t\}\bigr|,
  \label{eq:threshold-update}
\end{equation}
which treats the threshold as a dual variable enforcing that rate; when costs differ, $\tilde u_\phi$ of Eq.~(\ref{eq:net-score}) replaces $\hat u_\phi$ in $n_t^{\text{adm}}$, and the equal-cost form is shown.  \emph{Learn:} executed interactions are verified, recorded in $\calH_t$, and used for one step of $\mathcal{T}_{\text{pred}}$, with a decay factor on old history trading stability against tracking non-stationarity. These four components are the operational form of \cref{def:operator-decomp}, and the ranking rule of \cref{prop:IB} tells the gate how to use the predictor's output.

\begin{figure}[t]
\centering
\includegraphics[width=\linewidth]{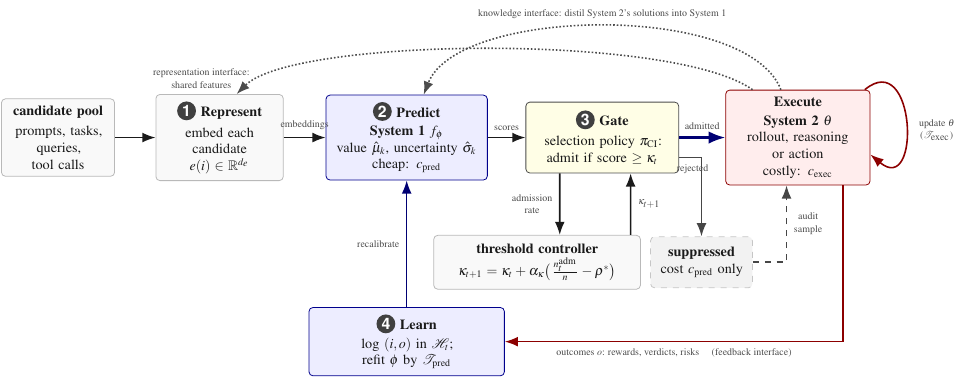}
\caption{\textbf{Architectural realization: represent, predict, gate, learn.} Every candidate is embedded \stepnum{1} and scored cheaply by System~1 \stepnum{2}. The gate \stepnum{3} admits only candidates whose score clears the threshold $\kappa_t$, which a controller adjusts to hold a target admission rate $\rho^*$ (Eq.~(\ref{eq:threshold-update})). Admitted candidates are executed by System~2; suppressed ones cost only $c_{\text{pred}}$, and a small audit sample of them is executed anyway so that the predictor's misses can be detected. Executed outcomes are logged and refit System~1 \stepnum{4}, while System~2 learns from the same outcomes. Dotted arcs mark the two other interfaces of co-design: what is distilled from System~2 into System~1, and which of System~2's representations System~1 may access.}
\label{fig:architecture}
\end{figure}

\paragraph{Three interfaces of co-design.}
Beyond the gate, co-design fixes three interfaces. The \emph{knowledge interface} determines what is distilled from System-2 into System-1: where System-2 repeatedly solves a class of candidates, the solution itself, not only its value, can be amortized into System-1, as expert iteration does \citep{anthony2017thinking}, so that System-1 answers routine cases outright and escalates the rest. The \emph{representation interface} determines what System-1 observes: using System-2's own features buys accuracy at the price of sharing its blind spots, whereas encoding the candidate independently keeps the two error profiles uncorrelated, and conditioning System-1 on context lets it adapt to a new task, scenario or user without retraining System-2. The \emph{feedback interface} determines which System-2 outcomes supervise System-1 --- realized rewards, gradients, verifier verdicts, or an audit sample of suppressed candidates --- and thereby fixes which of its own errors System-1 can ever observe (\cref{sec:selective-supervision}). Co-design is therefore not merely a size-matching problem. The same position also makes System-1 a natural place for oversight: because it sees every candidate before System-2 acts, cost, irreversibility and safety constraints (Eq.~(\ref{eq:agent-cost})) can be imposed before a capable executive commits.

\paragraph{Gates within gates: hierarchical extensions.}
The same gate can operate at several temporal and representational scales at once. For a hierarchy of levels $\ell\in\{1,\ldots,L_h\}$ (motor, skill, plan), each with its own predictor and threshold,
\begin{equation}
  \pi_{\mathrm{CI}}^{(\ell)}\!\bigl(i_t^{(\ell)}\mid\calH_{0:t}^{(\ell)}\bigr)
  \;\propto\;
  \exp\!\bigl(\beta_\ell\,\hat u_\phi^{(\ell)}\bigr)\,
  \indicator{\hat u_\phi^{(\ell)}\ge\kappa_t^{(\ell)}},
  \label{eq:nested-ci}
\end{equation}
with each level gating access to the one below; each $\pi_{\mathrm{CI}}^{(\ell)}$ is normalized over its own admissible set and abstains when that set is empty, as in \cref{def:ci-policy}. Analysing the joint dynamics of such a multi-level hierarchy is left open.

\section{What Prediction Buys, and Why a Small System-1 Can Work}
\label{sec:theory}

\Cref{sec:formal,sec:architecture} defined the gate and the architecture built around it. Co-design pays only if System-1 is both much cheaper than System-2 and accurate enough to steer it; this section asks when both hold. \Cref{sec:amortization} states the core formal results: what prediction buys, in time and memory, and the regret a cheap predictor risks. The remaining subsections ask why a cheap System-1 can make a useful prediction about an expensive interaction at all, from three complementary views. These connections are interpretive rather than separate contributions; each supplies insight and, under its own assumptions, a guarantee, but none is a derivation of the framework. \emph{Information-theoretically} (\cref{sec:info-foundations}), the quantity gated on is expected information gain net of cost, with epistemic action selection as its limiting case. \emph{Generatively} (\cref{sec:generative}), an interaction is a transition of a world or policy model, and System-1 is an amortized projection of that model onto a single scalar, which is why valuation is cheaper than simulation. \emph{Control-theoretically} (\cref{sec:mpc}), executing only what the gate admits is receding-horizon control of a learner whose own dynamics are still being learned.

\subsection{Core Results: What Prediction Buys and What It Risks}
\label{sec:amortization}
\label{sec:advantages}

Predictive sampling has two structural advantages over any scheme that must execute interactions to evaluate them, both following from the fact that $\hat u_\phi$ scores a candidate from a cheap embedding \emph{without performing it}: one in time and one, conditionally, in memory. Against them stands one risk, that a cheap predictor selects the wrong interactions, which the end of this subsection bounds.

\paragraph{Screening efficiency.}
A reactive filter must execute an interaction before discovering that it was uninformative, paying full cost on the dead fraction. Predictive sampling replaces this with a forward pass of a small model, as in recent predictive prompt and task samplers \citep{qu2026moppps, mao2026dynamics, wen2026bridging}.

\begin{proposition}[Screening Gain Relative to Exhaustive Execute-then-Filter Evaluation]
\label{prop:amortized-screening}
Let a pool of $n$ candidates contain a fraction $\rho_u$ of utility-relevant ones, with execution cost $c_{\text{exec}}$, screening cost $c_{\text{screen}}=c_{\text{embed}}+c_{\text{score}}$ (encoding the candidate and scoring it; this is the $c_{\text{pred}}$ of Eq.~(\ref{eq:IIE})), and $r=c_{\text{screen}}/c_{\text{exec}}$. With a perfect screen, acquiring the utility-relevant subset costs
\begin{equation}
  \frac{\calC_{\mathrm{CI}}}{\calC_{\text{react}}}
  \;=\;
  \rho_u \;+\; r
  \label{eq:amortized-ratio}
\end{equation}
of what exhaustive execute-then-filter evaluation spends ($\calC_{\text{react}}=n\,c_{\text{exec}}$, the baseline that must execute every candidate to find all relevant ones), a speedup approaching $1/\rho_u$ as $r\to0$, provided a candidate's utility class cannot be known without executing or scoring it.
\end{proposition}

Under the shorthand of \cref{sec:generative} for a dense transformer the costs are concrete: evaluating a pool of $n$ candidates by rollout to obtain $B$ utility-relevant ones, versus screening them predictively, costs
\begin{align}
  \calC_{\text{unc}} &\;=\; \Theta\!\bigl(n\, G\, L\, d_{\text{model}}^2\bigr),
  \label{eq:unc-complexity}\\
  \calC_{\mathrm{CI}} &\;=\; \underbrace{\Theta(n\,c_{\text{embed}})}_{\text{encoding}}\;+\;\underbrace{\Theta(n\, d_\phi^2)}_{\text{prediction}}
  \;+\; \underbrace{\Theta(B\, G\, L\, d_{\text{model}}^2)}_{\text{execution}},
  \label{eq:ci-complexity}
\end{align}
with $d_\phi\ll d_{\text{model}}$, a cost ratio (now including encoding) of $B/n+\Theta\bigl((c_{\text{embed}}+d_\phi^2)/(GL\,d_{\text{model}}^2)\bigr)$: Eq.~(\ref{eq:amortized-ratio}) with $\rho_u=B/n$ and $r=\Theta\bigl((c_{\text{embed}}+d_\phi^2)/(GL\,d_{\text{model}}^2)\bigr)$. Sparse or fused kernels change the constants, not the argument. The encoding term need not be small. A frozen language-model encoder applied to a prompt of length $L_x$ costs $c_{\text{embed}}=\Theta(L_xd_{\mathrm{enc}}^2+L_x^2d_{\mathrm{enc}})$, negligible next to $G$ rollouts of length $L$ when $L_xd_{\mathrm{enc}}^2\ll GLd_{\text{model}}^2$ and paid once per candidate when embeddings are cached across the rounds of a static pool; for short rollouts, or when candidates must be re-encoded every round, it can dominate $r$. End to end, executing an audit fraction $a$ of suppressed candidates adds $a(1-\rho_u)$ to Eq.~(\ref{eq:amortized-ratio}), refitting the predictor adds $c_{\text{train}}/(n\,c_{\text{exec}})$ per round, and screening errors move the ratio as in Eq.~(\ref{eq:IIE-ratio-imperfect}); Eq.~(\ref{eq:amortized-ratio}) is the perfect-screen ceiling, not a realized speedup. The error rates separating the realized multiplier from this ceiling are measurable: periodically executing a small audit sample of \emph{suppressed} candidates estimates $\mathrm{fn}$, and the executed survivors estimate $\mathrm{fp}$; reporting these, and their cost-conditioned versions (\cref{rem:cost-conditioned}), turns Eq.~(\ref{eq:IIE-ratio-imperfect}) into a measured quantity.

\paragraph{Memory scaling.}
When evaluation requires real interaction, the candidate pool may also be capped by the memory for in-flight rollouts (KV cache for $G$ length-$L$ generations, replay buffers, simulator state). A predictively scored candidate contributes only an embedding $e(\tau)\in\R^{d_e}$, and embeddings can be streamed.

\begin{proposition}[Peak Rollout Memory Under Synchronous Evaluation]
\label{prop:memory-scope}
To select the top-$B$ of $n$ candidates, synchronous reactive evaluation needs peak memory $M_{\text{react}}(n)$, whereas predictive scoring streamed in micro-batches of size $b$ needs $M_{\mathrm{CI}}(n)$, where
\begin{align}
  M_{\text{react}}(n) &\;=\; \Theta\!\bigl(n\,G\,L\,d_{\text{model}}\bigr),
  \label{eq:react-memory}\\
  M_{\mathrm{CI}}(n) &\;=\; \Theta\!\bigl(B\,G\,L\,d_{\text{model}} \;+\; b\,d_e \;+\; B\bigr),
  \label{eq:ci-memory}
\end{align}
independent of $n$; the pool is then limited only by screening time $\Theta(n\,c_{\text{pred}})$ and score storage.
\end{proposition}

\begin{corollary}[Larger Pools Without Rollout Memory]
\label{cor:scope-quality}
The best true utility $U_{(1)}^n$ among $n$ i.i.d.\ candidates is non-decreasing in $n$: it converges to the essential supremum if utility is bounded, and grows as $n^{\xi}$ (in probability, up to a slowly varying factor) if $\Pr(U>u)=u^{-1/\xi}\ell(u)$. Under predictive sampling this gain costs no additional rollout memory (Eq.~(\ref{eq:ci-memory})), provided the predictor ranks the top of the pool correctly.
\end{corollary}

\noindent The memory advantage is conditional on the synchronous premise: a reactive filter that streams executions and keeps only scores also has $n$-independent memory, but pays $\Theta(n\,c_{\text{exec}})$ time, so the advantage then reduces to the screening gain. More generally, System-1 allocates storage as well as compute --- which sample groups stay in a replay buffer, which examples enter a capacity-limited batch \citep{zou2026uds}, which demonstrations occupy a context window \citep{wen2026you}, how much KV cache a request receives --- and the rule of Eq.~(\ref{eq:master}) applies unchanged, with memory in place of compute as the cost.

\paragraph{Selection regret.}
These savings help only if the cheap predictor selects nearly the right interactions.

\begin{definition}[Conditional Regret]
With $\mathcal{B}_t^*$ the oracle's top-$B$ set by true utility $u_t=u(\cdot;\theta_t)$ in the round-$t$ pool $\calI_t$ and $\hat{\mathcal{B}}_t$ the set selected by $\pi_{\mathrm{CI}}$ at round $t$,
\begin{equation}
  \calR_T(\pi_{\mathrm{CI}}, \pi^*)
  \;=\; \sum_{t=0}^{T-1}
    \Bigl[\,\sum_{i\in\mathcal{B}_t^*}u_t(i) - \sum_{i\in\hat{\mathcal{B}}_t}u_t(i)\Bigr].
  \label{eq:cond-regret}
\end{equation}
\end{definition}

\begin{theorem}[Regret Bound]
\label{thm:regret}
Assume (i) selection is the deterministic limit of the fixed-size batch rule Eq.~(\ref{eq:subset-selection}) with $\mu=0$ and $\calA$ the predicted utility $\hat u$, the cardinality-constrained form of the master objective Eq.~(\ref{eq:master}), so each round retains the top-$B$ set by $\hat u$; (ii) $\varepsilon_\phi=\sup_{0\le t<T}\sup_{i\in\calI_t}\abs{\hat u_t(i)-u_t(i)}$, a time-uniform error, with $\hat u_t(i)=\hat u_{\phi_t}(i\mid\calH_{0:t})$; and (iii) $\mathrm{gap}_{\min} = \min_{t}\min_{i\in\mathcal{B}_t^*,\,j\in\calI_t\setminus\mathcal{B}_t^*}\bigl(u_{t}(i)-u_{t}(j)\bigr)>0$. Then
\begin{equation}
  \calR_T(\pi_{\mathrm{CI}}, \pi^*)
  \;\le\;
  2\,T\,B\,\varepsilon_\phi\,\cdot\,
  \indicator{2\varepsilon_\phi \ge \mathrm{gap}_{\min}};
  \label{eq:regret-bound}
\end{equation}
in particular, the selection is exactly the oracle's whenever $2\varepsilon_\phi<\mathrm{gap}_{\min}$. More generally, with round-specific errors $\varepsilon_t=\sup_{i\in\calI_t}\abs{\hat u_t(i)-u_t(i)}$ and gaps $\mathrm{gap}_t$, $\calR_T\le\sum_{t<T}2B\varepsilon_t\,\indicator{2\varepsilon_t\ge\mathrm{gap}_t}$.
\end{theorem}

\begin{corollary}[High-Probability Regret]
\label{cor:regret-hp}
Suppose each round's pool has at most $n$ candidates; for every round $t$ and every candidate outside a blind-spot set $\calO_t\subseteq\calI_t$ with $|\calO_t|\le n_{\mathrm{ood}}$, the error $\hat u_t(i)-u_t(i)$ is zero-mean and $\sigma_\phi$-sub-Gaussian, while errors on $\calO_t$ are unrestricted; and the utility takes values in an interval of length $u_{\max}$ (e.g.\ $u\in[0,u_{\max}]$). Then, with probability at least $1-\delta$ and $\varepsilon_\phi^{\mathrm{eff}}=\sigma_\phi\sqrt{2\ln(2nT/\delta)}$,
\begin{equation}
  \calR_T(\pi_{\mathrm{CI}}, \pi^*)
  \;\le\;
  2\,T\,B\,\varepsilon_\phi^{\mathrm{eff}}\,
  \indicator{2\varepsilon_\phi^{\mathrm{eff}}\ge\mathrm{gap}_{\min}}
  \;+\; T\,\min(B,n_{\mathrm{ood}})\,u_{\max}.
  \label{eq:regret-hp}
\end{equation}
\end{corollary}

\noindent\textit{Interpretation.} The batch rule of assumption (i) is the cardinality-constrained counterpart of the single-draw greedy limit (\cref{prop:ci-argmax}), not a consequence of it. For a fixed positive selection margin, regret is exactly zero once the predictor is accurate relative to the separation of interactions ($2\varepsilon_\phi<\mathrm{gap}_{\min}$); regret is measured in the units of the declared utility, and $\mathrm{gap}_{\min}$ enters not as a denominator but as the threshold deciding whether any misranking can occur, with the bound degrading gracefully, linearly in $\varepsilon_\phi$ and $B$, once $2\varepsilon_\phi\ge\mathrm{gap}_{\min}$. In continuous interaction spaces the global gap may vanish, and $\mathrm{gap}_{\min}$ should then be interpreted as the margin at the $B$-th order statistic. The uniform error and strictly positive gap of \cref{thm:regret} are thus the clean idealization that exposes the mechanism; \cref{cor:regret-hp} is the version intended for realistic predictors, and its second term makes explicit that blind spots, not average error, dominate once the predictor is otherwise accurate. The theorem does not imply a self-improving loop: better predictions reduce regret and more informative interactions improve the policy, but a stronger System-2 can make its own utilities harder to predict (\cref{sec:architecture,sec:selective-supervision}).

\subsection{Understanding Through the Lens of Active Inference}
\label{sec:info-foundations}

\paragraph{Information net of cost: the utility of an interaction.}
Conditional interaction is resource allocation under uncertainty: at each decision point the system chooses whether to spend compute, environment, or inference resources on an interaction whose value is not yet known. Let $\theta$ be the latent state of knowledge (model parameters, latent task representations, or beliefs). Executing $i\in\calI$ given history $\calH_t$ yields an outcome $o$ and posterior $p(\theta\mid\calH_t,i,o)$. The \emph{expected information gain} is
\begin{equation}
  \calU_{\mathrm{IG}}(i\mid\calH_t)
  \;=\;
  \E_{o\sim p(o\mid i,\calH_t)}\!\left[
    \KL{p(\theta\mid\calH_t,i,o)}{p(\theta\mid\calH_t)}
  \right]
  \;=\; I(\theta;o\mid i,\calH_t),
  \label{eq:exp-info-gain}
\end{equation}
the utility of \cref{def:utility}. \emph{Bayesian surprise} is the realized, per-outcome divergence measured \emph{after} execution; expected information gain is its expectation under the predictive distribution, which must be estimated \emph{before} execution and is what the gate acts on. When outcomes are drawn from the same predictive model, realized surprises on executed interactions are therefore unbiased targets for the pre-execution estimate, which licenses the bootstrapping loop of \cref{sec:architecture}. Net of cost,
\begin{equation}
  \calU(i\mid\calH_t) = \calU_{\mathrm{IG}}(i\mid\calH_t) - \lambda\,\calC(i),
  \quad
  i_t^{*} = \argmax_{i\in\calI}\;\calU(i\mid\calH_t),
  \label{eq:net-utility}
\end{equation}
with $\lambda>0$ pricing information against resources; this is the $\beta\to\infty$ limit of Eq.~(\ref{eq:ci-policy}) applied to the cost-adjusted score of Eq.~(\ref{eq:net-score}), rule (b) of \cref{sec:master}, and agrees with the raw-score limit only when costs are equal.

A scalar price treats cost as one-dimensional and fungible, but in deployment latency, money, context occupancy, and safety or irreversibility are not interchangeable and cannot be collapsed into one currency without losing the constraints that matter. The framework accommodates this directly: $\lambda$ becomes a vector of prices, $\kappa_t$ a vector-valued threshold, and the gate a constrained rather than scalar comparison, where a hard deadline or a billing cap weights its component differently from a safety budget. The agentic ledger of Eq.~(\ref{eq:agent-cost}) carries exactly these terms; the scalar form is the special case where they share a currency.

\paragraph{Amortized variational estimation of utility.}
The exact posterior is unavailable, so a lightweight predictive subsystem (System-1; in the task-sampling literature, a risk prediction model \citep{wang2025modelpredictivetasksampling, wen2026bridging}, or a Bayesian filter over a prompt's solving dynamics \citep{mao2026dynamics}) maintains a variational approximation $q_\phi(\theta\mid\calH_t)$ fitted by minimizing the variational free energy
\begin{equation}
  \calF(\phi)
  = \KL{q_\phi(\theta\mid\calH_t)}{p(\theta)}
  - \E_{q_\phi}\!\bigl[\ln p(\calH_t\mid\theta)\bigr]
  = -\mathrm{ELBO}(\phi)\;\ge\;-\ln p(\calH_t).
  \label{eq:VFE}
\end{equation}
Minimizing $\calF$ improves the approximate posterior from which utilities are computed, $\hat u_\phi\approx\calU_{\mathrm{IG}}$; it does not by itself guarantee better utility estimates or better rankings, since a closer posterior approximation need not order candidates more accurately (compare the remark after \cref{prop:IB}). The approximation is substantial: a full Bayesian posterior over a modern network is intractable, so $\hat u_\phi$ is at best an amortized estimate of Eq.~(\ref{eq:utility-primitive}), and the gap may be large -- that gap \emph{is} the predictor bottleneck on which every gain of the co-design rests, and \cref{thm:regret} quantifies how it propagates to decisions. In practice the predictor learns the scalar functional directly (\cref{sec:generative}) rather than the posterior it summarizes. Unlike a literal free-energy agent, Eq.~(\ref{eq:VFE}) is not the selection objective (that is Eq.~(\ref{eq:net-utility})); it only keeps the predictor calibrated.

\paragraph{Epistemic action selection: the active-inference limit.}
In active inference an agent minimizes \emph{expected free energy} (EFE) \citep{friston2010}, which has a pragmatic and an epistemic term. Under uniform outcome preferences and with cost handled separately, the expected-free-energy criterion reduces to epistemic action selection based on expected information gain. Conditional interaction adds an explicit pre-execution cost gate and a learned predictor of that quantity.

\begin{proposition}[Active-Inference Limit]
\label{prop:active-inference}
Write the expected free energy of interaction $i$ as
\begin{equation}
  G(i)
  = \underbrace{-\,\E_{q_\phi(o\mid i)}\!\bigl[\ln p(o\mid C)\bigr]}_{\text{pragmatic value}}
  - \underbrace{\E_{q_\phi(o\mid i)}\!\bigl[\KL{q_\phi(\theta\mid o,i)}{q_\phi(\theta\mid i)}\bigr]}_{\text{epistemic value}},
  \label{eq:efe-decomp}
\end{equation}
with preferences $p(o\mid C)$ over a finite outcome space $\mathcal{O}$. If preferences are uniform and cost is accounted separately through $\calC(i)$, then
\begin{equation}
  \argmin_{i\in\calI} G(i)
  = \argmax_{i\in\calI}\; I_{q_\phi}(\theta;o\mid i)
  = \argmax_{i\in\calI}\; \calU_{\mathrm{IG}}^{q_\phi}(i\mid\calH_t),
  \label{eq:MI-optimality}
\end{equation}
where $\calU_{\mathrm{IG}}^{q_\phi}$ is Eq.~(\ref{eq:exp-info-gain}) evaluated under the variational model $q_\phi$, and coincides with $\calU_{\mathrm{IG}}$ when $q_\phi$ is the exact posterior and predictive. Under non-uniform preferences the two criteria diverge.
\end{proposition}

\noindent The identity is standard active-inference theory and we claim no novelty for it. Expected-free-energy decompositions differ across that literature; we adopt the pragmatic-plus-epistemic form of Eq.~(\ref{eq:efe-decomp}) and interpret the result as a correspondence, not a foundational derivation: it locates conditional interaction as the cost-gated, predictively amortized special case of epistemic action selection, applicable whenever information gain can be estimated more cheaply than the interaction can be executed.

\subsection{Understanding Through the Lens of Decision-Making}
\label{sec:generative}

Every interaction is coupled to the world through a generative transition kernel
\begin{equation}
  o \;\sim\; p_\psi\!\left(o \mid i,\, s_t\right),
  \qquad
  s_{t+1} \;=\; T\!\left(s_t,\, i,\, o\right),
  \label{eq:generative-transition}
\end{equation}
where $s_t$ is the state the interaction acts on, $o$ its outcome, $p_\psi$ the observation/transition model (environment dynamics in the physical regime; the policy's own generation in the digital regime), and $T$ the update folding the outcome into the next state. The uncertain parameter $\theta$ and the state $s_t$ may coincide (the policy is both uncertain and evolving) or differ (a robot's configuration versus its skill parameters).

In the digital regime it is essential to distinguish the \emph{interaction}, the costly generation of a sequence $y=(y^1,\dots,y^L)$ from context $x$,
\begin{equation}
  \pi_\theta(y\mid x)
  \;=\;
  \prod_{\ell=1}^{L}
  \pi_\theta\!\left(y^\ell \,\middle|\, x,\, y^{<\ell}\right),
  \label{eq:autoregressive-generation}
\end{equation}
which costs $L$ sequential forward passes (with key--value caching, roughly linear in $L$, with a per-token cost set by the model's depth, width and attention and feed-forward operations; we write $\Theta(L\,d_{\mathrm{model}}^2)$ as shorthand, absorbing depth and context-length terms; this is the width-dominated regime $L\lesssim d_{\mathrm{model}}$, and with full attention each token also attends over the accumulated context, adding $\Theta(L^2d_{\mathrm{model}})$ per sequence, which key--value caching does not remove), from the \emph{outcome}, the verdict of a verifier $\calV$ applied to $y$ (does the CoT solve the problem, do the unit tests pass):
\begin{equation}
  o \;=\; \calV(y), \qquad
  p_\psi(o\mid i, s_t)
  \;=\;
  \sum_{y\,:\,\calV(y)=o}\pi_\theta(y\mid x).
  \label{eq:digital-outcome}
\end{equation}
The kernel is the generation distribution pushed through the verifier: the verdict $o$ (often a scalar reward, correct/incorrect, or a score) is what updates the agent's beliefs, and the tokens $y$ are merely the costly medium through which it is obtained. The cost is paid on all $L$ tokens, whereas the information resides in $o$, which may be a single bit.

A prompt on which every rollout returns the same verdict yields zero GRPO advantage and hence no group-relative policy-gradient signal despite its full cost. This is a statement about the update, not about belief change: an all-fail group can still inform a posterior over the prompt's difficulty, and the gradient depends on the sequences through their log-likelihoods, not only on the verdict. The verdict channel models where the cost and the information sit, while the RLVR gate scores the declared proxy $u_{\mathrm{adv}}$ of \cref{rem:identifiability} rather than information gain. The predictor's task is to forecast the verdict's declared utility \emph{before} generating the sequence, which is why $c_{\mathrm{pred}}\ll c_{\mathrm{exec}}$; Fig.~\ref{fig:generative} shows the pipeline and the point at which System-1 cuts in.

The physical regime is identical with the verifier replaced by the environment's reward and the sequence by a trajectory with probability $\prod_k\pi_\theta(a_k\mid s_k)\,p(s_{k+1}\mid s_k,a_k)$, whose factors run over physical time steps that cannot be parallelized.

\begin{figure}[t]
\centering
\includegraphics[width=\linewidth]{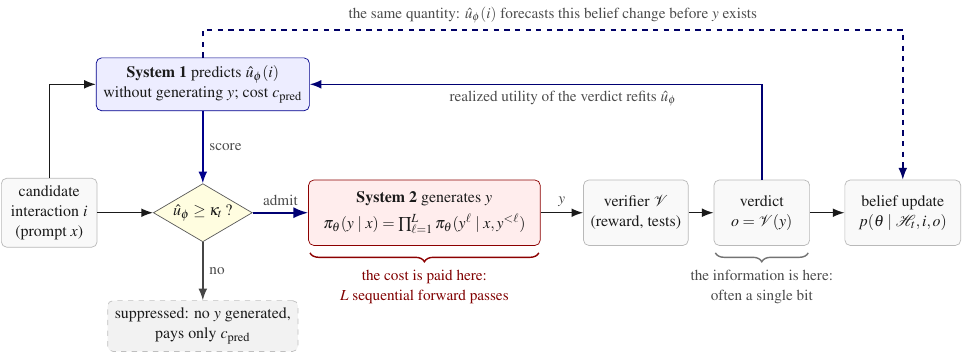}
\caption{\textbf{The generative model, and where System~1 cuts in.} The cost is paid in full on generating the sequence $y$, while the information lies in the verdict $o=\calV(y)$, often a single bit. System~1 forecasts the belief change that verdict will cause before $y$ exists, and the gate generates only if the forecast clears $\kappa_t$; suppressed candidates cost only $c_{\text{pred}}$. Realized verdicts refit System~1. With information gain as the utility the forecast is the expected Bayesian surprise; for any other declared utility the pipeline forecasts that utility instead.}
\label{fig:generative}
\end{figure}

\paragraph{One kernel, two interpretations: dynamics and decision.}
Viewed forward in time, Eq.~(\ref{eq:generative-transition}) is a dynamical-system transition. Viewed through the agent's posterior, it defines a decision under uncertainty, since the utility is a functional of $p_\psi$:
\begin{equation}
  u(i;\calH_t)
  \;=\;
  \E_{o\sim p_\psi(o\mid i,s_t)}\!\Bigl[
    \KL{p(\theta\mid\calH_t,i,o)}{p(\theta\mid\calH_t)}
  \Bigr].
  \label{eq:utility-as-generative-functional}
\end{equation}
The gate thus chooses among transitions by the posterior change each is predicted to cause. System-1 is a cheap surrogate estimating Eq.~(\ref{eq:utility-as-generative-functional}) without rolling $p_\psi$ forward; System-2 is the true transition, paid for only when the surrogate certifies it. Unlike a world model, which learns $p_\psi$ to simulate outcomes (re-incurring much of the cost the gate avoids), the predictor learns only this scalar functional. The bottleneck is therefore one of \emph{valuation}, not \emph{simulation}.

\paragraph{Valuation without simulation: the predictor as a decision-focused surrogate.}
This interpretation clarifies what $\hat u_\phi$ is: an amortized, decision-focused projection of the generative model, enough of $p_\psi$ to rank transitions by their epistemic value, and no more. That is why the predictor can be small relative to the executive (\cref{prop:amortized-screening}), and why its failures (\cref{sec:exploration,sec:selective-supervision}) are failures to \emph{value} transitions rather than to \emph{simulate} them. It also motivates the control-theoretic interpretation that follows: once interactions are transitions of a system with a generative model, choosing them to optimize a long-horizon objective is a control problem, and model predictive control is its standard closed-loop solution.

\subsection{Understanding Through the Lens of Control Theory}
\label{sec:mpc}

\paragraph{Receding-horizon control of a learning plant.}
Conditional interaction is \emph{structurally analogous} to model predictive control (MPC) of a non-stationary plant whose state is the policy parameter. The analogy is not an equivalence (MPC selects continuous actions; conditional interaction selects subsets of interactions), but it holds component by component.

\emph{(1) Model.} The plant state is $\theta_t$; the predictor $\Phi$ forecasts how an interaction batch $\mathcal{B}_t$ moves it,
\begin{equation}
  \hat\theta_{t+h} = \Phi_h(\theta_t,\mathcal{B}_t,\ldots),
  \qquad
  \theta_{t+1} = \theta_t - \eta\,\nabla_\theta \calL(\theta_t,\mathcal{B}_t) + \xi_t,
  \label{eq:policy-dynamics}
\end{equation}
where the right-hand equation gives the true dynamics with gradient noise $\xi_t$ of covariance $\propto\abs{\mathcal{B}_t}^{-1}$.

\emph{(2) Prediction.} System-1 forecasts the evaluation risk $R(\tau_{\text{eval}};\hat\theta_{t+h})$, $h=1,\dots,H$, without executing any interaction (Eq.~(\ref{eq:amortized-assessment})).

\emph{(3) Receding-horizon planning.} The control law minimizes discounted predicted risk over admissible batches,
\begin{equation}
  \mathcal{B}_t^* = \argmin_{\mathcal{B}\subseteq\calD,\,\abs{\mathcal{B}}=B}\;
    \sum_{h=1}^H \gamma^h\,\E\!\left[R(\tau_{\text{eval}};\hat\theta_{t+h})\right];
  \label{eq:mpc-control}
\end{equation}
$H=1$ is a one-step, risk-based analogue of the greedy gate of Eq.~(\ref{eq:ci-policy}). It recovers that gate exactly when a candidate's gross utility is its expected reduction of evaluation risk, costs are equal or priced into the score as in Eq.~(\ref{eq:net-score}), and the effect of a batch decomposes into candidate-level contributions, while $H>1$ accounts for how today's selection shapes tomorrow's reachable states.

\emph{(4) Execution.} Only the first batch is executed and the rest re-planned, which makes the loop robust to model error in distant horizons.

\emph{(5) Update.} Executed outcomes $y_t$ update the model, $p(\Phi\mid\calH_{0:t+1}) \propto p(y_t\mid\calH_{0:t},\Phi)\,p(\Phi\mid\calH_{0:t})$ (Eq.~(\ref{eq:rpm-fitting})).

Unlike textbook MPC, plant and controller co-evolve: the state $\theta_t$ that the controller steers is the same object whose dynamics the model $\Phi$ is still learning. Stability is therefore not classical; it rests on System-1 adapting faster than System-2 (\cref{sec:architecture}), and the closed-loop procedure below is this MPC cycle instantiated.

\paragraph{Index policies and optimal stopping.}
Define the value of executing $i$ as
\begin{equation}
  V(i,\calH_{0:t}) = \E\!\bigl[\Delta\calP(i)\mid\calH_{0:t}\bigr] - \calC(i),
  \label{eq:VoI}
\end{equation}
and execute $i$ iff $V\ge V_\emptyset$, the value of deferring. Here and in Eq.~(\ref{eq:adm-rule}), $\calC(i)$ is the \emph{priced} cost, in the units of the utility: $\lambda\,\calC_{\mathrm{raw}}(i)$ for a scalar ledger, or the composite price of Eq.~(\ref{eq:agent-cost}); $V_\emptyset$ is on the same scale. \emph{Under a discounted, independent-arm bandit approximation with a retirement option}, the optimal policy is an index rule: execute while the Gittins index $\mathfrak{G}(i,\calH_{0:t})$ \citep{gittins1979} exceeds the retirement value $\kappa_t$,
\begin{equation}
  \pi_{\mathrm{CI}}^*(i) = \indicator{\mathfrak{G}(i,\calH_{0:t})\ge\kappa_t}.
  \label{eq:gittins}
\end{equation}
The Gittins guarantee holds only under those bandit assumptions. Conditional interaction is in general non-stationary (the state $\theta_t$ drifts) and its interactions are dependent (they share one evolving policy), so Eq.~(\ref{eq:gittins}) motivates the threshold form of Eq.~(\ref{eq:ci-policy}) rather than certifying it. With that qualification, the principle inherits the information-theoretic interpretation of active inference (\cref{prop:active-inference}) and the resource-allocation interpretation of index policies as two limiting analyses, each valid under its own assumptions.

\begin{figure}[t]
\centering
\includegraphics[width=\linewidth]{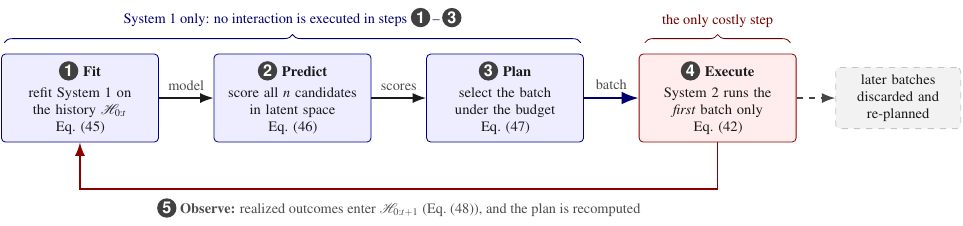}
\caption{\textbf{The closed loop as receding-horizon control.} System~1 is refitted \stepnum{1}, scores the pool \stepnum{2}, and solves a budgeted selection \stepnum{3}; only then does System~2 execute, and only the first batch \stepnum{4}, whose outcomes re-enter the history \stepnum{5}. Steps \stepnum{1}\,--\,\stepnum{3} execute nothing in the environment, but refitting, scoring and solving the budgeted selection still cost compute; the loop's cost is dominated by execution \stepnum{4} only when those steps are cheap enough, which is a design requirement (\cref{prop:amortized-screening}) rather than a consequence of their executing nothing.}
\label{fig:mpcloop}
\end{figure}

\paragraph{Predict, select, execute, refit: the closed loop.}
Fig.~\ref{fig:mpcloop} shows the resulting pipeline. The MPC components are realized by four operations. \emph{Model update:} the predictor's parameters are fitted online by maximizing the conditional log-likelihood of the observed history (the streaming counterpart of Eq.~(\ref{eq:VFE})), in practice by a few stochastic gradient steps per round:
\begin{equation}
  \max_{\phi}\;
  \calL_{\mathrm{ML}}(\phi)
  \;:=\;
  \sum_{t}\ln p_\phi\!\bigl(\calH_t\mid\calH_{0:t-1}\bigr).
  \label{eq:rpm-fitting}
\end{equation}
\emph{Prediction:} for each candidate $\tau_i$, $i=1,\dots,n$, the predictor returns the predicted mean and uncertainty of the learning signal $\ell$,
\begin{equation}
  p_\phi\!\bigl(\ell\mid\tau_i,\calH_{0:t};\theta_t\bigr)
  \;\longrightarrow\;
  \bigl\{m(\ell_i),\,\sigma(\ell_i)\bigr\}_{i=1}^n,
  \label{eq:amortized-assessment}
\end{equation}
at a cost independent of the cost of executing $\tau_i$, since the predictor operates on embeddings.

\emph{Planning and execution:} the estimates feed a budgeted subset selection, the fixed-size restriction of Eq.~(\ref{eq:master}),
\begin{equation}
  \max_{\mathbf{z}_t\in\{0,1\}^n}\;
  \sum_{i=1}^n z_t^{(i)}\,\calA(\ell_t^{(i)})
  \;+\;
  \mu\,\calS\!\Bigl(\bigl\{\tau_t^{(i)}: z_t^{(i)}=1\bigr\}\Bigr),
  \qquad
  \text{s.t.}\;\; \norm{\mathbf{z}_t}_1 = B,
  \label{eq:subset-selection}
\end{equation}
where $\calA$ aggregates $m(\ell_i)$ and $\sigma(\ell_i)$ and $\calS$ rewards diversity among co-selected interactions. If $\calA\ge0$, $\mu\ge0$, and $\calS$ is monotone submodular, the objective is monotone submodular and greedy selection attains a $(1-1/e)$ approximation. The condition matters: a diversity term that sums pairwise distances, as in Eq.~(\ref{eq:phys-acquisition}), is supermodular, since a new item adds more pairs to a larger set, so the guarantee does not cover it. Monotone submodular choices include facility-location coverage and $\log\det(I+K_S)$ for a positive semidefinite similarity kernel $K$.

\emph{Feedback:} executed outcomes are appended to the history,
\begin{equation}
  \calH_{0:t+1}
  \;\leftarrow\;
  \calH_{0:t}\,\cup\,
  \Bigl\{\bigl(\tau_t^{(i)},\,\ell_t^{(i)}\bigr)\,:\,z_t^{(i)}=1\Bigr\},
  \label{eq:history-update}
\end{equation}
and fed back into Eq.~(\ref{eq:rpm-fitting}) on the next round. The update is asymmetric: only executed interactions enter $\calH$, so the predictor is trained on a biased sample of its own choosing, and the validity of the procedure rests on the timescale separation discussed in \cref{sec:architecture}; the statistical consequences of this censoring are discussed in \cref{sec:selective-supervision}. Together, the four operations realize a complete closed-loop instance of the principle in which the predictive subsystem and the executive policy co-adapt through their shared history. Model predictive task sampling and its in-context extension run this loop over tasks \citep{wang2025modelpredictivetasksampling, wen2026bridging}, and recent agentic RL methods carry it inside long-horizon trajectories, allocating rollouts to shared prefixes and propagating terminal rewards back through the resulting trees \citep{zou2026trace, zou2026tree}.

\section{Conditional Interaction Across Digital and Physical Domains}
\label{sec:domains}
\label{sec:unified-domains}

Reasoning-model fine-tuning and embodied-agent training are instances of one abstraction, distinguished by state space, cost structure, and verifier. In both, conditional interaction learns a predictor of value from past outcomes and executes the interactions whose predicted value clears the threshold $\kappa_t$ of \cref{def:ci-policy}. The two instantiations share one predicament --- \emph{finding out whether a task is worth training on often costs almost as much as training on it} --- while differing in why. \Cref{tab:tt-ledger} summarizes the correspondence together with the test-time and agentic settings.

\subsection{Digital and Physical Instantiations in the Literature}
\label{sec:digital}
\label{sec:physical}

\paragraph{Digital: from filtering after generation to predicting before it.}
In RLVR only prompts whose sampled responses disagree carry a group-relative advantage, and that band drifts as the policy improves \citep{qu2026listwise}. Early methods filter after generating: dynamic sampling in DAPO discards prompts whose groups are uniformly correct or wrong \citep{yu2026dapo}, and online difficulty filtering keeps prompts of intermediate pass rate \citep{bae2026online}, both paying in full for the rollouts that reveal a prompt to be uninformative --- the execute-then-filter baseline of \cref{prop:amortized-screening}. Later methods predict before generating: GRESO skips prompts that their reward history marks as likely zero-variance \citep{zheng2025act}; MoPPS tracks each prompt's success rate with an online Bayesian posterior and approaches an evaluation-intensive dynamic-sampling baseline with far fewer rollouts \citep{qu2026moppps}; DPS, BOTS and PCL model prompt difficulty with a hidden-Markov process, Thompson sampling and an on-policy value model, respectively \citep{mao2026dynamics,shen2026bots,gao2026prompt}; and GPS trains one small generative predictor on the shared optimization history to select prompts for a much larger reasoning model, improving training efficiency and test-time compute allocation \citep{qu2026gps}. The same move extends beyond prompts, to online batch selection in supervised fine-tuning \citep{zou2026uds}, to demonstration selection for in-context learning, where a small trained model makes a single edit to the retrieved demonstrations instead of iteratively scoring or searching over demonstration subsets \citep{wen2026you}, to rollout allocation across tasks and trajectory prefixes in agentic RL \citep{zou2026trace}, and to critic-free, tree-aware credit assignment for long-horizon agents \citep{zou2026tree}. These results \emph{illustrate} rather than \emph{validate} the framework.

\paragraph{Physical: from randomizing to predicting adaptation risk.}
For embodied agents, judging a task configuration means acting on it, on hardware or in a high-fidelity simulator, at the pace of physics. Domain randomization trains across wide ranges of simulated appearance and dynamics \citep{tobin2017domain}, with an explicit assessment of transfer \citep{muratore2018domain}; because uniform randomization spends most trajectories on mastered configurations, later work prioritizes, through active and automatic domain randomization \citep{mehta2020active,akkaya2019solving} and through environment design and level replay driven by regret or learning potential \citep{dennis2020emergent,jiang2021prioritized}, though still by executing configurations to score them. When robustness matters, the target is the adverse tail of post-adaptation risk, measured by the conditional value-at-risk \citep{rockafellar2000} and pursued in robust meta-learning \citep{finn2017model,wang2023simple,lv2024theoretical}, risk-averse planning \citep{carpin2016risk} and safety-critical scenario generation \citep{rempe2022generating}. Locating that tail requires evaluating many tasks; MPTS removes the evaluation by learning a risk predictive model from the optimization history \citep{wang2025modelpredictivetasksampling}, and its successors add posterior sampling and diversity \citep{qu2025fast} and in-context risk prediction \citep{wen2026bridging}, yielding curriculum-like sampling under a calibrated predictor.

\paragraph{What the two literatures share.}
Evaluation requires generating in the digital regime and acting in the physical one, yet both literatures converge on the same loop: System-1 scores a candidate from cheap features and history, System-2 executes only what clears the score, and System-2's outcomes refit System-1. They differ in utility (mixed outcomes versus tail adaptation risk; \cref{rem:utility-scope}) and in the conserved resource (compute versus actuation); their acquisition functions are collected in \cref{app:acquisition}. Unlike post-hoc data valuation by influence functions or data Shapley \citep{koh2017understanding,ghorbani2019data}, and unlike active learning, experimental design and bandit index policies over a static pool \citep{krishnakumar2007active,chaloner1995bayesian,gittins1979}, these methods value candidates \emph{before} execution and track a \emph{moving} learner. They also extend earlier fast--slow pairings, such as expert iteration and AlphaZero \citep{anthony2017thinking,silver2018general} and speculative decoding \citep{leviathan2023fast}, from one game or one decoding step to training and deployment.

\section{Conditional Interaction at Test Time}
\label{sec:test-time}

The framework so far concerns training. The same decision problem appears at inference, but with a different state, utility and cost ledger. Nothing in \cref{def:ci-policy} and Eq.~(\ref{eq:IIE}) requires the training setting: the rule gates an interaction on predicted utility net of cost, whether the resource is a training rollout or an inference-time forward pass. We argue that \emph{test-time scaling} and \emph{adaptive decision-making}, usually studied separately, can be cast as test-time instances of conditional interaction. Prospective metacognitive control of test-time compute is an active line of work in its own right \citep{nazi2026triage, cao2026metacognitive}; the contribution here is not to discover it but to place it in the same ledger as training. Co-design therefore does not end at deployment: the System-1 that shaped training goes on to ration System-2's inference effort.

At deployment $\theta$ is fixed; what varies, and costs, is the computation and the number of real-world actions spent per output. An interaction is now an \emph{increment of inference}: a reasoning step, an extra sample, a tool or retrieval call, or a physical action. The executive state is $b_t = p(a\mid q,\calH_{0:t})$, a predictive distribution over candidate answers or latent solution states $a$ for query $q$, which we call the \emph{belief state} for convenience without requiring a literal Bayesian posterior. The utility of an increment is
\begin{equation}
  \mathbb{I}_{\mathrm{tt}}(i_t)
  \;=\;
  \E\!\bigl[\,\KL{b_{t+1}}{b_t}\;\big|\;\calH_{0:t},\,i_t\,\bigr],
  \label{eq:tt-utility}
\end{equation}
and its cost is the inference compute or actuation it requires, so \cref{def:ci-policy} applies verbatim with $\hat u_\phi$ predicting Eq.~(\ref{eq:tt-utility}) and $\kappa_t$ as a stopping threshold. Eq.~(\ref{eq:tt-utility}) measures epistemic movement, not improvement: a changed answer is not necessarily a better one, so movement of $b_t$ should be interpreted as computational sensitivity, and accuracy must be checked separately. The test-time hypothesis is that this sensitivity predicts where additional computation is useful. Only the ledger changes --- parameter belief change becomes answer belief change, and training compute becomes inference compute --- and that substitution is the entire test-time claim.

\subsection{Face I: Efficient Test-Time Scaling}
\label{sec:tts}

Test-time scaling improves a fixed model by spending more inference per query (longer CoT, repeated sampling with aggregation, tree search, tool calls). This spend is highly heterogeneous across queries: easy ones are solved by the first short trajectory and the hardest by no feasible budget, so extra reasoning moves $b_t$ only in an intermediate band, the analogue of the partially solved prompts of \cref{sec:digital}. Uniform test-time scaling therefore inherits the vanishing marginal efficiency of Eq.~(\ref{eq:marginal-efficiency}), now in inference cost. Conditional interaction continues to spend only while the predicted change of increment $c_t$ clears the threshold:
\begin{equation}
  \text{halt at step } t^\star
  \;=\;
  \min\bigl\{\,t : \hat u_\phi(c_{t}\mid\calH_{0:t}) < \kappa_t \,\bigr\},
  \label{eq:adaptive-halt}
\end{equation}
the optimal-stopping form of Eqs.~(\ref{eq:VoI}) and~(\ref{eq:gittins}) applied at inference. Familiar techniques are choices of increment and predictor: adaptive-length reasoning and early exit gate one more decoding step, adaptive self-consistency one more sample, budget-aware tree search one more node expansion. All suppress increments predicted to leave $b_t$ unchanged, so one proposition covers them, with one caveat. Eq.~(\ref{eq:tt-utility}) values the \emph{next} increment, and an increment can leave $b_t$ unchanged while enabling later ones that do change it, as when an intermediate thought pays off only several steps on. A gate that halts on the next-step utility can therefore stop before such a delayed gain. The guarantee below is stated for the \emph{continuation utility}
\begin{equation}
  u_{\mathrm{cont}}(q)\;=\;\E\!\bigl[\,\KL{b_C}{b_1}\;\big|\;q,\calH_{0:1}\,\bigr],
  \label{eq:tt-continuation}
\end{equation}
the expected belief change from the first increment to the full budget $C$, which System-1 must predict after the probing increment and whose prediction cost is charged as $c_{\text{pred}}$. Predicting the value of the next token, of a whole continuation, and of the final answer are different tasks, and the second and third are generally harder. Eq.~(\ref{eq:adaptive-halt}) with the next-step utility is covered only under a \emph{no-delayed-gain} condition: a query whose next increment leaves $b_t$ unchanged in expectation gains nothing from any later increment. Two stopping policies must also be kept apart. The \emph{one-probe gate} spends one increment, predicts Eq.~(\ref{eq:tt-continuation}), and then stops or runs to the full budget; it is the policy \cref{prop:tts-efficiency} covers. A \emph{sequential halting} policy, Eq.~(\ref{eq:adaptive-halt}), re-decides at every step and should score the continuation value from the current state, $u_{\mathrm{cont},t}(q)=\E[\KL{b_C}{b_t}\mid q,\calH_{0:t}]$; a halt at step $t$ then loses accuracy only on a false negative, so by a union bound its accuracy loss is at most $\sum_t\delta_{\mathrm{FN},t}$, the sum of the per-step false-negative masses, and its cost must charge a prediction at every step. We do not develop that guarantee further. Its error term, and that of \cref{prop:agentic-efficiency}, rests on the following margin condition, which turns a score error into a misclassification rate.

\begin{lemma}[{Misclassification Under a Margin Condition}]
\label{lem:margin}
Let candidates $i\sim\calD$ be utility-relevant when $u(i)\ge u_{\min}$, let the gate admit $i$ when $\hat u(i)\ge u_{\min}$, and suppose $\abs{\hat u(i)-u(i)}\le\varepsilon$ for every $i$. More generally, let the gate admit $i$ when $\hat u(i)\ge\kappa$. Every misclassified candidate, admitted while dead or suppressed while relevant, then lies in the band $\mathcal{B}_\varepsilon=\{i:\kappa-\varepsilon\le u(i)<u_{\min}\}\cup\{i:u_{\min}\le u(i)<\kappa+\varepsilon\}$, so the misclassified masses satisfy $\delta_{\mathrm{FP}}+\delta_{\mathrm{FN}}\le\Pr_{i\sim\calD}(\mathcal{B}_\varepsilon)$. With $\kappa=u_{\min}$, if the true utility satisfies the margin condition $\Pr_{i\sim\calD}\bigl(\abs{u(i)-u_{\min}}\le\varepsilon\bigr)\le m\,\varepsilon$ with constant $m$, then the gate misclassifies, by admitting a dead candidate or suppressing a relevant one, at most a fraction $m\,\varepsilon$ of candidates.
\end{lemma}

\noindent A density of $u$ bounded by $m/2$ on $[u_{\min}-\varepsilon,u_{\min}+\varepsilon]$ gives the condition. The band, not the margin condition, is the general object, and atoms are the case that matters in practice. At test time many queries have continuation utility exactly zero, and relevance means $u>0$ ($u_{\min}\to0^+$); the whole atom then lies within $\varepsilon$ of the floor, and the margin condition with $\kappa=u_{\min}$ fails. Placing the gate threshold above the error, $\kappa>\varepsilon$, excludes the atom: a query with $u=0$ has $\hat u\le\varepsilon<\kappa$ and is never admitted, so $\delta_{\mathrm{FP}}=0$ and $\delta_{\mathrm{FN}}\le\Pr(0<u<\kappa+\varepsilon)$, the mass of small but nonzero utilities. Whether that mass is small is a property of the task distribution, to be measured on an audit pool, not an assumption to take for granted.

\begin{proposition}[Conditional Test-Time Cost Bound Under a Halting Margin]
\label{prop:tts-efficiency}
Let $\calA^{\mathrm{tt}}$ be the queries on which the full budget is more accurate in expectation than one increment, with density $\rho_{\mathrm{acc}}^{\mathrm{tt}}$, and let the gate threshold a declared test-time utility whose relevant set contains $\calA^{\mathrm{tt}}$, with density $\rho_u^{\mathrm{tt}}\ge\rho_{\mathrm{acc}}^{\mathrm{tt}}$. Let unconditional scaling spend $C$ unit increments per query, let the answer be a function of the belief state, and let the selection policy be the one-probe gate: it spends one probing increment plus $c_{\text{pred}}$ and then halts or continues to the full budget, with $\delta_{\mathrm{FN}}$ and $\delta_{\mathrm{FP}}$ the fractions of all queries that are relevant but halted and dead but continued. Then the gated policy loses at most $\delta_{\mathrm{FN}}$ accuracy while spending at most a fraction
\begin{equation*}
  \rho_u^{\mathrm{tt}}+\frac{1+c_{\text{pred}}}{C}+\delta_{\mathrm{FP}}
\end{equation*}
of the unconditional compute, a multiplier approaching $1/\rho_u^{\mathrm{tt}}$ as $C\to\infty$ and $\delta_{\mathrm{FP}}\to0$. By \cref{lem:margin}, $\delta_{\mathrm{FN}}+\delta_{\mathrm{FP}}\le\Pr(\mathcal{B}_{\varepsilon_\phi})$ for a predictor with error at most $\varepsilon_\phi$, which is at most $m\varepsilon_\phi$ under the margin condition. The tightest choice is the accuracy-gain utility itself, for which $\rho_u^{\mathrm{tt}}=\rho_{\mathrm{acc}}^{\mathrm{tt}}$. The continuation utility of Eq.~(\ref{eq:tt-continuation}) with $u_{\min}\to0^+$ is also admissible: if its value is zero then $b_C=b_1$ almost surely, so the answer, and with it the accuracy, is unchanged; its gate should use $\kappa>\varepsilon_\phi$, so that the atom of zero-utility queries is never admitted. With density $\rho_{\mathrm{cont}}^{\mathrm{tt}}\ge\rho_{\mathrm{acc}}^{\mathrm{tt}}$ it gives the weaker multiplier $1/\rho_{\mathrm{cont}}^{\mathrm{tt}}$. The next-step utility of Eq.~(\ref{eq:tt-utility}) is admissible only under the no-delayed-gain condition.
\end{proposition}

\noindent The proposition is the inference-time counterpart of Eq.~(\ref{eq:IIE-ratio}), with the per-query halting decision in place of per-batch selection, and only $\rho_u^{\mathrm{tt}}$ has to be re-estimated. Its status needs care: the multiplier follows \emph{given} a small informative fraction and an accurate halting predictor, and unlike the training side neither premise has been measured, making this the one place where our efficiency claims lack even indirect support. The falsification protocol is direct: measure $\rho_{\mathrm{acc}}^{\mathrm{tt}}$ as the fraction of queries whose answer becomes correct when the budget is raised from a short default to the full one, run unconditional scaling at full budget for reference accuracy, then run the one-probe gate and record accuracy and mean compute. A larger accuracy loss than the false-negative mass $\delta_{\mathrm{FN}}$ measured on an audit sample, or a failure to realize the saving, refutes the test-time claim.

\subsection{Face II: Efficient Adaptive Decision-Making}
\label{sec:adm}

For systems acting in the world, the increments are external actions (a manipulation, a sensor query, a tool call, a clarifying question), each with a real cost and the capacity to change beliefs. The efficient policy acts when predicted value of information net of cost is positive, i.e.\ Eq.~(\ref{eq:VoI}) promoted to an online rule:
\begin{equation}
  \text{execute } i_t
  \quad\Longleftrightarrow\quad
  \E\!\bigl[\Delta\calP(i_t)\mid\calH_{0:t}\bigr]
  \;-\;
  \calC(i_t)
  \;\ge\;
  V_\emptyset,
  \label{eq:adm-rule}
\end{equation}
with $\Delta\calP(i_t)$ the expected improvement in task return, $\calC(i_t)$ the priced cost defined after Eq.~(\ref{eq:VoI}), and $V_\emptyset$ the value of deferring or taking the cheap default, adapted online as in Eq.~(\ref{eq:threshold-update}). Reducing task uncertainty is a different utility, not a second reading of $\Delta\calP$: it instantiates the same rule with $\Delta\calP$ replaced by the expected entropy reduction of the task posterior. Selective tool use, active perception, and interrupt-or-defer control are Eq.~(\ref{eq:adm-rule}) with different $i_t$ and utilities: a tool call scored by answer-entropy reduction, a sensor movement or question scored by world-model belief change, a re-plan triggered when predicted plan value degrades. The adaptation-risk predictor that prioritized training tasks in \cref{sec:physical} becomes, at deployment, the value-of-information predictor that gates actions: efficient adaptive decision-making can be cast as conditional interaction with the world as the executive.

\subsection{Selective Agentic Tool and Computer Use}
\label{sec:agentic}

An agent that plans, retrieves, calls tools, asks questions and operates a computer faces Eq.~(\ref{eq:adm-rule}) against a composite cost. At step $t$ it holds an answer/plan belief state $b_t$ and chooses from a heterogeneous menu (think, retrieve or call a tool, ask, delegate, commit), each action costing
\begin{equation}
  \calC(i_t)
  \;=\;
  \lambda^\top c(i_t)
  \;+\;
  \lambda_{\mathrm{risk}}\,\Pr[\text{irrecoverable}\mid i_t],
  \label{eq:agent-cost}
\end{equation}
where $c(i_t)=\bigl(c_{\text{lat}}(i_t),\,c_{\$}(i_t),\,c_{\text{tok}}(i_t)\bigr)$ collects latency, metered spend and context tokens, $\lambda\ge0$ prices them, and $\lambda_{\mathrm{risk}}$ prices an irrecoverable side effect. In the co-designed pair, System-1 (typically a verifier head or small router) admits $i_t$ only when Eq.~(\ref{eq:adm-rule}) holds for this cost, with $V_\emptyset$ the value of committing the current best answer, and System-2 is the agent that executes it. Tool use and computer use differ only in the size of the space to search: a tool call is one typed request among a few, whereas a computer-use episode branches at each of tens of steps, every one costing a screenshot, a model call and unparallelizable time.

\begin{proposition}[{Conditional Agentic Cost Bound for Separable Actions}]
\label{prop:agentic-efficiency}
Suppose (a) omitting any set of decision-irrelevant actions leaves the answer unchanged; (b) the gate thresholds a predicted relevance score with error at most $\varepsilon_\phi$, and on every task the action scores satisfy the margin condition of \cref{lem:margin} with constant $m$, so that in expectation at most a fraction $m\,\varepsilon_\phi$ of the actions are misclassified; and (c) prediction costs a fraction $r$ of an average action. Consider tasks with $N$ candidate actions in which no action costs more than $R_c$ times the task's average action cost. If utility-relevant actions carry a share $\rho_u^{\calC}$ of the expected call-everything cost (the cost-weighted density of utility-relevant actions, which reduces to their fraction $\rho_u$ when all actions cost the same), the conditional agent's accuracy is within $N m\,\varepsilon_\phi$ of call-everything accuracy, at expected cost
\begin{equation*}
  \E[\calC_{\mathrm{CI}}]
  \;\le\;
  \bigl(\rho_u^{\calC}+r+m\,\varepsilon_\phi\,R_c\bigr)\,
  \E[\calC_{\text{all}}].
\end{equation*}
\end{proposition}

\noindent The proposition is therefore a separable-action special case. Assumption (a) fails on long interface trajectories, where a dozen individually uninformative clicks may be jointly needed to reach the screen that matters, and whenever a locally uninformative action, such as navigating, authenticating or opening a file, is a prerequisite for a useful one; there utility belongs to the sub-trajectory or state transition, valued by a state-dependent value of information, and the unit to score is a sub-trajectory, and since the accuracy loss grows linearly in $N$, the per-action error must shrink with the horizon. Pure exploration (\cref{sec:exploration}) and mandated actions such as safety logging lie outside the gate. Across \cref{sec:tts,sec:adm} and this subsection the predict--gate--execute pattern is shared while utilities and costs are not (\cref{tab:tt-ledger}), so a System-1 calibrated against System-2 during training can serve as the inference-time router, with deployment outcomes flowing back as its training data.

\begin{table}[t]
\centering
\small
\renewcommand{\arraystretch}{1.25}
\setlength{\tabcolsep}{4pt}
\resizebox{\textwidth}{!}{%
\begin{tabular}{@{}lllll@{}}
\toprule
\textbf{Schema element} & \textbf{Digital (training)} & \textbf{Physical (training)} & \textbf{Test time} & \textbf{Agentic} \\
\midrule
Interaction $i$ (and execution) & prompt $\tau\to$ CoT generation & task config.\ $\xi\to$ env.\ rollout & reasoning step / sample & tool / think / ask action \\
Executive state         & policy parameters $\theta$ & policy parameters $\theta$ & answer belief state $b_t$ & answer / plan belief state $b_t$ \\
Verifier / feedback     & reward checker / scorer & adaptation loss & consistency / halting signal & tool result / task outcome \\
Utility $\mathbb{I}(i)$ & mixed-outcome proxy $u_{\mathrm{adv}}$ & predicted adaptation risk & accuracy gain (or belief change) & value of information \\
System-1 output         & $\E_{z\sim\pi_t}[w(z)\mid\tau]$ & $\hat{\calR}_{\text{adapt}}(\tau;\theta_t)$ & predicted value of continuing & predicted VoI net of cost \\
Cost $\calC(i)$         & training compute ($L\,d_{\text{model}}^2$) & physics fidelity / actuation & inference compute & latency + \$ + tokens + risk \\
Gate $\kappa_t$         & informative-sample rate & risk threshold & halting threshold & call/commit threshold \\
Efficiency result      & Eq.~(\ref{eq:IIE-ratio}) & \cref{prop:amortized-screening} & \cref{prop:tts-efficiency} & \cref{prop:agentic-efficiency} \\
\bottomrule
\end{tabular}%
}
\caption{One decision pattern under four cost ledgers. Execution is part of each interaction (first row), not a separate stage. The selection pattern is shared across columns; the utilities and learning dynamics are not, and the conserved resource differs: gradients, actuation, inference compute, or the composite cost of Eq.~(\ref{eq:agent-cost}).}
\label{tab:tt-ledger}
\end{table}

\section{A Testable Scaling Hypothesis, and Open Problems}
\label{sec:implications}

Co-design makes two kinds of prediction: what the pair should achieve as it scales, and where it can fail. We state the first as a testable hypothesis about informative compute, then turn to the main obstacle --- System-1 is supervised only by the interactions it lets System-2 execute --- and to open problems. It helps to separate three kinds of claim. \emph{Proved under stated assumptions}: the IIE accounting (\cref{prop:iie-gain}), the screening and memory results (\cref{prop:amortized-screening,prop:memory-scope}), the regret bounds (\cref{thm:regret,cor:regret-hp}), the test-time and agentic cost bounds (\cref{prop:tts-efficiency,prop:agentic-efficiency}), co-scaling under fixed error curves (\cref{prop:co-scaling}) and horizontal rescaling (\cref{prop:rescaling}). \emph{Supported by application-specific evidence}: the savings of predictive prompt and task sampling (\cref{sec:digital}). \emph{Proposed hypotheses}: a co-designed frontier that dominates on cost, quality of service and task success at once (Fig.~\ref{fig:pareto3d}), the decay of $\rho_u$ (\cref{prin:decay}), and a changed scaling exponent.

\subsection{The Scaling Hypothesis: Informative Compute}

Standard scaling laws express loss as a power law in total compute, $\calL(\calC)=(\calC_0/\calC)^{\alpha_C}$, with
\begin{equation}
  \calC_{\mathrm{info}}(\calC)\;=\;\int_0^{\calC}\rho_u^{\mathrm{exec}}(c)\,dc
  \label{eq:scaling-law}
\end{equation}
the \emph{informative compute} absorbed along a run, where $\rho_u^{\mathrm{exec}}(c)$ is the executed fraction at cumulative compute $c$, the share of compute then spent on utility-relevant interactions. Under unconditional training it equals the \emph{cost-weighted} pool density
\begin{equation}
  \rho_u^{\calC}(\theta_c)\;=\;\frac{\E_{i\sim\calD}\bigl[\calC(i)\,\indicator{u(i;\theta_c)\ge u_{\min}}\bigr]}{\E_{i\sim\calD}[\calC(i)]},
  \label{eq:rho-cost}
\end{equation}
the compute share of relevant interactions, as in \cref{prop:agentic-efficiency}; it reduces to the interaction share $\rho_u(\theta_c)$ only when execution costs are equal, and can differ substantially otherwise. Under a perfect gate it is one. With equal costs, as \cref{prop:rescaling} assumes, the two runs of \cref{prop:rescaling} accumulate informative compute at rates $\rho_u$ and $1$. The scaling hypothesis below is stated for $\rho_u$ in this equal-cost case; with heterogeneous costs, $\rho_u^{\calC}$ takes its place. The framework predicts more than a constant-factor gain: if the fraction of utility-relevant interactions changes as a learner becomes more capable, total compute is not the only meaningful coordinate of scaling, and the question is whether performance curves become more comparable when plotted against Eq.~(\ref{eq:scaling-law}). Two hypotheses answer differently and are experimentally distinguishable: the conservative one is that gating produces only a horizontal rescaling, the stronger one that it also alters the trajectory of interaction density and thereby the effective exponent.

\begin{proposition}[Horizontal Rescaling of the Scaling Curve]
\label{prop:rescaling}
Suppose (a) unconditional training follows a loss curve $\calL_{\mathrm{base}}(\calC)$ on a pool whose density of utility-relevant interactions is a constant $\rho_u\in(0,1]$; (b) dead interactions leave the learner unchanged; (c) progress depends only on the utility-relevant interactions executed, whose distribution and order the selection policy preserves; and (d) the selection policy is perfect, prediction is free, and executions cost the same. Then
\begin{equation}
  \calL_{\mathrm{CI}}(\calC)\;=\;\calL_{\mathrm{base}}\!\left(\calC/\rho_u\right),
  \label{eq:rescaling}
\end{equation}
i.e.\ on a log-compute axis the conditional curve is the unconditional one shifted left by $\ln(1/\rho_u)$.
\end{proposition}

This is a rescaling identity, valid for any shape of $\calL_{\mathrm{base}}$: it is the precise content of the $1/\rho_u$ multiplier used throughout. Assumption (a) fixes $\rho_u$, which the next hypothesis denies in general; when the pool fraction drifts, the constant factor becomes the time-varying reparameterization Eq.~(\ref{eq:scaling-law}), and that integral is what the curve-collapse test below plots against.

\begin{hypothesis}[Informative Interaction Decay]
\label{prin:decay}
After an initial transient, $\rho_u(\theta_t)$ is non-increasing as the system improves: mastered interactions leave the utility-relevant set, so $\rho_u$ tracks the shrinking measure of reachable but not-yet-mastered interactions.
\end{hypothesis}

\noindent The hypothesis is empirical and can fail, since new capabilities can expose new informative tasks and an improved predictor changes what is selected. Its logical direction matters, and it rests on an assumption that the definition of $\rho_u$ does not supply. Nothing in Eq.~(\ref{eq:rho-operational}) makes relevant interactions equally useful: two prompts with non-zero advantage can produce updates of very different size and direction. We therefore assume \emph{constant conditional progress}: over the regime studied, a unit of utility-relevant compute reduces the expected loss by the same amount $\eta$. Under that assumption unconditional learning is the absorption of informative compute, $d\calL/d\calC=-\eta\,\rho_u(\theta)$, and an observed power law corresponds to one particular decay, $\rho_u=k\,\calL^{1+1/\alpha}$ with $\alpha_C=\alpha$, but it does not follow that $\rho_u$ \emph{produces} the scaling law. The construction establishes consistency, not derivation: an \emph{independently measured} trajectory must satisfy it, so slower decay means a larger exponent. Gradient noise, intrinsic dimension and loss curvature generate power-law-like curves too; $\rho_u$ is simply the most \emph{actionable}, since gating intervenes on it directly (\cref{rem:identifiability}).

\noindent Against informative compute measured as in Eq.~(\ref{eq:rho-operational}), the null implied by \cref{prop:rescaling} predicts that the two curves \emph{coincide}, the alternative that the gated curve is \emph{steeper} even after rescaling. A strict improvement would not by itself establish \cref{prin:decay}, since different optimizer dynamics, lower gradient variance, a selection-induced curriculum or reduced redundancy could all produce it; attributing it to $\rho_u$ requires measuring $\rho_u$ along both runs and testing whether the difference in exponents is explained by the difference in those trajectories. The savings reported in \cref{sec:digital} are consistent with the null.

\subsection{Selective Supervision: The Gate Censors Its Own Training Data}
\label{sec:selective-supervision}

The most consequential technical limitation is structural: once System-1 decides which interactions are executed, it also decides which interactions it is trained on. With $p_{\mathrm{cand}}(i)$ the candidate distribution and $e(i)=\Pr[i\text{ executed}\mid i,\calH_{0:t}]$ the gate's execution probability, the predictor is trained on
\begin{equation}
  p_{\mathrm{train}}(i)\;\propto\;p_{\mathrm{cand}}(i)\,e(i)\;\neq\;p_{\mathrm{cand}}(i),
  \label{eq:selective-train}
\end{equation}
so labels are missing not at random: whether an outcome is observed depends on the predictor's own estimate of it, and in the worst case a candidate scored as low-value is suppressed, its outcome never observed, and the score never corrected. This is the selective-labels problem of algorithmic decision-making \citep{lakkaraju2017selective}, the positivity problem of causal inference, and, under shifting distributions, the label-querying problem of active self-training \citep{zhou2022online,zhou2022active}. Three features make it tractable rather than hopeless. Positivity can be designed in, but not by temperature: the indicator in Eq.~(\ref{eq:ci-policy}) gives every candidate with $\hat u<\kappa_t$ probability zero at every finite $\beta$, so a soft gate only randomizes within the admissible set. Positivity needs an explicit component that reaches the suppressed region: an exploration mixture $\pi_\epsilon=(1-\epsilon)\,\pi_{\mathrm{CI}}+\epsilon\,\pi_0$ with $\pi_0(i)>0$ on the whole pool, which gives $e(i)\ge\epsilon\,\pi_0(i)$ per draw, or independent audit sampling of suppressed candidates at a known rate $a$, which gives $e(i)\ge a$. Either way $e(i)\ge e_{\min}>0$. The propensities are known, because the gate assigns them, so weighting the predictor's loss by $1/e(i)$ \citep{horvitz1952generalization}, or combining that weighting with the predictor's own regression in a doubly robust estimator \citep{dudik2011doubly}, gives unbiased estimates of its error on the full candidate distribution. And the audit sample of suppressed candidates (\cref{sec:amortization}) is exactly such a positivity-preserving exploration stream, provided its rate is fixed independently of the predictor's scores. What remains open is quantitative: how large $e_{\min}$ must be for the predictor's error to shrink while the policy and the candidate distribution drift, how the variance of inverse-propensity weights trades against the screening gain, and whether any rate like Eq.~(\ref{eq:accuracy-growth}) survives adaptive, censored collection. Of all the problems the framework raises, we regard this as the one most in need of theory.

\subsection{Open Problems and Outlook}

The decay profile of $\rho_u(\theta_t)$ has not yet been measured, and doing so is the cleanest single test of the framework; it would also test a mechanistic account of \emph{emergence}, since a sub-skill whose $\rho_u$ crosses a usable threshold abruptly should appear abruptly. We regard this as the most actionable contribution of this Perspective.

Several questions stay open. Because aggressive gating discards informative interactions while permissive gating readmits dead ones, efficiency is maximized at an \emph{interior} admission rate $\rho^*(\varepsilon_\phi)$, whose characterization sits inside a non-stationary bandit problem \citep{auer2002}, as does its inference-time counterpart, whether a halting predictor can reach the accuracy \cref{prop:tts-efficiency} requires without its own forward pass eroding the savings. \Cref{prop:co-scaling} fixes the direction in which System-1 should grow with System-2, and sweeping System-1's size at several System-2 sizes while measuring $\mathrm{fp}$, $\mathrm{fn}$ and $c_{\text{pred}}$ would give the compute-optimal split between them. And the framework is single-learner: in agent societies, self-play and debate, informativeness is a property of joint beliefs, suggesting a \emph{relational} utility and an informative fraction defined on an interaction graph, since an interaction uninformative to its initiator may inform an observer. We regard this last as the most consequential boundary of the framework.

Conditional interaction finally resonates, without being equivalent, with the enactive view of cognition urged for AI by \citet{rafiee2026toward}, building on \citet{varela2017embodied} and the ``era of experience'' \citep{silver2025era}: perception as selective activity, in which an agent acts \emph{in order to} reveal information. Our utility nonetheless answers to an externally given task and cost ledger rather than the agent's own self-maintenance, so we formalize the \emph{selective-engagement} facet alone. Within that scope, routing tasks to System-2 executives (\cref{sec:master}) extends the gate from data to model selection and \cref{sec:test-time} adds the inference boundary, so a fully conditional system, whose data, computation, inference effort and external actions are all gated by a co-designed System-1, is a promising architecture for resource-efficient intelligence at scale, and the argument strengthens with capability if both \cref{prin:decay} and \cref{prop:co-scaling} hold. Beyond human-level performance \citep{mulgan2016superintelligence}, human data and curricula would no longer say what is worth learning, and a system would have to choose its own experiments; its System-1 would then decide what a superhuman System-2 learns and does, making the gate a major source of efficiency.

\section{Conclusion}

This work poses the perspective that the next stage of progress toward AGI may come from the co-design of System-1 and System-2: a cheap System-1 that values each candidate interaction before it is paid for, and a costly System-2 that executes only what clears the gate, each trained on the other's outputs. Conditional interaction, the predict--gate--execute--update loop that couples them, applies alike to training rollouts, test-time reasoning and agentic actions, and under explicit assumptions yields screening, memory, regret and co-scaling results. These results motivate the scaling hypothesis of \cref{sec:implications}. Our answer to the research question is therefore a qualified yes: the co-design is promising wherever interactions are costly, unequal in value and predictable in advance. 
The decisive next step is measurement: quantify $\rho_u$, compare scaling curves at matched informative compute, and test whether co-design changes more than a constant factor.

\bibliographystyle{plainnat}
\bibliography{scibib}

\newpage

\appendix

\section{Proofs}
\label{app:proofs}

Each subsection below restates a result exactly as it appears in the main text, keeping its original number, and proves it step by step. All results are proved for a general declared utility $u$ and its density of utility-relevant interactions $\rho_u$ (Eq.~(\ref{eq:rho-operational})); the informative interaction density $\rho_{\mathrm{info}}$ is the special case in which $u$ is the expected information gain, and every proof below specializes to it by reading ``utility-relevant'' as ``informative''. We write $u(i)=u(i;\theta_t)$ for the true utility of interaction $i$ and $\hat u(i)=\hat u_\phi(i\mid\calH_{0:t})$ for its prediction, suppressing arguments not in play. Expectations are over stochastic gradients, predictor outputs, and the interaction stream unless noted. Together, these results make precise what co-design buys and what it risks.

\subsection{Proof of \texorpdfstring{\cref{prop:ci-argmax}}{Proposition (Greedy Argmax Limit)}}
\label{app:ci-argmax}

\begin{restated}{prop:ci-argmax}
Let $\calI$ be finite and $\kappa_t\le\hat u_t^{\max}$, so that every maximizer is admissible. As $\beta\to\infty$,
\begin{equation*}
  \pi_{\mathrm{CI}}(i_t\mid\calH_{0:t})
  \;\longrightarrow\;
  \frac{\indicator{i_t\in\calI_t^\star}}{|\calI_t^\star|},
  \tag{\ref{eq:ci-argmax-limit}}
\end{equation*}
the greedy rule $i_t=\argmax_{i}\hat u(i\mid\calH_{0:t})$ with ties broken uniformly. If the maximizer $i_t^\star$ is unique, with gap $\Delta_t>0$ to the best other admissible candidate, then
\begin{equation*}
  1-\pi_{\mathrm{CI}}(i_t^\star\mid\calH_{0:t})
  \;\le\;
  \bigl(|\calI_t^{\mathrm{adm}}|-1\bigr)\,e^{-\beta\Delta_t},
  \tag{\ref{eq:ci-argmax-rate}}
\end{equation*}
so, when $|\calI_t^{\mathrm{adm}}|\ge2$, $\beta\ge\Delta_t^{-1}\ln\bigl((|\calI_t^{\mathrm{adm}}|-1)/\delta\bigr)$ selects $i_t^\star$ with probability at least $1-\delta$; when $|\calI_t^{\mathrm{adm}}|=1$ it is selected with probability one for every $\beta$.
\end{restated}

\begin{proof}
\textbf{Step 1 (Restrict the policy to the admissible set).}
The indicator in Eq.~(\ref{eq:ci-policy}) vanishes outside $\calI_t^{\mathrm{adm}}$, so for every $i\in\calI$
\begin{equation*}
  \pi_{\mathrm{CI}}(i\mid\calH_{0:t})
  =\frac{\indicator{i\in\calI_t^{\mathrm{adm}}}\,e^{\beta\hat u_t(i)}}{\sum_{j\in\calI_t^{\mathrm{adm}}}e^{\beta\hat u_t(j)}}.
\end{equation*}
Because $\kappa_t\le\hat u_t^{\max}$, every maximizer satisfies $\hat u_t(i)=\hat u_t^{\max}\ge\kappa_t$, hence $\calI_t^\star\subseteq\calI_t^{\mathrm{adm}}$, the denominator is at least $e^{\beta\hat u_t^{\max}}>0$, and $Z_t>0$.

\textbf{Step 2 (Normalize by the largest term).}
Define the utility deficit $d_t(i)=\hat u_t^{\max}-\hat u_t(i)\ge0$, which is zero exactly when $i\in\calI_t^\star$. Multiplying numerator and denominator by $e^{-\beta\hat u_t^{\max}}$,
\begin{equation*}
  \pi_{\mathrm{CI}}(i\mid\calH_{0:t})
  =\frac{\indicator{i\in\calI_t^{\mathrm{adm}}}\,e^{-\beta d_t(i)}}{\sum_{j\in\calI_t^{\mathrm{adm}}}e^{-\beta d_t(j)}}.
\end{equation*}

\textbf{Step 3 (Limit of the denominator).}
Split the denominator into maximizers and the rest:
\begin{equation*}
  \sum_{j\in\calI_t^{\mathrm{adm}}}e^{-\beta d_t(j)}
  =\underbrace{\sum_{j\in\calI_t^\star}1}_{=|\calI_t^\star|}
  +\sum_{j\in\calI_t^{\mathrm{adm}}\setminus\calI_t^\star}e^{-\beta d_t(j)}.
\end{equation*}
Every term of the second sum has $d_t(j)>0$, so it tends to $0$ as $\beta\to\infty$; the sum is finite because $\calI$ is finite. The denominator therefore converges to $|\calI_t^\star|$.

\textbf{Step 4 (Limit of each probability).}
For $i\in\calI_t^\star$ the numerator equals $1$, so $\pi_{\mathrm{CI}}(i)\to1/|\calI_t^\star|$. For $i\in\calI_t^{\mathrm{adm}}\setminus\calI_t^\star$ the numerator $e^{-\beta d_t(i)}\to0$, so $\pi_{\mathrm{CI}}(i)\to0$. For $i\notin\calI_t^{\mathrm{adm}}$, $\pi_{\mathrm{CI}}(i)=0$ for every $\beta$. Together these give Eq.~(\ref{eq:ci-argmax-limit}).

\textbf{Step 5 (Rate for a unique maximizer).}
Let $\calI_t^\star=\{i_t^\star\}$ and $S_\beta=\sum_{j\in\calI_t^{\mathrm{adm}}\setminus\{i_t^\star\}}e^{-\beta d_t(j)}$. By Step 2, $\pi_{\mathrm{CI}}(i_t^\star)=1/(1+S_\beta)$. Each of the $|\calI_t^{\mathrm{adm}}|-1$ terms of $S_\beta$ has $d_t(j)\ge\Delta_t$, so $S_\beta\le(|\calI_t^{\mathrm{adm}}|-1)e^{-\beta\Delta_t}$. Hence
\begin{equation*}
  1-\pi_{\mathrm{CI}}(i_t^\star)
  =\frac{S_\beta}{1+S_\beta}
  \le S_\beta
  \le\bigl(|\calI_t^{\mathrm{adm}}|-1\bigr)e^{-\beta\Delta_t},
\end{equation*}
which is Eq.~(\ref{eq:ci-argmax-rate}).

\textbf{Step 6 (Sufficient inverse temperature).}
The right-hand side of Eq.~(\ref{eq:ci-argmax-rate}) is at most $\delta$ if and only if $e^{-\beta\Delta_t}\le\delta/(|\calI_t^{\mathrm{adm}}|-1)$, i.e.\ $\beta\ge\Delta_t^{-1}\ln\bigl((|\calI_t^{\mathrm{adm}}|-1)/\delta\bigr)$. (If $|\calI_t^{\mathrm{adm}}|=1$ the policy is already deterministic for every $\beta$.)
\end{proof}

\subsection{Proof of \texorpdfstring{\cref{prop:iie-gain}}{Proposition (IIE Gain)}}
\label{app:iie-gain}

\begin{restated}{prop:iie-gain}
Assume dead candidates carry no utility ($\bar{\mathbb{I}}_{\text{dead}}=0$, as when $u_{\min}\to0^+$ for a non-negative utility), and use the system-level ledger of Eq.~(\ref{eq:IIE}), in its population form Eq.~(\ref{eq:IIE-pop}), which charges a predictor cost $c_{\text{pred}}$ to every screened candidate.
\begin{enumerate}[leftmargin=1.6em]
\item \emph{(Perfect, free predictor.)} If exactly the utility-relevant candidates are executed and $c_{\text{pred}}=0$,
\begin{equation*}
  \frac{\mathrm{IIE}^{\mathrm{CI}}}{\mathrm{IIE}^{\text{unc}}}
  \;=\;
  \frac{\bar{\mathbb{I}}_{\text{rel}}/\bar{\calC}_{\text{rel}}}{\bar{\mathbb{I}}/\bar{\calC}}
  \;=\;
  \frac{1}{\rho_u}\cdot\frac{\bar{\calC}}{\bar{\calC}_{\text{rel}}}
  \;=\;
  1+\frac{1-\rho_u}{\rho_u}\cdot\frac{\bar\calC_{\text{dead}}}{\bar\calC_{\text{rel}}},
  \tag{\ref{eq:IIE-ratio}}
\end{equation*}
which equals $1/\rho_u$ when $\bar\calC_{\text{dead}}=\bar\calC_{\text{rel}}$ and exceeds it if and only if $\bar\calC_{\text{dead}}>\bar\calC_{\text{rel}}$.
\item \emph{(Imperfect, costly predictor.)} With class-conditional false-positive rate $\mathrm{fp}$ (the fraction of dead candidates admitted) and false-negative rate $\mathrm{fn}$ (the fraction of utility-relevant candidates suppressed),
\begin{equation*}
  \frac{\mathrm{IIE}^{\mathrm{CI}}}{\mathrm{IIE}^{\text{unc}}}
  \;=\;
  \frac{\bar{\mathbb{I}}_{\text{rel}}/\bar{\calC}_{\text{rel}}}{\bar{\mathbb{I}}/\bar{\calC}}
  \cdot
  \left(1+\frac{\mathrm{fp}\,(1-\rho_u)\,\bar\calC_{\text{dead}}+c_{\text{pred}}}{(1-\mathrm{fn})\,\rho_u\,\bar\calC_{\text{rel}}}\right)^{-1},
  \tag{\ref{eq:IIE-ratio-imperfect}}
\end{equation*}
which is non-increasing in each of $\mathrm{fp}$, $\mathrm{fn}$ and $c_{\text{pred}}$, strictly decreasing in any of them whose coefficient is positive, and recovers Eq.~(\ref{eq:IIE-ratio}) as $\mathrm{fp},c_{\text{pred}}\to0$.
\end{enumerate}
\end{restated}

\begin{proof}
\textbf{Step 1 (Per-candidate accounting).}
We work with the population efficiency of Eq.~(\ref{eq:IIE-pop}), the ratio of the expected utility per screened candidate to the expected cost per screened candidate; the finite-stream $\mathrm{IIE}_T$ converges to it almost surely for an i.i.d.\ stream, and for a fixed pool the identities below hold exactly with empirical fractions and class means. We normalize everything per candidate. Recall $\bar{\mathbb{I}}=\rho_u\bar{\mathbb{I}}_{\text{rel}}+(1-\rho_u)\bar{\mathbb{I}}_{\text{dead}}$ and $\bar\calC=\rho_u\bar\calC_{\text{rel}}+(1-\rho_u)\bar\calC_{\text{dead}}$, and the clean case sets $\bar{\mathbb{I}}_{\text{dead}}=0$.

\textbf{Step 2 (Unconditional interaction).}
Unconditional interaction executes every candidate, so its utility and cost per candidate are $\bar{\mathbb{I}}$ and $\bar\calC$, and $\mathrm{IIE}^{\text{unc}}=\bar{\mathbb{I}}/\bar\calC$.

\textbf{Step 3 (Utility under the gate).}
The gate executes a fraction $1-\mathrm{fn}$ of the utility-relevant candidates (mass $(1-\mathrm{fn})\rho_u$) and a fraction $\mathrm{fp}$ of the dead ones (mass $\mathrm{fp}(1-\rho_u)$). In the clean case the latter add no utility, so the utility per candidate is
\begin{equation*}
  \mathbb{S}_{\mathrm{CI}}=(1-\mathrm{fn})\,\rho_u\,\bar{\mathbb{I}}_{\text{rel}}.
\end{equation*}

\textbf{Step 4 (Cost under the gate).}
Executed utility-relevant candidates cost $\bar\calC_{\text{rel}}$ each, executed dead candidates $\bar\calC_{\text{dead}}$ each, and every candidate is scored at $c_{\text{pred}}$, so
\begin{equation*}
  \mathbb{C}_{\mathrm{CI}}=(1-\mathrm{fn})\,\rho_u\,\bar\calC_{\text{rel}}+\mathrm{fp}\,(1-\rho_u)\,\bar\calC_{\text{dead}}+c_{\text{pred}}.
\end{equation*}

\textbf{Step 5 (Factor out the utility-relevant cost).}
Dividing $\mathbb{S}_{\mathrm{CI}}$ and $\mathbb{C}_{\mathrm{CI}}$ by $(1-\mathrm{fn})\rho_u\bar\calC_{\text{rel}}$ (assume $\mathrm{fn}<1$),
\begin{equation*}
  \mathrm{IIE}^{\mathrm{CI}}=\frac{\mathbb{S}_{\mathrm{CI}}}{\mathbb{C}_{\mathrm{CI}}}
  =\frac{\bar{\mathbb{I}}_{\text{rel}}}{\bar\calC_{\text{rel}}}\bigl(1+g\bigr)^{-1},
  \qquad
  g=\frac{\mathrm{fp}\,(1-\rho_u)\,\bar\calC_{\text{dead}}+c_{\text{pred}}}{(1-\mathrm{fn})\,\rho_u\,\bar\calC_{\text{rel}}}.
\end{equation*}
Dividing by $\mathrm{IIE}^{\text{unc}}$ from Step 2 gives Eq.~(\ref{eq:IIE-ratio-imperfect}).

\textbf{Step 6 (Monotonicity).}
Write the ratio as $K(1+g)^{-1}$ with $K=(\bar{\mathbb{I}}_{\text{rel}}/\bar\calC_{\text{rel}})/(\bar{\mathbb{I}}/\bar\calC)$ independent of the error rates. The map $g\mapsto(1+g)^{-1}$ is decreasing, and $g$ is non-decreasing in $\mathrm{fp}$ (coefficient $(1-\rho_u)\bar\calC_{\text{dead}}\ge0$), in $c_{\text{pred}}$ (coefficient positive), and in $\mathrm{fn}$ (through $(1-\mathrm{fn})^{-1}$, which increases on $[0,1)$). Hence the ratio is non-increasing in each, and strictly decreasing whenever the corresponding coefficient is positive.

\textbf{Step 7 (Ideal case).}
With $\mathrm{fp}=c_{\text{pred}}=0$, $g=0$ and the ratio equals $K$, the first equality of Eq.~(\ref{eq:IIE-ratio}). Substituting $\bar{\mathbb{I}}=\rho_u\bar{\mathbb{I}}_{\text{rel}}$,
\begin{equation*}
  K=\frac{\bar{\mathbb{I}}_{\text{rel}}}{\bar\calC_{\text{rel}}}\cdot\frac{\bar\calC}{\rho_u\bar{\mathbb{I}}_{\text{rel}}}
  =\frac{1}{\rho_u}\cdot\frac{\bar\calC}{\bar\calC_{\text{rel}}},
\end{equation*}
the second equality. Substituting $\bar\calC$,
\begin{equation*}
  \frac{1}{\rho_u}\cdot\frac{\rho_u\bar\calC_{\text{rel}}+(1-\rho_u)\bar\calC_{\text{dead}}}{\bar\calC_{\text{rel}}}
  =1+\frac{1-\rho_u}{\rho_u}\cdot\frac{\bar\calC_{\text{dead}}}{\bar\calC_{\text{rel}}},
\end{equation*}
the third equality.

\textbf{Step 8 (Comparison with $1/\rho_u$).}
Subtracting, $K-\frac{1}{\rho_u}=\frac{1-\rho_u}{\rho_u}\bigl(\frac{\bar\calC_{\text{dead}}}{\bar\calC_{\text{rel}}}-1\bigr)$, which for $\rho_u\in(0,1)$ is zero if and only if $\bar\calC_{\text{dead}}=\bar\calC_{\text{rel}}$ and positive if and only if $\bar\calC_{\text{dead}}>\bar\calC_{\text{rel}}$.

\textbf{Step 9 (Dead candidates with positive utility).}
If the zero-utility assumption is dropped, $\bar{\mathbb{I}}_{\text{dead}}\ge0$, the perfect free predictor still has $\mathrm{IIE}^{\mathrm{CI}}=\rho_u\bar{\mathbb{I}}_{\text{rel}}/(\rho_u\bar\calC_{\text{rel}})$, so the first equality of Eq.~(\ref{eq:IIE-ratio}) holds unchanged. Since $\rho_u\bar{\mathbb{I}}_{\text{rel}}=\bar{\mathbb{I}}-(1-\rho_u)\bar{\mathbb{I}}_{\text{dead}}$,
\begin{equation*}
  K=\frac{\bar\calC}{\bar\calC_{\text{rel}}}\cdot\frac{\bar{\mathbb{I}}_{\text{rel}}}{\bar{\mathbb{I}}}
  =\frac{1}{\rho_u}\cdot\frac{\bar\calC}{\bar\calC_{\text{rel}}}\Bigl(1-\frac{(1-\rho_u)\,\bar{\mathbb{I}}_{\text{dead}}}{\bar{\mathbb{I}}}\Bigr)
  \;\le\;\frac{1}{\rho_u}\cdot\frac{\bar\calC}{\bar\calC_{\text{rel}}},
\end{equation*}
with equality if and only if $\bar{\mathbb{I}}_{\text{dead}}=0$ (for $\rho_u<1$). The closed forms of Eq.~(\ref{eq:IIE-ratio}) are then upper bounds; this is the relevant case when the floor $u_{\min}$ of Eq.~(\ref{eq:rho-operational}) is strictly positive.
\end{proof}

\noindent\emph{Remarks.} False negatives affect the ratio only through $g$: the IIE is utility per unit cost, and suppressed utility-relevant candidates are neither paid for nor counted. Their main effect is on \emph{yield}, $\mathbb{S}_{\mathrm{CI}}$, which falls by the factor $1-\mathrm{fn}$ and determines how large a pool is needed to reach a target. In the notation of Eq.~(\ref{eq:yield}), $\mathbb{S}_{\mathrm{CI}}$ is the expected yield. Setting $c_{\text{pred}}=0$ recovers execution-only efficiency.

\subsection{Proof of \texorpdfstring{\cref{prop:amortized-screening}}{Proposition (Amortized Screening Gain)}}
\label{app:amortized-screening}

\begin{restated}{prop:amortized-screening}
Let a pool of $n$ candidates contain a fraction $\rho_u$ of utility-relevant ones, with execution cost $c_{\text{exec}}$, screening cost $c_{\text{screen}}=c_{\text{embed}}+c_{\text{score}}$ (encoding the candidate and scoring it; this is the $c_{\text{pred}}$ of Eq.~(\ref{eq:IIE})), and $r=c_{\text{screen}}/c_{\text{exec}}$. With a perfect screen, acquiring the utility-relevant subset costs
\begin{equation*}
  \frac{\calC_{\mathrm{CI}}}{\calC_{\text{react}}}
  \;=\;
  \rho_u \;+\; r
  \tag{\ref{eq:amortized-ratio}}
\end{equation*}
of what exhaustive execute-then-filter evaluation spends ($\calC_{\text{react}}=n\,c_{\text{exec}}$, the baseline that must execute every candidate to find all relevant ones), a speedup approaching $1/\rho_u$ as $r\to0$, provided a candidate's utility class cannot be known without executing or scoring it.
\end{restated}

\begin{proof}
\textbf{Step 1 (Information model).}
We assume that a candidate's utility class (relevant or dead) cannot be learned without either executing it (reactive evaluation) or scoring it with the predictor (conditional interaction), and that screening costs $c_{\text{screen}}=c_{\text{embed}}+c_{\text{pred}}$: the predictor head costs $c_{\text{score}}=\Theta(d_\phi^2)$ because it acts on a fixed-size embedding, and encoding the candidate costs $c_{\text{embed}}$, zero when embeddings are precomputed and cached; neither depends on the rollout length $L$ or the group size $G$.

\textbf{Step 2 (Reactive cost and its lower bound).}
A reactive scheme has no signal about an unexecuted candidate. If it skipped a candidate, it could not tell whether that candidate is utility-relevant, so it could not guarantee to acquire the full utility-relevant subset. It must therefore execute all $n$ candidates:
\begin{equation*}
  \calC_{\text{react}}=n\,c_{\text{exec}}=\Theta(n\,G\,L\,d_{\text{model}}^2),
\end{equation*}
which matches Eq.~(\ref{eq:unc-complexity}). Under the information model of Step~1, no execute-then-filter scheme that must guarantee the full utility-relevant subset can spend less; schemes that sample adaptively, use side information or accept losses fall outside this baseline.

\textbf{Step 3 (Conditional cost).}
Conditional interaction encodes and scores all $n$ candidates, at cost $n\,c_{\text{screen}}$. A perfect screen then executes exactly the $\rho_u n$ utility-relevant candidates, at cost $\rho_u n\,c_{\text{exec}}$. Hence
\begin{equation*}
  \calC_{\mathrm{CI}}=n\,c_{\text{screen}}+\rho_u\,n\,c_{\text{exec}},
\end{equation*}
which matches Eq.~(\ref{eq:ci-complexity}) with $B=\rho_u n$.

\textbf{Step 4 (Ratio and speedup).}
Dividing Step 3 by Step 2,
\begin{equation*}
  \frac{\calC_{\mathrm{CI}}}{\calC_{\text{react}}}
  =\frac{n\,c_{\text{screen}}+\rho_u n\,c_{\text{exec}}}{n\,c_{\text{exec}}}
  =\rho_u+\frac{c_{\text{screen}}}{c_{\text{exec}}}
  =\rho_u+r .
\end{equation*}
The speedup is the reciprocal $1/(\rho_u+r)$, which increases to $1/\rho_u$ as $r\to0$.
\end{proof}

\subsection{Proof of \texorpdfstring{\cref{prop:memory-scope}}{Proposition (Memory-Decoupled Scope)} and \texorpdfstring{\cref{cor:scope-quality}}{Corollary}}
\label{app:memory-scope}

\begin{restated}{prop:memory-scope}
To select the top-$B$ of $n$ candidates, synchronous reactive evaluation needs peak memory $M_{\text{react}}(n)$, whereas predictive scoring streamed in micro-batches of size $b$ needs $M_{\mathrm{CI}}(n)$, where
\begin{align*}
  M_{\text{react}}(n) &\;=\; \Theta\!\bigl(n\,G\,L\,d_{\text{model}}\bigr),
  \tag{\ref{eq:react-memory}}\\
  M_{\mathrm{CI}}(n) &\;=\; \Theta\!\bigl(B\,G\,L\,d_{\text{model}} \;+\; b\,d_e \;+\; B\bigr),
  \tag{\ref{eq:ci-memory}}
\end{align*}
independent of $n$; the pool is then limited only by screening time $\Theta(n\,c_{\text{pred}})$ and score storage.
\end{restated}

\begin{proof}
\textbf{Step 1 (Assumptions).}
(i) Synchronous reactive evaluation keeps each candidate's in-flight state resident until all outcomes can be compared; this state is $\Theta(GLd_{\text{model}})$ per candidate (the KV cache of $G$ length-$L$ generations, or trajectory and simulator state). (ii) Predictive scoring needs only the embedding $e(\tau)\in\R^{d_e}$ and returns a scalar. (iii) Scoring has no side effects.

\textbf{Step 2 (Reactive peak memory).}
By (i), at the moment outcomes are compared all $n$ states are resident, so the peak memory is $n\cdot\Theta(GLd_{\text{model}})$, giving Eq.~(\ref{eq:react-memory}).

\textbf{Step 3 (Reactive pool cap).}
Requiring $n\,\Theta(GLd_{\text{model}})\le M$ gives $n\le n_{\max}^{\text{react}}=\Theta\bigl(M/(GLd_{\text{model}})\bigr)$.

\textbf{Step 4 (Predictive peak memory).}
Process candidates in micro-batches of size $b$. By (ii) and (iii), each embedding can be discarded once its scalar score is recorded, so at most $b$ embeddings are resident at once: $\Theta(b\,d_e)$. The scores feed a running top-$B$ structure: $\Theta(B)$. The only large state is that of the $B$ executed survivors: $\Theta(BGLd_{\text{model}})$. Summing the three terms gives Eq.~(\ref{eq:ci-memory}), and none of them depends on $n$.

\textbf{Step 5 (What still depends on $n$).}
Screening all candidates takes time $\Theta(n\,c_{\text{pred}})$, and storing every score takes $\Theta(n)$ memory (or $\Theta(B)$ if only the running top-$B$ is kept). These, not rollout memory, bound the pool.
\end{proof}

\noindent\emph{Remark (streaming reactive evaluation).} If assumption (i) is dropped, a reactive filter can also execute candidates in micro-batches and keep only their scores, giving $n$-independent peak memory at $\Theta(n\,c_{\text{exec}})$ time. The memory advantage is thus genuine for synchronous pipelines (e.g.\ group-based advantage estimation holding all $G$ rollouts of a prompt together, or batched simulators), and otherwise reduces to the time advantage of \cref{prop:amortized-screening}.

\begin{restated}{cor:scope-quality}
The best true utility $U_{(1)}^n$ among $n$ i.i.d.\ candidates is non-decreasing in $n$: it converges to the essential supremum if utility is bounded, and grows as $n^{\xi}$ (in probability, up to a slowly varying factor) if $\Pr(U>u)=u^{-1/\xi}\ell(u)$. Under predictive sampling this gain costs no additional rollout memory (Eq.~(\ref{eq:ci-memory})), provided the predictor ranks the top of the pool correctly.
\end{restated}

\begin{proof}
Let $U_1,U_2,\dots$ be i.i.d.\ candidate utilities and $U_{(1)}^n=\max_{i\le n}U_i$.

\textbf{Step 1 (Monotonicity).}
Adding a candidate cannot lower the maximum: $U_{(1)}^{n+1}=\max(U_{(1)}^n,U_{n+1})\ge U_{(1)}^n$ pointwise, so $U_{(1)}^n$ is non-decreasing in $n$, pointwise and hence stochastically.

\textbf{Step 2 (Bounded utility).}
Let $u^+$ be the essential supremum. For any $s>0$, $p_s=\Pr(U\le u^+-s)<1$ by definition of $u^+$, and independence gives $\Pr(U_{(1)}^n\le u^+-s)=p_s^{\,n}\to0$. Hence $U_{(1)}^n\to u^+$ in probability, and almost surely because the sequence is monotone.

\textbf{Step 3 (Heavy upper tail).}
Suppose $\Pr(U>u)=u^{-1/\xi}\ell(u)$ with $\ell$ slowly varying, and choose $a_n$ with $\Pr(U>a_n)=1/n$; then $a_n$ is regularly varying in $n$ with index $\xi$. For $x>0$,
\begin{equation*}
  \Pr\bigl(U_{(1)}^n\le x\,a_n\bigr)=\bigl(1-\Pr(U>x\,a_n)\bigr)^n
  =\Bigl(1-\tfrac{x^{-1/\xi}\ell(xa_n)}{n\,\ell(a_n)}\Bigr)^{n}
  \longrightarrow\exp\bigl(-x^{-1/\xi}\bigr),
\end{equation*}
using $\ell(xa_n)/\ell(a_n)\to1$. This is the Fréchet law, so $U_{(1)}^n/a_n$ converges in distribution and $U_{(1)}^n$ grows as $n^{\xi}$ up to a slowly varying factor, in probability. For $\xi>1$ the mean of $U$ is infinite (for $\xi=1$ it may be, depending on $\ell$), so the growth statement is made in probability, not in expectation.

\textbf{Step 4 (Cost of realizing the growth).}
By \cref{prop:memory-scope}, enlarging the pool leaves peak rollout memory at Eq.~(\ref{eq:ci-memory}); by \cref{prop:amortized-screening}, each extra candidate adds only $c_{\text{pred}}$. The growth of Steps 2 and 3 concerns the best \emph{true} utility in the pool; it is realized by the selection only if the predictor ranks the top of the pool correctly, which is quantified by \cref{thm:regret,cor:regret-hp}.
\end{proof}

\subsection{Proof of \texorpdfstring{\cref{prop:active-inference}}{Proposition (Active Inference Optimality)}}
\label{app:active-inference}

\begin{restated}{prop:active-inference}
Write the expected free energy of interaction $i$ as
\begin{equation*}
  G(i)
  = \underbrace{-\,\E_{q_\phi(o\mid i)}\!\bigl[\ln p(o\mid C)\bigr]}_{\text{pragmatic value}}
  - \underbrace{\E_{q_\phi(o\mid i)}\!\bigl[\KL{q_\phi(\theta\mid o,i)}{q_\phi(\theta\mid i)}\bigr]}_{\text{epistemic value}},
  \tag{\ref{eq:efe-decomp}}
\end{equation*}
with preferences $p(o\mid C)$ over a finite outcome space $\mathcal{O}$. If preferences are uniform and cost is accounted separately through $\calC(i)$, then
\begin{equation*}
  \argmin_{i\in\calI} G(i)
  = \argmax_{i\in\calI}\; I_{q_\phi}(\theta;o\mid i)
  = \argmax_{i\in\calI}\; \calU_{\mathrm{IG}}^{q_\phi}(i\mid\calH_t),
  \tag{\ref{eq:MI-optimality}}
\end{equation*}
where $\calU_{\mathrm{IG}}^{q_\phi}$ is Eq.~(\ref{eq:exp-info-gain}) evaluated under the variational model $q_\phi$, and coincides with $\calU_{\mathrm{IG}}$ when $q_\phi$ is the exact posterior and predictive. Under non-uniform preferences the two criteria diverge.
\end{restated}

\begin{proof}
\textbf{Step 1 (Well-posedness).}
Assume $q_\phi(\theta\mid o,i)$ is absolutely continuous with respect to $q_\phi(\theta\mid i)$ for every $o$ and $i$, so all divergences below are finite.

\textbf{Step 2 (The pragmatic term is constant).}
Under uniform preferences $p(o\mid C)=1/|\mathcal{O}|$, so $-\ln p(o\mid C)=\ln|\mathcal{O}|$ for every $o$, and its expectation under any $q_\phi(o\mid i)$ is $\ln|\mathcal{O}|$, independent of $i$.

\textbf{Step 3 (Reduce to the epistemic term).}
Subtracting a constant does not change a minimizer, and minimizing the negative of a quantity is maximizing it, so by Eq.~(\ref{eq:efe-decomp})
\begin{equation}
\begin{aligned}
  \argmin_i G(i)
  &=\argmin_i\Bigl(-\E_{q_\phi(o\mid i)}\bigl[\KL{q_\phi(\theta\mid o,i)}{q_\phi(\theta\mid i)}\bigr]\Bigr)\\
  &=\argmax_i\;\E_{q_\phi(o\mid i)}\bigl[\KL{q_\phi(\theta\mid o,i)}{q_\phi(\theta\mid i)}\bigr].
\end{aligned}
  \label{eq:app-epistemic}
\end{equation}

\textbf{Step 4 (The epistemic term is a mutual information).}
Write the divergence as an expectation and use Bayes' rule, $q_\phi(\theta\mid o,i)/q_\phi(\theta\mid i)=q_\phi(\theta,o\mid i)/[q_\phi(\theta\mid i)\,q_\phi(o\mid i)]$:
\begin{equation}
  \E_{q_\phi(o\mid i)}\!\Bigl[\E_{q_\phi(\theta\mid o,i)}\ln\frac{q_\phi(\theta\mid o,i)}{q_\phi(\theta\mid i)}\Bigr]
  =\E_{q_\phi(\theta,o\mid i)}\!\left[\ln\frac{q_\phi(\theta,o\mid i)}{q_\phi(\theta\mid i)\,q_\phi(o\mid i)}\right]
  =I_{q_\phi}(\theta;o\mid i).
  \label{eq:app-mi}
\end{equation}

\textbf{Step 5 (Identify the utility).}
By definition, $\calU_{\mathrm{IG}}^{q_\phi}(i\mid\calH_t)$ is the expected posterior-to-prior divergence of Eq.~(\ref{eq:exp-info-gain}) with $p$ replaced by $q_\phi$, which by Eq.~(\ref{eq:app-mi}) equals $I_{q_\phi}(\theta;o\mid i)$. Combining Steps 3 and 4 gives Eq.~(\ref{eq:MI-optimality}). If $q_\phi$ is the exact posterior and predictive, $\calU_{\mathrm{IG}}^{q_\phi}=\calU_{\mathrm{IG}}$; otherwise the maximizer is that of the variational surrogate, not necessarily of the true information gain.

\textbf{Step 6 (Non-uniform preferences).}
If $p(o\mid C)$ is not uniform, the pragmatic term depends on $i$ through $q_\phi(o\mid i)$, it no longer drops out in Step 3, and the minimizer of $G$ trades information gain against preference satisfaction.
\end{proof}

\subsection{Proof of \texorpdfstring{\cref{thm:regret}}{Theorem (Regret Bound)} and \texorpdfstring{\cref{cor:regret-hp}}{Corollary}}
\label{app:regret}

\begin{restated}{thm:regret}
Assume (i) selection is the deterministic limit of the fixed-size batch rule Eq.~(\ref{eq:subset-selection}) with $\mu=0$ and $\calA$ the predicted utility $\hat u$, the cardinality-constrained form of the master objective Eq.~(\ref{eq:master}), so each round retains the top-$B$ set by $\hat u$; (ii) $\varepsilon_\phi=\sup_{0\le t<T}\sup_{i\in\calI_t}\abs{\hat u_t(i)-u_t(i)}$, a time-uniform error, with $\hat u_t(i)=\hat u_{\phi_t}(i\mid\calH_{0:t})$; and (iii) $\mathrm{gap}_{\min} = \min_{t}\min_{i\in\mathcal{B}_t^*,\,j\in\calI_t\setminus\mathcal{B}_t^*}\bigl(u_{t}(i)-u_{t}(j)\bigr)>0$. Then
\begin{equation*}
  \calR_T(\pi_{\mathrm{CI}}, \pi^*)
  \;\le\;
  2\,T\,B\,\varepsilon_\phi\,\cdot\,
  \indicator{2\varepsilon_\phi \ge \mathrm{gap}_{\min}};
  \tag{\ref{eq:regret-bound}}
\end{equation*}
in particular, the selection is exactly the oracle's whenever $2\varepsilon_\phi<\mathrm{gap}_{\min}$. More generally, with round-specific errors $\varepsilon_t=\sup_{i\in\calI_t}\abs{\hat u_t(i)-u_t(i)}$ and gaps $\mathrm{gap}_t$, $\calR_T\le\sum_{t<T}2B\varepsilon_t\,\indicator{2\varepsilon_t\ge\mathrm{gap}_t}$.
\end{restated}

\begin{proof}
\textbf{Step 1 (Per-round regret).}
By Eq.~(\ref{eq:cond-regret}), $\calR_T=\sum_{t=0}^{T-1}r_t$ with $r_t=\sum_{i\in\mathcal{B}_t^*}u_t(i)-\sum_{i\in\hat{\mathcal{B}}_t}u_t(i)$. It suffices to bound each $r_t$. Fix a round $t$ and, in Steps 2--6, write $u=u_t$ and $\hat u=\hat u_t$; by the time-uniform assumption (ii), $\abs{\hat u(i)-u(i)}\le\varepsilon_\phi$ for every $i\in\calI_t$.

\textbf{Step 2 (Only swapped items contribute).}
Write $\mathcal{B}_t^*\setminus\hat{\mathcal{B}}_t=\{i_1,\dots,i_m\}$ (oracle items missed) and $\hat{\mathcal{B}}_t\setminus\mathcal{B}_t^*=\{j_1,\dots,j_m\}$ (wrong items selected). Both sets have the same size $m\le B$ because both selections have size $B$. Items in $\mathcal{B}_t^*\cap\hat{\mathcal{B}}_t$ cancel, so
\begin{equation}
  r_t=\sum_{k=1}^m\bigl[u(i_k)-u(j_k)\bigr]
  \label{eq:app-rt}
\end{equation}
for any pairing of the two sets.

\textbf{Step 3 (Ranking inequality).}
$\hat{\mathcal{B}}_t$ is the top-$B$ set by $\hat u$ and $i_k\notin\hat{\mathcal{B}}_t$, so every selected item scores at least as high: $\hat u(j_k)\ge\hat u(i_k)$ for all $k$.

\textbf{Step 4 (Each swap costs at most $2\varepsilon_\phi$).}
By (ii), $u(i_k)\le\hat u(i_k)+\varepsilon_\phi$ and $u(j_k)\ge\hat u(j_k)-\varepsilon_\phi$. Subtracting and using Step 3,
\begin{equation}
  u(i_k)-u(j_k)\le\hat u(i_k)-\hat u(j_k)+2\varepsilon_\phi\le2\varepsilon_\phi.
  \label{eq:app-swap-bound}
\end{equation}

\textbf{Step 5 (A swap needs a large error).}
Since $i_k\in\mathcal{B}_t^*$ and $j_k\notin\mathcal{B}_t^*$, (iii) gives $u(i_k)-u(j_k)\ge\mathrm{gap}_{\min}$. Combined with Eq.~(\ref{eq:app-swap-bound}), a swap is possible only if $\mathrm{gap}_{\min}\le2\varepsilon_\phi$. Hence if $2\varepsilon_\phi<\mathrm{gap}_{\min}$, then $m=0$, $\hat{\mathcal{B}}_t=\mathcal{B}_t^*$, and $r_t=0$.

\textbf{Step 6 (Per-round bound).}
Combining Steps 2, 4 and 5, $r_t\le m\cdot2\varepsilon_\phi\cdot\indicator{2\varepsilon_\phi\ge\mathrm{gap}_{\min}}\le2B\varepsilon_\phi\,\indicator{2\varepsilon_\phi\ge\mathrm{gap}_{\min}}$.

\textbf{Step 7 (Sum over rounds).}
Summing Step 6 over $t=0,\dots,T-1$ gives Eq.~(\ref{eq:regret-bound}). Steps 2--6 use the error and gap of round $t$ only, so replacing $\varepsilon_\phi$ and $\mathrm{gap}_{\min}$ by $\varepsilon_t$ and $\mathrm{gap}_t$ gives $r_t\le2B\varepsilon_t\indicator{2\varepsilon_t\ge\mathrm{gap}_t}$ and the round-specific bound. Both sides are in the units of the declared utility (nats when $u$ is information gain); $\mathrm{gap}_{\min}$ enters only through the indicator.
\end{proof}

\begin{restated}{cor:regret-hp}
Suppose each round's pool has at most $n$ candidates; for every round $t$ and every candidate outside a blind-spot set $\calO_t\subseteq\calI_t$ with $|\calO_t|\le n_{\mathrm{ood}}$, the error $\hat u_t(i)-u_t(i)$ is zero-mean and $\sigma_\phi$-sub-Gaussian, while errors on $\calO_t$ are unrestricted; and the utility takes values in an interval of length $u_{\max}$ (e.g.\ $u\in[0,u_{\max}]$). Then, with probability at least $1-\delta$ and $\varepsilon_\phi^{\mathrm{eff}}=\sigma_\phi\sqrt{2\ln(2nT/\delta)}$,
\begin{equation*}
  \calR_T(\pi_{\mathrm{CI}}, \pi^*)
  \;\le\;
  2\,T\,B\,\varepsilon_\phi^{\mathrm{eff}}\,
  \indicator{2\varepsilon_\phi^{\mathrm{eff}}\ge\mathrm{gap}_{\min}}
  \;+\; T\,\min(B,n_{\mathrm{ood}})\,u_{\max}.
  \tag{\ref{eq:regret-hp}}
\end{equation*}
\end{restated}

\begin{proof}
\textbf{Step 1 (Tail bound for one candidate).}
A zero-mean $\sigma_\phi$-sub-Gaussian error satisfies $\Pr\bigl(|\hat u_t(i)-u_t(i)|>s\bigr)\le2e^{-s^2/(2\sigma_\phi^2)}$ for all $s>0$.

\textbf{Step 2 (Union bound over candidates and rounds).}
There are at most $nT$ candidate--round pairs outside the blind spots $\calO_t$. Setting $s=\varepsilon_\phi^{\mathrm{eff}}=\sigma_\phi\sqrt{2\ln(2nT/\delta)}$ makes each tail probability at most $\delta/(nT)$, so with probability at least $1-\delta$ every error outside the blind spots in every round is at most $\varepsilon_\phi^{\mathrm{eff}}$. The rest of the proof works on this event.

\textbf{Step 3 (Classify swapped pairs).}
In round $t$, pair the swapped items as in Step 2 of the proof of \cref{thm:regret}. Call a pair \emph{clean} if neither of its items lies in $\calO_t$ and \emph{contaminated} otherwise.

\textbf{Step 4 (Clean pairs).}
Steps 3--5 of the proof of \cref{thm:regret} apply verbatim with $\varepsilon_\phi^{\mathrm{eff}}$ in place of $\varepsilon_\phi$: each clean pair costs at most $2\varepsilon_\phi^{\mathrm{eff}}$ and exists only if $2\varepsilon_\phi^{\mathrm{eff}}\ge\mathrm{gap}_{\min}$. There are at most $B$ of them, contributing at most $2B\varepsilon_\phi^{\mathrm{eff}}\indicator{2\varepsilon_\phi^{\mathrm{eff}}\ge\mathrm{gap}_{\min}}$.

\textbf{Step 5 (Contaminated pairs).}
Each item appears in at most one pair, so at most $n_{\mathrm{ood}}$ pairs contain a blind-spot item, and at most $B$ pairs exist in total; hence there are at most $\min(B,n_{\mathrm{ood}})$ contaminated pairs. Since the utility ranges over an interval of length $u_{\max}$, each costs $u(i_k)-u(j_k)\le u_{\max}$.

\textbf{Step 6 (Sum over rounds).}
Adding Steps 4 and 5 and summing over $T$ rounds gives Eq.~(\ref{eq:regret-hp}) on the event of Step 2.
\end{proof}

\noindent\emph{Remark.} The blind-spot term is stated as a count rather than a pool fraction deliberately: selection by $\hat u$ preferentially admits candidates whose utility is overestimated, so a single blind spot with unbounded overestimation is selected with certainty, and a bound of the form $B\,\delta_{\mathrm{ood}}u_{\max}$, with $\delta_{\mathrm{ood}}$ the blind-spot fraction, would not hold in general. If utilities are tied at the $B$-th position ($\mathrm{gap}_{\min}=0$), the indicators are always active and the bounds degrade linearly in the error.

\subsection{Proof of \texorpdfstring{\cref{lem:margin}}{Lemma (Margin Condition)}}
\label{app:margin}

\begin{restated}{lem:margin}
Let candidates $i\sim\calD$ be utility-relevant when $u(i)\ge u_{\min}$, let the gate admit $i$ when $\hat u(i)\ge u_{\min}$, and suppose $\abs{\hat u(i)-u(i)}\le\varepsilon$ for every $i$. More generally, let the gate admit $i$ when $\hat u(i)\ge\kappa$. Every misclassified candidate, admitted while dead or suppressed while relevant, then lies in the band $\mathcal{B}_\varepsilon=\{i:\kappa-\varepsilon\le u(i)<u_{\min}\}\cup\{i:u_{\min}\le u(i)<\kappa+\varepsilon\}$, so the misclassified masses satisfy $\delta_{\mathrm{FP}}+\delta_{\mathrm{FN}}\le\Pr_{i\sim\calD}(\mathcal{B}_\varepsilon)$. With $\kappa=u_{\min}$, if the true utility satisfies the margin condition $\Pr_{i\sim\calD}\bigl(\abs{u(i)-u_{\min}}\le\varepsilon\bigr)\le m\,\varepsilon$ with constant $m$, then the gate misclassifies, by admitting a dead candidate or suppressing a relevant one, at most a fraction $m\,\varepsilon$ of candidates.
\end{restated}

\begin{proof}
Let the gate admit $i$ when $\hat u(i)\ge\kappa$; the lemma's first statement is the case $\kappa=u_{\min}$.

\textbf{Step 1 (A false negative lies in the upper part of the band).}
A relevant candidate is suppressed when $u(i)\ge u_{\min}$ and $\hat u(i)<\kappa$. Then $u(i)\le\hat u(i)+\varepsilon<\kappa+\varepsilon$, so $u(i)\in[u_{\min},\kappa+\varepsilon)$.

\textbf{Step 2 (A false positive lies in the lower part of the band).}
A dead candidate is admitted when $u(i)<u_{\min}$ and $\hat u(i)\ge\kappa$. Then $u(i)\ge\hat u(i)-\varepsilon\ge\kappa-\varepsilon$, so $u(i)\in[\kappa-\varepsilon,u_{\min})$.

\textbf{Step 3 (Count).}
Every misclassified candidate therefore lies in $\mathcal{B}_\varepsilon$, so $\delta_{\mathrm{FP}}+\delta_{\mathrm{FN}}\le\Pr_{i\sim\calD}(\mathcal{B}_\varepsilon)$. With $\kappa=u_{\min}$, $\mathcal{B}_\varepsilon\subseteq\{i:\abs{u(i)-u_{\min}}\le\varepsilon\}$, whose probability under $\calD$ is at most $m\varepsilon$ by the margin condition.
\end{proof}

\subsection{Proof of \texorpdfstring{\cref{prop:tts-efficiency}}{Proposition (Test-Time Compute Efficiency)} and \texorpdfstring{\cref{prop:agentic-efficiency}}{Proposition (Agentic Call Efficiency)}}
\label{app:tts-efficiency}

\begin{restated}{prop:tts-efficiency}
Let $\calA^{\mathrm{tt}}$ be the queries on which the full budget is more accurate in expectation than one increment, with density $\rho_{\mathrm{acc}}^{\mathrm{tt}}$, and let the gate threshold a declared test-time utility whose relevant set contains $\calA^{\mathrm{tt}}$, with density $\rho_u^{\mathrm{tt}}\ge\rho_{\mathrm{acc}}^{\mathrm{tt}}$. Let unconditional scaling spend $C$ unit increments per query, let the answer be a function of the belief state, and let the selection policy be the one-probe gate: it spends one probing increment plus $c_{\text{pred}}$ and then halts or continues to the full budget, with $\delta_{\mathrm{FN}}$ and $\delta_{\mathrm{FP}}$ the fractions of all queries that are relevant but halted and dead but continued. Then the gated policy loses at most $\delta_{\mathrm{FN}}$ accuracy while spending at most a fraction
\begin{equation*}
  \rho_u^{\mathrm{tt}}+\frac{1+c_{\text{pred}}}{C}+\delta_{\mathrm{FP}}
\end{equation*}
of the unconditional compute, a multiplier approaching $1/\rho_u^{\mathrm{tt}}$ as $C\to\infty$ and $\delta_{\mathrm{FP}}\to0$. By \cref{lem:margin}, $\delta_{\mathrm{FN}}+\delta_{\mathrm{FP}}\le\Pr(\mathcal{B}_{\varepsilon_\phi})$ for a predictor with error at most $\varepsilon_\phi$, which is at most $m\varepsilon_\phi$ under the margin condition. The tightest choice is the accuracy-gain utility itself, for which $\rho_u^{\mathrm{tt}}=\rho_{\mathrm{acc}}^{\mathrm{tt}}$. The continuation utility of Eq.~(\ref{eq:tt-continuation}) with $u_{\min}\to0^+$ is also admissible: if its value is zero then $b_C=b_1$ almost surely, so the answer, and with it the accuracy, is unchanged; its gate should use $\kappa>\varepsilon_\phi$, so that the atom of zero-utility queries is never admitted. With density $\rho_{\mathrm{cont}}^{\mathrm{tt}}\ge\rho_{\mathrm{acc}}^{\mathrm{tt}}$ it gives the weaker multiplier $1/\rho_{\mathrm{cont}}^{\mathrm{tt}}$. The next-step utility of Eq.~(\ref{eq:tt-utility}) is admissible only under the no-delayed-gain condition.
\end{restated}

\begin{proof}
\textbf{Step 1 (Assumptions made explicit).}
(i) Unconditional scaling spends $C$ unit-cost increments on every query. (ii) Every query outside the relevant set $R$ of the declared utility lies outside $\calA^{\mathrm{tt}}$, so its expected accuracy after the first increment is at least that after $C$ increments. For the continuation utility of Eq.~(\ref{eq:tt-continuation}) with $u_{\min}\to0^+$ this holds because a query outside $R$ has $\E[\KL{b_C}{b_1}]=0$; the divergence is non-negative, so $b_C=b_1$ almost surely, and since the answer is a function of the belief state, the answer and its expected accuracy are unchanged. For the next-step utility of Eq.~(\ref{eq:tt-utility}) the same conclusion needs the no-delayed-gain condition, because a zero next-step change does not rule out later changes. (iii) The gated policy spends one probing increment, pays $c_{\text{pred}}$ for its prediction, and then either halts or continues to the full budget $C$. (iv) Per-query accuracy lies in $[0,1]$.

\textbf{Step 2 (Partition of queries).}
Let $\rho=\rho_u^{\mathrm{tt}}$. Each query falls into one of four classes: relevant and continued (fraction $\rho-\delta_{\mathrm{FN}}$), relevant but halted (false negatives, $\delta_{\mathrm{FN}}$), dead and halted ($1-\rho-\delta_{\mathrm{FP}}$), and dead but continued (false positives, $\delta_{\mathrm{FP}}$); these are unconditional masses of the whole query pool, not the class-conditional rates $\mathrm{fp},\mathrm{fn}$ of \cref{prop:iie-gain}. By \cref{lem:margin}, $\delta_{\mathrm{FN}}+\delta_{\mathrm{FP}}\le\Pr(\mathcal{B}_{\varepsilon_\phi})$, which is at most $m\varepsilon_\phi$ under the margin condition.

\textbf{Step 3 (Accuracy class by class).}
Continued queries, relevant or not, receive the full budget and give exactly the unconditional answer. Correctly halted dead queries are, by (ii), at least as accurate as under the full budget. Only false negatives can lose accuracy, and by (iv) each loses at most one unit. Hence the accuracy loss is at most $\delta_{\mathrm{FN}}$.

\textbf{Step 4 (Compute).}
Every query costs $c_{\text{pred}}$ for prediction, plus $C$ if continued or $1$ if halted. The continued fraction is $\rho-\delta_{\mathrm{FN}}+\delta_{\mathrm{FP}}\le\rho+\delta_{\mathrm{FP}}$ and the halted fraction is at most $1$, so the expected cost per query is at most
\begin{equation*}
  (\rho+\delta_{\mathrm{FP}})\,C+1+c_{\text{pred}}.
\end{equation*}

\textbf{Step 5 (Compute fraction).}
Dividing by the unconditional cost $C$ gives the upper bound $\rho+\frac{1+c_{\text{pred}}}{C}+\delta_{\mathrm{FP}}$ on the compute fraction stated in the proposition, which tends to $\rho$ as $C\to\infty$ and $\delta_{\mathrm{FP}}\to0$; the multiplier then approaches $1/\rho$.
\end{proof}

\begin{restated}{prop:agentic-efficiency}
Suppose (a) omitting any set of decision-irrelevant actions leaves the answer unchanged; (b) the gate thresholds a predicted relevance score with error at most $\varepsilon_\phi$, and on every task the action scores satisfy the margin condition of \cref{lem:margin} with constant $m$, so that in expectation at most a fraction $m\,\varepsilon_\phi$ of the actions are misclassified; and (c) prediction costs a fraction $r$ of an average action. Consider tasks with $N$ candidate actions in which no action costs more than $R_c$ times the task's average action cost. If utility-relevant actions carry a share $\rho_u^{\calC}$ of the expected call-everything cost (the cost-weighted density of utility-relevant actions, which reduces to their fraction $\rho_u$ when all actions cost the same), the conditional agent's accuracy is within $N m\,\varepsilon_\phi$ of call-everything accuracy, at expected cost
\begin{equation*}
  \E[\calC_{\mathrm{CI}}]
  \;\le\;
  \bigl(\rho_u^{\calC}+r+m\,\varepsilon_\phi\,R_c\bigr)\,
  \E[\calC_{\text{all}}].
\end{equation*}
\end{restated}

\begin{proof}
\textbf{Step 1 (Setup).}
Consider a task with $N$ candidate actions with costs $c_1,\dots,c_N\le c_{\max}$ and total $C_{\text{all}}=\sum_kc_k$, so the average action costs $\bar c=C_{\text{all}}/N$ and, by assumption, $c_{\max}\le R_c\,\bar c$. The predictor costs $r\bar c$ per action, $rC_{\text{all}}$ in total.

\textbf{Step 2 (Accuracy).}
Let $\mathcal{E}$ be the set of actions the conditional agent executes, and $\mathcal{E}_{\mathrm{rel}}$ the set of utility-relevant actions. If there is no false negative, $\mathcal{E}\supseteq\mathcal{E}_{\mathrm{rel}}$, so $\mathcal{E}$ differs from the call-everything set only by omitting some decision-irrelevant actions, and by (a) the committed answer is unchanged. The answer can therefore differ only on the event that at least one false negative occurs. Conditional on the task, by (b) and Markov's inequality, this event has probability at most $\E[\#\text{false negatives}]\le m\varepsilon_\phi N$, which bounds the accuracy loss; averaging over tasks preserves the bound. The loss grows with the task length $N$.

\textbf{Step 3 (Cost of one task).}
The conditional agent pays the predictor on every action ($rC_{\text{all}}$), the full cost on executed relevant actions (at most the task's relevant cost $C_{\text{rel}}=\sum_{k\in\mathcal{E}_{\mathrm{rel}}}c_k$), and the full cost on false positives, each at most $c_{\max}\le R_c\bar c$. Conditional on the task, by (b) the expected number of false positives is at most $m\varepsilon_\phi N$, so
\begin{equation*}
  \E[C_{\mathrm{CI}}\mid\text{task}]
  \le rC_{\text{all}}+C_{\text{rel}}+m\varepsilon_\phi N R_c\bar c
  = rC_{\text{all}}+C_{\text{rel}}+m\varepsilon_\phi R_c\,C_{\text{all}}.
\end{equation*}

\textbf{Step 4 (Take expectations over tasks).}
By definition of the cost-weighted density, $\E[C_{\text{rel}}]=\rho_u^{\calC}\,\E[C_{\text{all}}]$. Averaging Step 3 over tasks therefore gives
\begin{equation*}
  \E[C_{\mathrm{CI}}]\le\bigl(\rho_u^{\calC}+r+m\varepsilon_\phi R_c\bigr)\,\E[C_{\text{all}}],
\end{equation*}
the stated bound with every constant explicit. When all actions cost the same, $C_{\text{rel}}/C_{\text{all}}$ is the fraction of relevant actions and $\rho_u^{\calC}$ reduces to $\rho_u$.
\end{proof}

\noindent\emph{Remark.} Separability is essential. If two individually irrelevant actions are jointly relevant (for instance, two retrievals whose combination answers the question), suppressing both can change the answer, and the accuracy statement fails.

\subsection{Proof of \texorpdfstring{\cref{prop:IB}}{Proposition (Budgeted Selection)}}
\label{app:IB}

\begin{restated}{prop:IB}
Assume equal costs, a budget of $B$ executions, and that each candidate's utility depends on the predictor outputs only through its own, $\E[U(i)\mid\hat U]=\E[U(i)\mid\hat U(i)]$. Then, among all selection rules that see only the predictor outputs $\hat U$, executing the $B$ candidates with the largest posterior mean $\E[U\mid\hat U]$ maximizes expected executed utility and hence the IIE. Setting $\kappa_t$ to the $B$-th largest posterior mean makes the admissible set of Eq.~(\ref{eq:ci-policy}) exactly that top-$B$ set, up to ties at the $B$-th value.
\end{restated}

\begin{proof}
\textbf{Step 1 (Admissible rules).}
A selection rule that depends on the candidates only through their predictor outputs chooses a set $S$ of size $B$ that is a function of $\hat U=(\hat U(1),\dots,\hat U(n))$ (and possibly of independent randomization).

\textbf{Step 2 (IIE reduces to executed utility).}
With equal costs $\bar\calC$, a constant scoring cost $c_{\text{pred}}$ for each of the $n$ screened candidates, and exactly $B$ executions, the IIE denominator in Eq.~(\ref{eq:IIE}) equals $n\,c_{\text{pred}}+B\bar\calC$ for every such $S$, so maximizing IIE is equivalent to maximizing $\E\bigl[\sum_{i\in S}U(i)\bigr]$.

\textbf{Step 3 (Tower property).}
Because $S$ is determined by $\hat U$,
\begin{equation*}
  \E\Bigl[\sum_{i\in S}U(i)\Bigr]
  =\E\Bigl[\,\E\Bigl[\sum_{i\in S}U(i)\,\Big|\,\hat U\Bigr]\Bigr]
  =\E\Bigl[\sum_{i\in S}\E\bigl[U(i)\mid\hat U\bigr]\Bigr].
\end{equation*}

\textbf{Step 4 (Locality).}
Assume that each candidate's utility depends on the predictor outputs only through its own output, i.e.\ $\E[U(i)\mid\hat U]=\E[U(i)\mid\hat U(i)]=:\mu(\hat U(i))$.

\textbf{Step 5 (Pointwise optimization).}
For each realization of $\hat U$, the inner sum $\sum_{i\in S}\mu(\hat U(i))$ over sets of size $B$ is maximized by the $B$ largest values of $\mu(\hat U(i))$: exchanging any selected item for an unselected one with a larger value increases the sum. The top-$B$ rule is therefore optimal for every realization, hence in expectation, and randomization cannot help.

\textbf{Step 6 (Realization by the gate).}
Setting $\hat u=\mu(\hat U)=\E[U\mid\hat U]$ and $\kappa_t$ to the $B$-th largest of these values makes the admissible set $\calI_t^{\mathrm{adm}}$ of Eq.~(\ref{eq:ci-admissible}) exactly the top-$B$ set of Step 5 (when several candidates tie at the $B$-th value, the admissible set contains all of them and ties are broken arbitrarily, which leaves the executed utility unchanged), so executing the admissible set realizes the optimum. This is the thresholding half of Eq.~(\ref{eq:ci-policy}); the sampling half is irrelevant here, since the whole admissible set is executed, and the batch rule is the cardinality-constrained form of Eq.~(\ref{eq:subset-selection}) rather than the single-draw limit of \cref{prop:ci-argmax}.
\end{proof}

\noindent\emph{Remark (relation to the information bottleneck).} Stationary points of Eq.~(\ref{eq:IB}) satisfy the self-consistent IB equations \citep{tishby2000information}, $q(\hat U\mid\calH)\propto q(\hat U)\exp\bigl(-\beta_{\mathrm{IB}}\,\KL{p(U\mid\calH)}{p(U\mid\hat U)}\bigr)$. Stationarity does not imply that $\hat U$ is the most informative feature of its complexity, and a higher $I(\hat U;U)$ does not by itself guarantee better top-$B$ selection. The proposition therefore makes no claim about which $\hat U$ is optimal, only about how to use a given one.

\subsection{Proof of \texorpdfstring{\cref{prop:co-scaling}}{Proposition (Co-Scaling)}}
\label{app:co-scaling}

\begin{restated}{prop:co-scaling}
Assume the setting of \cref{prop:iie-gain}(ii); that $\rho_u$, $\mathrm{fp}(s)$ and $\mathrm{fn}(s)<1$ do not depend on $k$; that $c_{\text{pred}}(s)/(1-\mathrm{fn}(s))$ is non-decreasing in $s$; and that the IIE ratio Eq.~(\ref{eq:IIE-ratio-imperfect}) attains its maximum over $\calS_1$ for every $k$, on a set $S^\star(k)$ of optimal System-1 sizes. Then (i) $S^\star$ is non-decreasing in $k$ in the strong set order: if $k'>k$, $s\in S^\star(k)$ and $s'\in S^\star(k')$, then $\max(s,s')\in S^\star(k')$ and $\min(s,s')\in S^\star(k)$, so the largest and the smallest optimal sizes, where they exist, are non-decreasing in $k$; and (ii) the maximal IIE ratio is non-decreasing in $k$.
\end{restated}

\begin{proof}
\textbf{Step 1 (The ratio in terms of $s$ and $k$).}
Write $\bar\calC_{\text{rel}},\bar\calC_{\text{dead}}$ for the execution costs at $k=1$. After scaling System~2 they become $k\bar\calC_{\text{rel}}$ and $k\bar\calC_{\text{dead}}$, so the pool mean $\bar\calC$ also scales by $k$. The first factor of Eq.~(\ref{eq:IIE-ratio-imperfect}), $(\bar{\mathbb{I}}_{\text{rel}}/\bar\calC_{\text{rel}})/(\bar{\mathbb{I}}/\bar\calC)=\bar{\mathbb{I}}_{\text{rel}}\bar\calC/(\bar{\mathbb{I}}\,\bar\calC_{\text{rel}})$, is a ratio of two costs and is therefore unchanged by $k$; it does not depend on $s$ either, and we call it $K>0$. Dividing numerator and denominator of the overhead term by $k$,
\begin{equation*}
  \frac{\mathrm{IIE}^{\mathrm{CI}}}{\mathrm{IIE}^{\text{unc}}}(s,k)
  \;=\;\frac{K}{1+h(s,k)},
  \qquad
  h(s,k)\;=\;a(s)+\frac{b(s)}{k},
\end{equation*}
with
\begin{equation*}
  a(s)=\frac{\mathrm{fp}(s)\,(1-\rho_u)\,\bar\calC_{\text{dead}}}{(1-\mathrm{fn}(s))\,\rho_u\,\bar\calC_{\text{rel}}},
  \qquad
  b(s)=\frac{c_{\text{pred}}(s)}{(1-\mathrm{fn}(s))\,\rho_u\,\bar\calC_{\text{rel}}}.
\end{equation*}
By assumption neither $a$ nor $b$ depends on $k$, and $b\ge0$.

\textbf{Step 2 (Reduction to minimizing $h$).}
Since $x\mapsto K/(1+x)$ is strictly decreasing on $x\ge0$, maximizing the IIE ratio over $s$ is equivalent to minimizing $h(\cdot,k)$, and $S^\star(k)=\argmin_{s\in\calS_1}h(s,k)$. Put $t=1/k$, so that increasing $k$ means decreasing $t$, and $h=a(s)+b(s)\,t$.

\textbf{Step 3 (Increasing differences).}
The factor $\rho_u\bar\calC_{\text{rel}}$ is a positive constant, so $b$ is non-decreasing in $s$ by the assumption on $c_{\text{pred}}/(1-\mathrm{fn})$. Hence for $s'\ge s$ the difference
\begin{equation*}
  h(s',t)-h(s,t)\;=\;a(s')-a(s)+\bigl(b(s')-b(s)\bigr)\,t
\end{equation*}
is non-decreasing in $t$.

\textbf{Step 4 (Strong set order).}
Let $k'>k$, i.e.\ $t'=1/k'<t=1/k$, and take $s\in S^\star(k)$, $s'\in S^\star(k')$. If $s\le s'$ then $\max(s,s')=s'$ and $\min(s,s')=s$ and there is nothing to prove, so suppose $s'<s$. Optimality of $s$ at $t$ gives $h(s,t)\le h(s',t)$. By Step 3 applied to the pair $s'<s$,
\begin{equation*}
  h(s,t')-h(s',t')\;\le\;h(s,t)-h(s',t)\;\le\;0,
\end{equation*}
so $s=\max(s,s')$ is also optimal at $t'$, i.e.\ $s\in S^\star(k')$. Symmetrically, optimality of $s'$ at $t'$ gives $h(s',t')\le h(s,t')$, and Step 3 gives
\begin{equation*}
  h(s',t)-h(s,t)\;=\;-\bigl(h(s,t)-h(s',t)\bigr)\;\le\;-\bigl(h(s,t')-h(s',t')\bigr)\;=\;h(s',t')-h(s,t')\;\le\;0,
\end{equation*}
so $s'=\min(s,s')$ is optimal at $t$, i.e.\ $s'\in S^\star(k)$. This is (i). For the extreme points: if $\bar s=\max S^\star(k)$ exists, taking $s=\bar s$ and any $s'\in S^\star(k')$ shows $\max(\bar s,s')\in S^\star(k')$, so every upper bound of $S^\star(k')$ is at least $\bar s$; likewise $\min S^\star(k)\le\min S^\star(k')$ whenever both minima exist.

\textbf{Step 5 (The optimal ratio).}
For each fixed $s$, $h(s,k)=a(s)+b(s)/k$ is non-increasing in $k$ because $b(s)\ge0$. A pointwise minimum of non-increasing functions is non-increasing, so $h^\star(k)=\min_s h(s,k)$ is non-increasing in $k$, and the maximal ratio $K/(1+h^\star(k))$ is non-decreasing in $k$. This is (ii).
\end{proof}

\noindent\emph{Remark.} The monotonicity assumption on $c_{\text{pred}}/(1-\mathrm{fn})$ is mild: $1-\mathrm{fn}$ lies in $(0,1]$ and can rise at most to one, while the scoring cost grows roughly in proportion to System~1's size. If it fails, the conclusion can fail too, since a larger predictor both more accurate and relatively cheaper per recalled utility-relevant candidate could become optimal only at small $k$. The assumption that error rates do not depend on $k$ is the substantive one, discussed after the proposition.

\subsection{Proof of \texorpdfstring{\cref{prop:rescaling}}{Proposition (Horizontal Rescaling)}}
\label{app:rescaling}

\begin{restated}{prop:rescaling}
Suppose (a) unconditional training follows a loss curve $\calL_{\mathrm{base}}(\calC)$ on a pool whose density of utility-relevant interactions is a constant $\rho_u\in(0,1]$; (b) dead interactions leave the learner unchanged; (c) progress depends only on the utility-relevant interactions executed, whose distribution and order the selection policy preserves; and (d) the selection policy is perfect, prediction is free, and executions cost the same. Then
\begin{equation*}
  \calL_{\mathrm{CI}}(\calC)\;=\;\calL_{\mathrm{base}}\!\left(\calC/\rho_u\right),
  \tag{\ref{eq:rescaling}}
\end{equation*}
i.e.\ on a log-compute axis the conditional curve is the unconditional one shifted left by $\ln(1/\rho_u)$.
\end{restated}

\begin{proof}
\textbf{Step 1 (Unconditional informative compute).}
By (a) and (d), each unit of compute executes candidates of equal cost, a fraction $\rho_u$ of which are utility-relevant. Unconditional training at compute $\calC'$ therefore spends informative compute $\rho_u\calC'$ (the integral of Eq.~(\ref{eq:scaling-law}) with $\rho_u^{\mathrm{exec}}\equiv\rho_u$).

\textbf{Step 2 (Conditional informative compute).}
By (d), the gate executes only utility-relevant candidates at no screening cost, so conditional training at compute $\calC$ spends informative compute $\calC$ ($\rho_u^{\mathrm{exec}}\equiv1$).

\textbf{Step 3 (Only informative compute matters).}
By (b), executing a dead candidate leaves the learner unchanged, and by (c) the learner's state depends only on the utility-relevant interactions executed, whose distribution and order are the same in both regimes. Two runs that have executed the same informative compute are therefore in the same state and have the same loss.

\textbf{Step 4 (Match the runs).}
By Steps 1--3, $\calL_{\mathrm{CI}}(\calC)=\calL_{\mathrm{base}}(\calC')$ whenever $\rho_u\calC'=\calC$, i.e.\ $\calC'=\calC/\rho_u$, which is Eq.~(\ref{eq:rescaling}).

\textbf{Step 5 (Log-axis interpretation).}
On a logarithmic compute axis, $\ln(\calC/\rho_u)=\ln\calC+\ln(1/\rho_u)$, so the conditional curve is the unconditional one shifted left by $\ln(1/\rho_u)$. No property of $\calL_{\mathrm{base}}$ beyond being a function of compute was used, so the identity holds for any curve shape.
\end{proof}


\section{Acquisition Functions for the Digital and Physical Instantiations}
\label{app:acquisition}

This appendix collects the acquisition functions behind the methods discussed in \cref{sec:digital}, and shows that they share one cost-normalized form: the scalar on which System-1 gates System-2.

\paragraph{Digital: scoring a prompt before generation.}
Given the prompt and the model's recent history, System-1 estimates how the outcomes of a prompt $\tau$ would be distributed and scores it by how much of that distribution falls in the informative, mixed-outcome band,
\begin{equation}
  \calA_{\text{digital}}(\tau)
  \;=\; \E_{z\sim\pi_t(\cdot\mid\tau)}\!\bigl[w(z)\bigr],
  \label{eq:digital-acquisition}
\end{equation}
with $z$ summarizing how the prompt would resolve under the current policy and $w$ large on the mixed-outcome band. The predictive prompt selectors of \cref{sec:digital} are of this form.

\paragraph{Physical: adaptation risk and its tail.}
The quantity worth predicting for a task configuration is how badly the agent would still do on it \emph{after} adapting to it, the \emph{adaptation risk}
\begin{equation}
  \calR_{\text{adapt}}(\tau;\theta)
  \;=\; \E_{\xi\sim p(\xi\mid\tau)}\!\left[
    \ell\!\bigl(\theta-\eta_{\text{in}}\nabla_\theta\calL(\theta,\tau,\xi),\,\tau_{\text{target}}\bigr)
  \right],
  \label{eq:adaptation-risk}
\end{equation}
and concentrating training on its adverse tail is minimizing a conditional value-at-risk \citep{rockafellar2000} at level $\nu\in[0,1)$,
\begin{equation}
  \min_\theta\;
  \CVaR_\nu\!\bigl[\ell(\theta_{\text{adapted}},\tau)\bigr]
  \;=\; \min_\theta\;\frac{1}{1-\nu}\int_\nu^1 F_\ell^{-1}(v)\,dv,
  \label{eq:cvar}
\end{equation}
which reduces to the average case at $\nu=0$. MPTS and PDTS \citep{wang2025modelpredictivetasksampling, qu2025fast} estimate this risk from a task embedding and the optimization history. MPTS scores each candidate optimistically by $\hat{\calR}_{\text{adapt}}(\tau;\theta_t)+\gamma\,\sigma_\phi(\tau)$. PDTS replaces this upper-confidence score with one posterior sample $\tilde{\calR}_{\text{adapt}}(\tau;\theta_t)$ per candidate and selects the whole batch at once, trading summed risk against the diversity of the subset:
\begin{equation}
  \mathcal{B}_t^\star
  \;=\; \argmax_{\mathcal{B}\subseteq\hat{\calT}_t,\;|\mathcal{B}|=B}\;
       \sum_{\tau\in\mathcal{B}} \tilde{\calR}_{\text{adapt}}(\tau;\theta_t)
       \;+\; \mu\,\mathrm{Div}(\mathcal{B}),
  \label{eq:phys-acquisition}
\end{equation}
where $\hat{\calT}_t$ is the candidate pool and $\mathrm{Div}(\mathcal{B})$ grows with the pairwise distances $d(\tau,\tau')$ between selected tasks. The problem is NP-hard and is solved approximately; when $\mathrm{Div}$ is a sum of pairwise distances it is supermodular, so the $(1-1/e)$ greedy guarantee of \cref{sec:mpc} does not apply, and a coverage or log-determinant diversity term would restore it; it is the physical instance of the budgeted subset selection of Eq.~(\ref{eq:subset-selection}).

\begin{remark}[Common acquisition form]
\label{prop:unified-acquisition}
The per-prompt score of Eq.~(\ref{eq:digital-acquisition}) and the per-task risk summed in Eq.~(\ref{eq:phys-acquisition}) are both local approximations of the cost-normalized expected utility
\begin{equation}
  \calA^*(i\mid\calH_{0:t})
  \;=\; \frac{1}{\calC(i)}\,\E_{q_\phi}\!\bigl[\,u(i;\theta_t)\,\bigr],
  \label{eq:unified-acquisition}
\end{equation}
with the mixed-outcome proxy, locally a Fisher-weighted gradient norm (Steps 3--4 below), in the digital case and predicted tail adaptation risk in the physical case; the diversity term captures diminishing returns from similar interactions. They share the form, not the scalar (\cref{rem:utility-scope}).
\end{remark}

\subsection{Derivation for \texorpdfstring{\cref{prop:unified-acquisition}}{Remark (Common Acquisition Form)}}
\label{app:unified-acquisition}

\noindent This is a heuristic derivation under local (Laplace) approximations, not a theorem. Only Steps 2 and 7 are exact, and only within the Gaussian models they state. The rest is a local analogy: a gradient step is not a Bayesian posterior update, a large gradient can reflect noise or a harmful direction rather than useful learning, and the connection to the acquisition functions of the cited methods is a motivation, not a derivation of them.

\textbf{Step 1 (Local model).}
Assume the pre-interaction belief is approximately Gaussian, $\mathcal{N}(\theta_t,\calF_t^{-1})$, with precision $\calF_t$ (the accumulated Fisher information), and that an interaction moves the mean by a small step $\Delta\theta(i)=\eta\,g(i)$, where $g(i)$ is the score of the observed learning signal.

\textbf{Step 2 (Quadratic information gain).}
For two Gaussians with the same covariance $\calF_t^{-1}$ and means differing by $\Delta\theta$, the KL divergence is exactly $\frac{1}{2}\Delta\theta^\top\calF_t\Delta\theta$. A single observation contracts the covariance only slightly, so to leading order
\begin{equation}
  \KL{p(\theta\mid\calH_{0:t},i,o)}{p(\theta\mid\calH_{0:t})}
  \;\approx\;
  \tfrac12\,\Delta\theta(i)^\top \calF_t\,\Delta\theta(i).
  \label{eq:app-KL-quad}
\end{equation}

\textbf{Step 3 (Expected utility).}
Substituting $\Delta\theta(i)=\eta\,g(i)$ and taking expectations under $q_\phi$ gives $\E_{q_\phi}[u(i;\theta_t)]\approx\frac{\eta^2}{2}\E_{q_\phi}\bigl[g(i)^\top\calF_t\,g(i)\bigr]=\frac{\eta^2}{2}\E_{q_\phi}\norm{g(i)}_{\calF_t}^2$.

\textbf{Step 4 (Digital case).}
If the score depends on the prompt only through the latent informativeness $z$, then conditioning on $z$ and dividing by cost gives $\calA^*=\E_{z\sim\pi_t(\cdot\mid\tau)}[w(z)]$ with $w(z)\propto\E[\norm{g}^2_{\calF_t}\mid z]/\calC$. This is Eq.~(\ref{eq:digital-acquisition}); $w$ is large exactly where outcomes are mixed, and the three-state weight $\indicator{z=\text{partial}}$ is its binarization.

\textbf{Step 5 (Physical case).}
Linearizing the post-adaptation loss of Eq.~(\ref{eq:adaptation-risk}) around $\theta_t$, its expectation is to leading order the predicted risk $\hat{\calR}_{\text{adapt}}(\tau;\theta_t)$, and the spread of this linear predictor under $q_\phi$ is $\sigma_\phi(\tau)$. An optimistic (upper-confidence) estimate of the risk is then $\hat{\calR}_{\text{adapt}}+\gamma\sigma_\phi$, the per-task score of MPTS, and a posterior sample $\tilde{\calR}_{\text{adapt}}$, as summed in Eq.~(\ref{eq:phys-acquisition}), is its randomized counterpart; with roughly constant simulator cost, dividing by $\calC$ does not change the ranking.

\textbf{Step 6 (Where the diversity term does \emph{not} come from).}
For a batch, the summed update $\sum_{i}\Delta\theta(i)$ in Eq.~(\ref{eq:app-KL-quad}) produces cross terms $\eta^2g(i)^\top\calF_t\,g(j)$, which are \emph{positive} when updates align. The mean-shift approximation therefore does not produce a redundancy penalty.

\textbf{Step 7 (Where it does come from).}
Redundancy comes from covariance contraction, which Step 2 neglected. In the linear-Gaussian model $o_i=g(i)^\top\theta+\varepsilon_i$ with $\varepsilon_i\sim\mathcal{N}(0,\sigma^2)$ conditionally independent given $\theta$, the information gained by adding $i$ after a set $S$ is
\begin{equation*}
  I(\theta;o_i\mid o_S)=\tfrac12\ln\bigl(1+\sigma^{-2}g(i)^\top\Sigma_S\,g(i)\bigr),
\end{equation*}
where $\Sigma_S$ is the posterior covariance after $S$. Since $\Sigma_S$ shrinks along the directions $g(j)$, $j\in S$, already probed, this marginal gain falls as $g(i)$ aligns with the selected interactions. Information gain is thus submodular in the executed set of this linear-Gaussian model; nothing here shows that the task utility of a generic language model or embodied agent is submodular, and the subset diversity term $\mu\,\mathrm{Div}(\mathcal{B})$ in Eq.~(\ref{eq:phys-acquisition}) is a cheap proxy for this diminishing return, though a pairwise-distance sum is not itself submodular.

\end{document}